\documentclass{article}
\usepackage[final]{neurips_2023}

\usepackage[utf8]{inputenc} % allow utf-8 input
\usepackage[T1]{fontenc}    % use 8-bit T1 fonts
\usepackage{hyperref}       % hyperlinks
\usepackage{url}            % simple URL typesetting
\usepackage{booktabs}       % professional-quality tables
\usepackage{amsfonts}       % blackboard math symbols
\usepackage{nicefrac}       % compact symbols for 1/2, etc.
\usepackage{microtype}      % microtypography
\usepackage{xcolor}         % colors
\usepackage{graphicx}
\usepackage{subfigure}

\usepackage{amsmath}
\usepackage{amsthm}
\usepackage{wrapfig}
\usepackage{color}
\usepackage{caption}
\usepackage{diagbox}
\usepackage{tabularx}
\usepackage{natbib}

\let\hat\widehat
\let\tilde\widetilde

\usepackage{amssymb}
\usepackage{mathtools}
\usepackage{amsthm}

\usepackage[capitalize,noabbrev]{cleveref}

\theoremstyle{plain}
\newtheorem{theorem}{Theorem}[section]
\newtheorem{proposition}[theorem]{Proposition}

\theoremstyle{definition}
\newtheorem{definition}[theorem]{Definition}

\theoremstyle{remark}

\usepackage{smile}

\def\##1\#{\begin{align}#1\end{align}}
\def\$#1\${\begin{align*}#1\end{align*}}

\newcommand{\la}{\langle}
\newcommand{\ra}{\rangle}
\usepackage[textsize=tiny]{todonotes}
\usepackage{bbm}

\title{Unsupervised Behavior Extraction via Random Intent Priors}

\author{%
  \hspace{-0.14in}
  Hao Hu$^1$\footnotemark[1], Yiqin Yang$^2$\footnotemark[1], Jianing Ye$^1$, Ziqing Mai$^1$, Chongjie Zhang$^3$ \\
  \hspace{-0.14in}
  $^1$Institute for Interdisciplinary Information Sciences, Tsinghua University \\
  \hspace{-0.14in}
  $^2$Department of Automation, Tsinghua University \\
\hspace{-0.14in}
  $^3$Department of Computer Science \& Engineering, Washington University in St. Louis \\
  \hspace{-0.14in}
  \texttt{\{huh22, yangyiqi19, yejn21, maizq21\}@mails.tsinghua.edu.cn}\\
  \hspace{-0.14in}
  \texttt{chongjie@wustl.edu} \\
}

\begin{document}

\maketitle

\renewcommand{\thefootnote}{\fnsymbol{footnote}}
\footnotetext[1]{Equal contribution.}

\begin{abstract}
   % What?
   Reward-free data is abundant and contains rich prior knowledge of human behaviors, but it is not well exploited by offline reinforcement learning (RL) algorithms.
  In this paper, we propose \textbf{UBER}, an unsupervised approach to extract useful behaviors from offline reward-free datasets via diversified rewards. 
  UBER assigns different pseudo-rewards sampled from a given prior distribution to different agents to extract a diverse set of behaviors, and reuse them as candidate policies to facilitate the learning of new tasks. Perhaps surprisingly, we show that rewards generated from random neural networks are sufficient to extract diverse and useful behaviors, some even close to expert ones. 
  We provide both empirical and theoretical evidence to justify the use of random priors for the reward function. Experiments on multiple benchmarks showcase UBER's ability to learn effective and diverse behavior sets that enhance sample efficiency for online RL, outperforming existing baselines.
  By reducing reliance on human supervision, UBER broadens the applicability of RL to real-world scenarios with abundant reward-free data.

\end{abstract}

\section{Introduction}

Self-supervised learning has made substantial advances in various areas like computer vision and natural language processing \citep{openai2023gpt4, caron2021emerging, wang2020self} since it leverages large-scale data without the need of human supervision. Offline reinforcement learning has emerged as a promising framework for learning sequential policies from pre-collected datasets~\citep{kumar2020conservative, fujimoto2021minimalist, ma2021offline, yang2021believe, hu2022role}, but it is not able to directly take advantage of unsupervised data, as such data lacks reward signal for learning and often comes from different tasks and different scenarios. Nonetheless, reward-free data like experience from human players and corpora of human conversations is abundant while containing rich behavioral information, and incorporating them into reinforcement learning has a huge potential to help improve data efficiency and achieve better generalization.

How can we effectively utilize the behavioral information in unsupervised offline data for rapid online learning? In online settings, \citet{eysenbach2018diversity,sharma2019dynamics}  investigate the extraction of diverse skills from reward-free online environments by maximizing diversity objectives. Meanwhile, in offline settings, a rich body of literature exists that focuses on leveraging the dynamic information from reward-free datasets~\citep{yu2022leverage,hu2023provable}. However, the former approach is not directly applicable to offline datasets, while the latter approach overlooks the behavioral information in the dataset. ~\citet{ajay2020opal,singh2020parrot} employ generative models to extract behaviors in the offline dataset. However, using a behavior-cloning objective is conservative in nature and has a limited ability to go beyond the dataset to enrich diversity.

To bridge this gap, we propose \textbf{U}nsupervised \textbf{B}ehavior \textbf{E}xtraction via \textbf{R}andom Intent Priors (UBER), a simple and novel approach that utilizes random intent priors to extract useful behaviors from offline reward-free datasets and reuse them for new tasks, as illustrated in Figure~\ref{fig: illustration}. Specifically, our method samples the parameters of the reward function from a prior distribution to generate different intents for different agents. These pseudo-rewards encourage the agent to go beyond the dataset and exhibit diverse behaviors. It also encourages the behavior to be useful since each behavior is the optimal policy to accomplish some given intent. The acquired behaviors are then applied in conjunction with policy expansion~\citep{zhang2023policy} or policy reuse~\citep{zhang2022cup} techniques to accelerate online task learning. Surprisingly, we observe that random intent prior is sufficient to produce a wide range of useful behaviors, some of which even approach expert performance levels.
We provide theoretical justifications and empirical evidence for using random intent priors in Section~\ref{sec:theory} and~\ref{sec:exp}, respectively. 
% , as shown in Figure~\ref{fig: random reward td3+bc result}. 

Our experiments across multiple benchmarks demonstrate UBER's proficiency in learning behavior libraries that enhance sample efficiency for existing DRL algorithms, outperforming current baselines. By diminishing the reliance on human-designed rewards, UBER expands the potential applicability of reinforcement learning to real-world scenarios.

In summary, this paper makes the following contributions: (1) We propose UBER, a novel unsupervised RL approach that samples diverse reward functions from intent priors to extract diverse behaviors from reward-free datasets; (2) We demonstrate both theoretically and empirically that random rewards are sufficient to generate diverse and useful behaviors within offline datasets; (3) Our experiments on various tasks show that UBER outperforms existing unsupervised methods and enhances sample efficiency for existing RL algorithms.

\begin{figure}[t]
    \centering
    \includegraphics[scale=0.42]{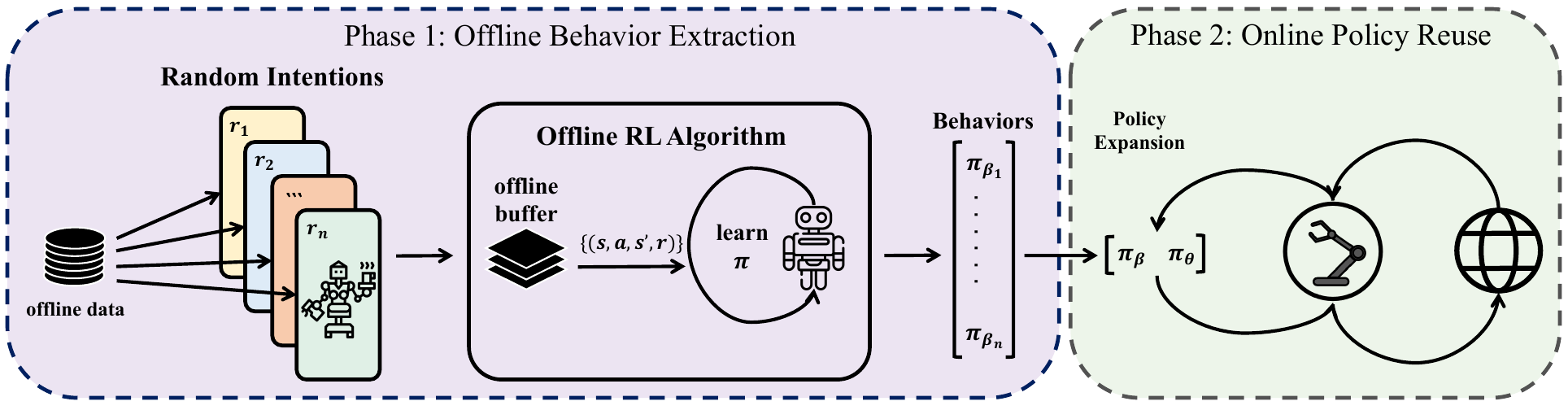}

    \caption{The framework of UBER. The procedure consists of two phases. In the first offline phase, we assign different reward functions to different agent to extract diverse and useful behaviors from the offline dataset. In the second phase, we reuse previous behavior by adding them to the candidate policy set to accelerate online learning for the new task.
    }
    \label{fig: illustration}
\end{figure}

\subsection{Related Works}

% \citet{kostrikov2021offline}

\paragraph{Unsupervised Offline RL.} \citet{yu2021conservative} considers reusing data from other tasks but assumes an oracle reward function for the new tasks. \citet{yu2022leverage} and \citet{hu2023provable} utilize reward-free data for improving offline learning but assume a labeled offline dataset is available for learning a reward function. \citet{yebecome, ghosh2023reinforcement} consider the setting where reward and action are absent. We focus on the reward-free setting and leverage offline data to accelerate online learning.

\paragraph{Unsupervised Behavior Extraction.} Many recent algorithms have been proposed for intrinsic behavioral learning without a reward. Popular methods includes prediction
methods~\citep{burda2018large,pathak2017curiosity,pathak2019self}, maximal entropy-based methods~\citep{campos2021beyond,liu2021behavior,liu2021aps,mutti2020policy,seo2021state,yarats2021reinforcement}, and maximal mutual
information-based methods~\citep{eysenbach2018diversity,hansen2019fast,liu2021aps,sharma2019dynamics}. However, they require an online environment to achieve the diversity objective. ~\citet{ajay2020opal,singh2020parrot} employ generative models to extract behaviors in the offline dataset but have a limited ability to enrich diversity.

\paragraph{Offline-to-online learning and policy reuse.} Offline-to-online RL has been popular recently since it provides a promising paradigm to improve offline further learned policies. \citet{nair2020awac} is among the first to propose a direct solution to offline-to-online RL. \citep{lee2022offline} proposes to use a balanced replay buffer and ensembled networks to smooth the offline-to-online transition; \citet{zhang2023policy} proposes to use the expanded policy for offline-to-online RL. \citet{ball2023efficient} proposes to reuse the off-policy algorithm with modifications like ensembles and layer normalization. ~\citet{zhang2022cup} proposes a critic-guided approach to reuse previously learned policies. \cite{zhang2023policy} proposes a similar and simpler way to reuse previous policies.

\paragraph{Imitation Learning} Our method is also related to imitation learning~\citep{hussein2017imitation}.  Popular imitation learning methods include behavior cloning methods \citet{pomerleau1988alvinn} and inverse RL  methods~\cite {ho2016generative,xiao2019wasserstein,dadashi2020primal}. However, they assume the demonstration in the dataset to be near optimal, while our method works well with datasets of mediocre quality.

% \paragraph{Offline Meta Learning}

\section{Preliminaries}
\subsection{Episodic Reinforcement Learning}
We consider finite-horizon episodic Markov Decision Processes (MDPs), defined by the tuple $(\cS,\cA,H,\cP,r),$ where $\cS$ is a state space, $\cA$ is an action space, $H$ is the horizon and $\cP=\{\cP_h\}_{h=1}^H, r=\{r_h\}^H_{h=1}$ are the transition function and reward function, respectively.\par

A policy $\pi =\{\pi_h\}_{h=1}^H$ specifies a decision-making strategy in which the agent chooses its actions based on the current state, i.e., $a_h \sim \pi_h(\cdot \given s_h)$. The value function $V^\pi_{h}: \cS \rightarrow \RR$ is defined as the sum of future rewards starting at state $s$ and step $h\in [H]$, and similarly, the Q-value function, i.e.
\begin{equation}
V^\pi_{h}(s) = \EE_{\pi}\Big[ \sum_{t=h}^{H} r_t(s_t, a_t)\Biggiven s_h=s  \Big],\quad Q^\pi_{h}(s,a) = \EE_\pi\Big[\sum_{t=h}^{H} r_h(s_t, a_t)\Biggiven s_h=s, a_h=a  \Big].
\label{eq:def_value_fct}
\end{equation}
where the expectation is with respect to the trajectory $\tau$ induced by policy $\pi$. 
% The action-value function $Q^\pi_h:\cS\times \cA\to \RR$ is similarly defined as
% \begin{equation}

% \label{eq:def_q_fct}
% \end{equation}

We define the Bellman operator as
\begin{align}
(\BB_h f)(s,a) &= \EE\bigl[r_h(s, a) + f(s')\bigr],
\label{eq:def_bellman_op}
\end{align}
for any $f:\mathcal{S}\rightarrow \mathbb{R}$ and $h\in [H]$.
The optimal Q-function~$Q^*$, optimal value function~$V^*$ and optimal policy $\pi^*$ are related by the Bellman optimality equation
\begin{align}
 V^*_{h}(s) = \max_{a\in \cA}Q^*_{h}(s,a),\quad Q^*_{h}(s,a) = (\BB_h V^*_{h}) (s,a),\quad \pi^*_{h} (\cdot \given s)=\argmax_{\pi}\EE_{a\sim \pi}{Q^*_{h}(s, a)}.
\label{eq:dp_optimal_values}
\end{align}
% while the optimal policy is defined as
% \begin{align*}
% \end{align*}

We define the suboptimality, as the performance difference of the optimal policy $\pi^*$ and the current policy $\pi_k$ given the initial state $s_1=s$. That is  
\begin{equation}
    \text{SubOpt}(\pi;s) = V^{\pi^*}_{1}(s) - V^{\pi}_{1}(s). \notag
    \label{eq:def_regret_2}
\end{equation}

\subsection{Linear Function Approximation}
To derive a concrete bound, we consider the \textit{linear MDP}~\citep{jin2020provably,jin2021pessimism} as follows, where the transition kernel and expected reward function are linear with respect to a feature map.

\begin{definition}[Linear MDP] \label{assumption:linear}
    $\rm{MDP}(\cS, \cA, H, \PP, r)$ is a \emph{linear MDP} with a feature map $\phi: \cS \times \cA \rightarrow \RR^d$, if for any $h\in [H]$, there exist $d$ \emph{unknown} (signed) measures $\mu_h = (\mu_h^{(1)}, \ldots, \mu_h^{(d)})$ over $\cS$ and an \emph{unknown} vector $z_h \in \RR^d$, such that 
     for any $(s, a) \in \cS \times \cA$, we have 
    % \vspace{-2ex}
    \begin{align}\label{eq:linear_transition}  
    % \forall (x, a) \in \cS\times \cA, \qquad 
    \PP_h(\cdot\given s, a) = \la\phi(s, a), \mu_h(\cdot)\ra, \qquad 
    r_h(s, a) = \la\phi(s, a), z_h\ra.  
    \end{align}
    Without loss of generality, we assume $\norm{\phi(s, a)} \le 1$ for all $(s,a ) \in \cS \times \cA$, and $\max\{\norm{\mu_h(\cS)}, \norm{z_h}\} \le \sqrt{d}$ for all $h \in [H]$.
\end{definition}

When emphasizing the dependence over the reward parameter $z$, we also use $r_h(s,a,z), V^{\pi}_h(s,z)$ and $V^*_h(s,z)$ to denote the reward function, the policy value and the optimal value under parameter $z$, respectively.

\section{Method}

\emph{How can we effectively leverage abundant offline reward-free data for improved performance in online tasks?} Despite the absence of explicit supervision signals in such datasets, they may contain valuable behaviors with multiple modes. These behaviors can originate from human expertise or the successful execution of previous tasks. To extract these behaviors, we assume that each behavior corresponds to specific intentions, which are instantiated by corresponding reward functions. Then, by employing a standard offline RL algorithm, we can extract desired behavior modes from the offline dataset using their associated reward functions. Such behavior can be reused for online fine-tuning with methods such as policy expansion~\citep{zhang2023policy} or policy reuse~\citep{zhang2022cup}. 

Then, the question becomes how to specify possible intentions for the dataset. Surprisingly, we found that employing a random prior over the intentions can extract diverse and useful behaviors effectively in practice, which form the offline part of our method. The overall pipeline is illustrated in Figure~\ref{fig: illustration}. Initially, we utilize a random intention prior to generating diverse and random rewards, facilitating the learning of various modes of behavior. Subsequently, standard offline algorithms are applied to extract behaviors with generated reward functions. Once a diverse set of behaviors is collected, we integrate them with policy reuse techniques to perform online fine-tuning.

Formally, we define $\cZ$ as the space of intentions, where each intention $z\in \cZ$ induces a reward function $r_z: \cS\times\cA \rightarrow \RR$ that defines the objective for an agent. Let the prior over intentions be $\beta\in \Delta(\cZ)$. In our case, we let $\cZ$ be the weight space $\cW$ of the neural network and $\beta$ be a distribution over the weight $w\in \cW$ of the neural network $f_w$ and sample $N$ reward functions from $\beta$ independently. Specifically, we have 

\$
r_i(s,a) = f_{w_i}(s,a),~ w_i \sim \beta, ~\forall i = 1,\ldots, N.
\$

We choose $\beta$ to be a standard initialization distribution for neural networks~\citep{glorot2010understanding,he2015delving}. 
% like He initialization~\citep{}. 
We choose TD3+BC~\citep{fujimoto2021minimalist} as our backbone offline algorithm, but our framework is general and compatible with any existing offline RL algorithms. Then, the loss function for the offline phase can be written as
\begin{align}
    \cL^{\rm offline}_{\rm critic}(\theta_i) &= \mathbb{E}_{(s,a,r_i,s')\sim\mathcal{D}_{i}^{\rm off}}\left[\left(r_i+\gamma Q_{\theta'_{i}}(s', \tilde{a})-Q_{\theta_{i}}(s, a)\right)^2 \right], \label{eqn:offline_critic}\\
    \cL^{\rm offline}_{\rm actor}(\phi_i) &= -\mathbb{E}_{(s,a)\sim\mathcal{D}_{i}^{\rm off}}\left[\lambda Q_{\theta_{i}}(s,\pi_{\phi_i}(s)) - (\pi_{\phi_i}(s)-a)^2 \right],\label{eqn:offline_actor}
\end{align}
where $\tilde{a}=\pi_{\phi'_i}(s')+\epsilon$ is the smoothed target policy and $\lambda$ is the weight for behavior regularization. We adopt a discount factor $\gamma$ even for the finite horizon problem as commonly used in practice~\citep{fujimoto2018addressing}.

As for the online phase, we adopt the policy expansion framework~\citep{zhang2023policy} to reuse policies. We first construct an expanded policy set $\tilde{\pi}=[\pi_{\phi_1},\ldots, \pi_{\phi_N}, \pi_w]$, where $\{{\phi_i}\}_{i=1}^N$ are the weights of extracted behaviors from offline phase and $\pi_w$ is a new randomly initialized policy. We then use the critic as the soft policy selector, which generates a distribution over candidate policies as follows: 

\#
\label{soft_policy_expansion}
P_{\tilde{\pi}}[i] = \frac{\exp {\left(\alpha \cdot Q(s,\tilde{\pi}_i(s))\right)}}{\sum_{j} \exp {\left(\alpha \cdot Q(s,\tilde{\pi}_j(s)\right))}}, \forall i \in\{1, \cdots, N+1\},
\#

where $\alpha$ is the temperature. We sample a policy $i \in\{1, \cdots, N+1\}$ according to $P_{\tilde{\pi}}$ at each time step and follows $\pi_{i}$ to collect online data. The policy $\pi_w$ and $Q_\theta$-value network undergo continuous training using the online RL loss, such as TD3~\citep{fujimoto2018addressing}):

\begin{align}
    \cL^{\rm online}_{\rm critic}(\theta) &= \mathbb{E}_{(s,a,r,s')\sim\mathcal{D}^{\rm on}}\left[\left(r+\gamma Q_{\theta'}(s', \tilde{a})-Q_{\theta}(s, a)\right)\right]^2, \label{eqn:online_critic}\\
    \cL^{\rm online}_{\rm actor}(w) &= -\mathbb{E}_{(s,a)\sim\mathcal{D}^{\rm on}}\left[ Q_{\theta}(s,\pi_{w}(s))\right].\label{eqn:online_actor}
\end{align}

The overall algorithm is summarized in Algorithm~\ref{alg: offline rl} and \ref{alg: online rl}.
We highlight elements important to our approach in \textcolor{purple}{purple}.

\begin{algorithm}[ht]
    \caption{Phase 1: Offline Behavior Extraction}
    \label{alg: offline rl}
    \begin{algorithmic}[1]
        \STATE {\bf Require}: Behavior size $N$, offline reward-free dataset $\cD^{\rm off}$, prior intention distribution $\beta$
        \STATE Initialize parameters of $N$ independent offline agents $\{Q_{\theta_i}, \pi_{\phi_i}\}_{i=1}^{N}$
        \FOR{$i=1, \cdots, N$}
        \STATE \textcolor{purple}{Sample a reward function from priors $z_i \sim \beta$}
        \STATE \textcolor{purple}{Reannotate $\cD^{\rm off}$ as $\cD_{i}^{\rm off}$ with reward $r_{z_i}$}
        \FOR{each training iteration}
        \STATE Sample a random minibatch $\{\tau_j\}_{j=1}^{B}\sim \cD_{i}^{\rm off}$
        \STATE Calculate $\cL_{\rm critic}^{\rm offline}(\theta_i)$ as in Equation~\eqref{eqn:offline_critic} and update $\theta_i$
        \STATE Calculate $\cL_{\rm actor}^{\rm offline}(\phi_i)$ as in Equation~\eqref{eqn:offline_actor} and update $\phi_i$
        \ENDFOR
        \ENDFOR\\
        \STATE \textbf{Return} $ \{\pi_{\phi_{i}}\}_{i=1}^{N}$
    \end{algorithmic}
    \end{algorithm}
    
    \begin{algorithm}[ht]
        \caption{Phase 2: Online Policy Reuse}
        \label{alg: online rl}
        \begin{algorithmic}[1]
            \STATE {\bf Require}: $ \{\pi_{\phi_{i}}\}_{i=1}^{N}$, offline dataset $\cD^{\rm off}$, update-to-data ratio $G$
            \STATE Initialize online agents $Q_{\theta}, \pi_{w}$ and replay buffer $\cD^{\rm on}$
            \STATE Construct expanded policy set $\tilde{\pi} = [\pi_{\phi_1},\ldots, \pi_{\phi_N}, \pi_w]$
            \FOR{each episode}
            \STATE Obtain initial state $s_1$ from the environment 
            \FOR{step $t=1, \cdots, T$}
            \STATE \textcolor{purple}{Construct $P_{\tilde{\pi}}$ According to Equation~\eqref{soft_policy_expansion}}
            \STATE \textcolor{purple}{Pick an policy to act $\pi_t\sim P_{\tilde{\pi}} $, $a_t \sim \pi_t(\cdot|s_t)$}
            \STATE Store transition $(s_t, a_t, r_t, s_{t+1})$ in $\cD^{\rm on}$
            \FOR{$g=1,\cdots,G$}
            \STATE Calculate $\cL_{\rm critic}^{\rm online}(\theta)$ as in Equation~\eqref{eqn:online_critic} and update $\theta$
            \ENDFOR
            \STATE Calculate $\cL_{\rm actor}^{\rm online}(w)$ as in Equation~\eqref{eqn:online_actor} and update $w$
            \ENDFOR
            \ENDFOR
        \end{algorithmic}
    \end{algorithm}

\section{Theoretical Analysis}
\label{sec:theory}
\emph{How sound is the random intent approach for behavior extraction?}
This section aims to establish a theoretical characterization of our proposed method. We first consider the completeness of our method. We show that any behavior can be formulated as the optimal behavior under some given intent, and any behavior set that is covered by the offline dataset can be learned effectively given the corresponding intent set. Then we consider the coverage of random rewards. We show that with a reasonable number of random reward functions, the true reward function can be approximately represented by the linear combination of the random reward functions with a high probability. Hence our method is robust to the randomness in the intent sampling process.

\subsection{Completeness of the Intent Method for Behavior Extraction}
We first investigate the completeness of intent-based methods for behavior extraction. Formally, we have the following proposition.
\begin{proposition}
    \label{intention_completeness}
    For any behavior $\pi$, there exists an intent $z$ with reward function $r(\cdot,\cdot,z)$ such that $\pi$ is the optimal policy under $r(\cdot,\cdot,z)$. That is, for any $\pi=\{\pi_h\}_{h=1}^H, \pi_h: \cS\rightarrow \Delta(\cA)$, there exists $z\in \cZ$ such that
    \$
    V^{\pi}_h(s,z) = V^*_h(s,z),~ \forall (s,a,h) \in \cS\times\cA\times [H]. 
    \$
\end{proposition}
\begin{proof}
    Please refer to Appendix~\ref{appenix: proof1} for detailed proof.
\end{proof}

Proposition~\ref{intention_completeness} indicates that any behavior can be explained as achieving some intent. This is intuitive since we can relabel the trajectory generated by any policy as successfully reaching the final state in hindsight~\citep{andrychowicz2017hindsight}. 
 With Proposition~\ref{intention_completeness}, we can assign an intent $z$ to any policy $\pi$ as $\pi_z$. However, it is unclear whether we can learn effectively from offline datasets given an intent set $\cZ$. This is guaranteed by the following theorem.

\begin{theorem}
    \label{theorem:1}
    Consider linear MDP as defined in Definition~\ref{assumption:linear}. With an offline dataset $\cD$ with size $N$, and the PEVI algorithm~\citep{jin2021pessimism}, the suboptimality of learning from an intent $z\in \cZ$ satisfies
    \#
        \label{eq:main_result}
        \text{\rm SubOpt}(\widehat{\pi};s,z) = \cO\left(\sqrt{\frac{C^\dagger_z d^2H^3\iota}{N}}\right),
    \#
    for sufficiently large $N$, where $\iota= \log(d N|\cZ|/\delta)$, $c$ is an absolute constant and 
    \$
    C^{\dagger}_{z} = \max_{h\in[H]} \sup_{\|x\|=1}\frac{x^{\top}\Sigma_{\pi_z,h} x}{x^{\top}\Sigma_{\rho_h} x}, 
    \$
    with 
    $
    \Sigma_{\pi_z,h} = \EE_{(s,a)\sim d_{\pi_z,h}(s,a)} [\phi(s,a)\phi(s,a)^{\top}],~ \Sigma_{\rho_h} = \EE_{\rho_h} [\phi(s,a)\phi(s,a)^{\top}].
    $
\end{theorem}
\begin{proof}
    Please refer to Appendix~\ref{appenix: proof2} for detailed proof.
\end{proof}

$C^{\dagger}_z$ represents the maximum ratio between the density of empirical state-action distribution $\rho$ and the density $d_{\pi_z}$ induced
from the policy $\pi_z$. Theorem~\ref{theorem:1} states that, for any behavior, $\pi_z$ that is well-covered by the offline dataset $\cD$, we can use standard offline RL algorithm to extract it effectively, in the sense that learned behavior $\widehat{\pi}$ has a small suboptimality under reward function $r(\cdot,\cdot,z)$. 
Compared to standard offline result~\citep{jin2021pessimism}, Theorem~\ref{theorem:1} is only worse off by a factor of $\log |\cZ|$, which enables us to learn multiple behaviors from a single dataset effectively, as long as the desired behavior mode is well contained in the dataset. 

\subsection{Coverage of Random Rewards}
This section investigates the coverage property of random rewards. Intuitively, using random rewards may suffer from high variance, and it is possible that all sampled intents lead to undesired behaviors. Theorem~\ref{theorem:2} shows that such cases rarely happen since with a reasonable number of random reward functions, the linear combination of them can well cover the true reward function with a high probability. 
% policies are robust to perturbations in the reward function.

\begin{theorem}
\label{theorem:2}
    Assume the reward function $r(s,a)$ admits a RKHS represention $\psi(s,a)$ with $\|\psi(s,a)\|_\infty \leq \kappa$ almost surely. Then with $N=c_0 \sqrt{M}\log(18\sqrt{M}\kappa^2/\delta)$ random reward functions $\{r_{i}\}_{i=1}^N$, the linear combination of the set of random reward functions $\widehat{r}(s,a)$ can approximate the true reward function with error 

$$
\E_{(s,a)\sim \rho} [\widehat{r}(s,a) - r(s,a)]^2 \leq c_1 \log^2(18/\delta)/\sqrt{M},
$$

with probability $1-\delta$, where $M$ is the size of the offline dataset $\cD$, $c_0$ and $c_1$ are universal constants and $\rho$ is the distribution that generates the offline dataset $\cD$.
\end{theorem} 
% \begin{proposition}
%     \label{theorem:2}
%     For any two reward function $r_1,r_2$ that are close to each other such that $\|r_1(\cdot, \cdot)-r_2(\cdot,\cdot)\|_\infty \leq \varepsilon$, the corresponding optimal policy are also close in value functions evaluated on $r_1$, i.e.
%     \$
%     \|V^{\pi_1}_{r_1}(\cdot) - V^{\pi_2}_{r_1}(\cdot)\|_\infty \leq 2 H \varepsilon,
%     \$
%     where $\pi_1,\pi_2$ are the optimal policies under $r_1$ and $r_2$, respectively.
% \end{proposition}
\begin{proof}
    Please refer to Appendix~\ref{appenix: proof3} for detailed proof.
\end{proof}

Theorem~\ref{theorem:2} indicates that RL algorithms are insensitive to the randomness in the reward function used for learning, as long as we are using a reasonable number of such functions. This phenomenon is more significant in offline algorithms since they are conservative in nature to avoid over-generalization. This partly explains why using random rewards is sufficient for behavior extraction, and can even generate policies comparable to oracle rewards. Such a phenomenon is also observed in prior works~\citep{shin2023benchmarks,li2023survival}.

\section{Experiments}
\label{sec:exp}
\begin{figure}[t]
    \centering
    \subfigure{\includegraphics[scale=0.21]{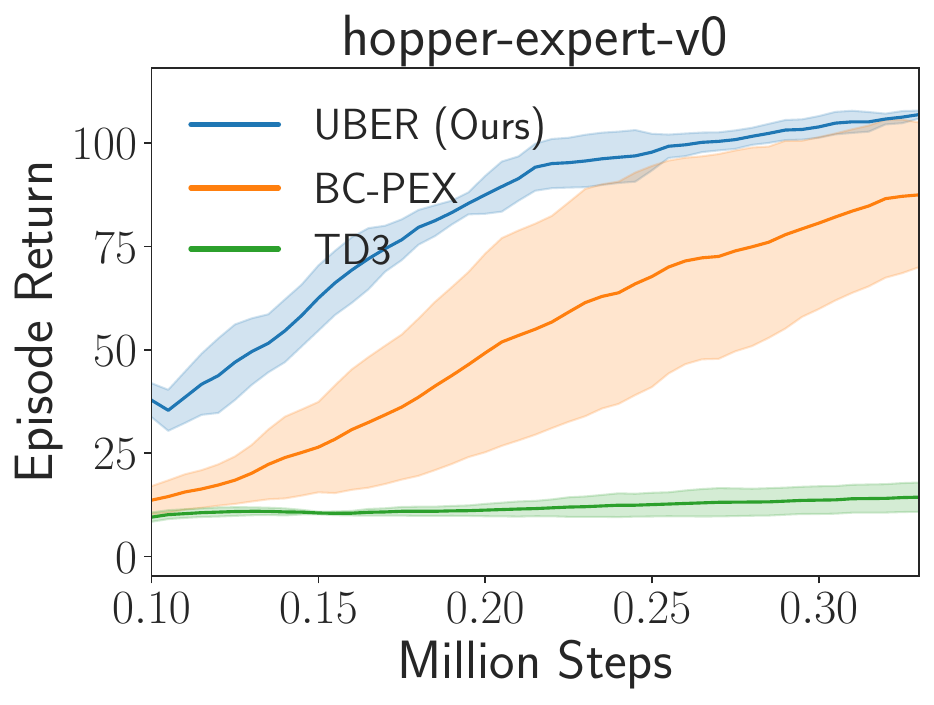}}\subfigure{\includegraphics[scale=0.21]{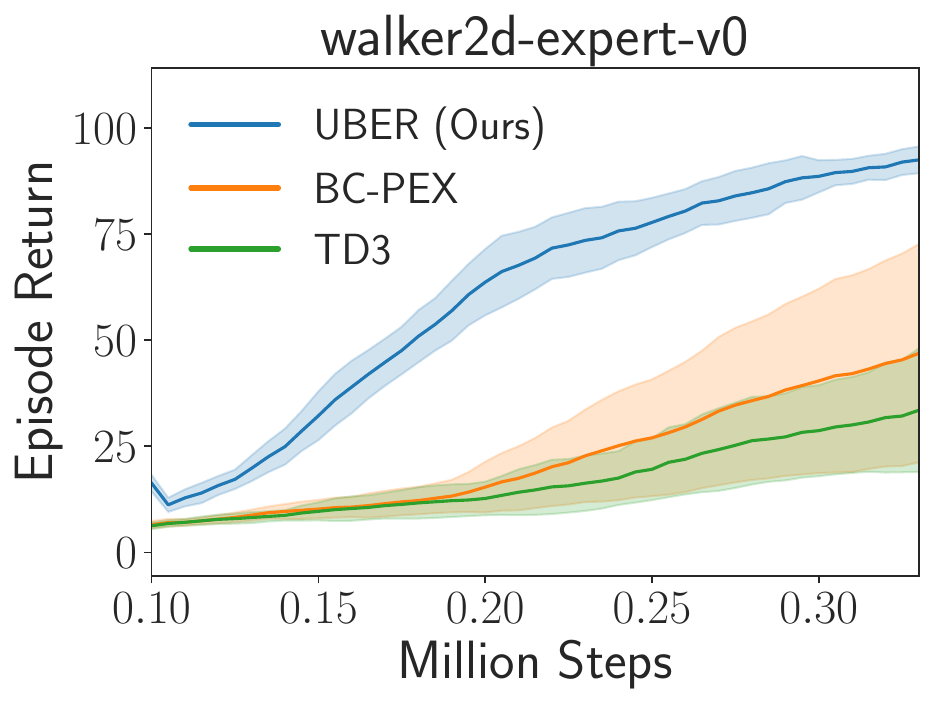}}
    \subfigure{\includegraphics[scale=0.21]{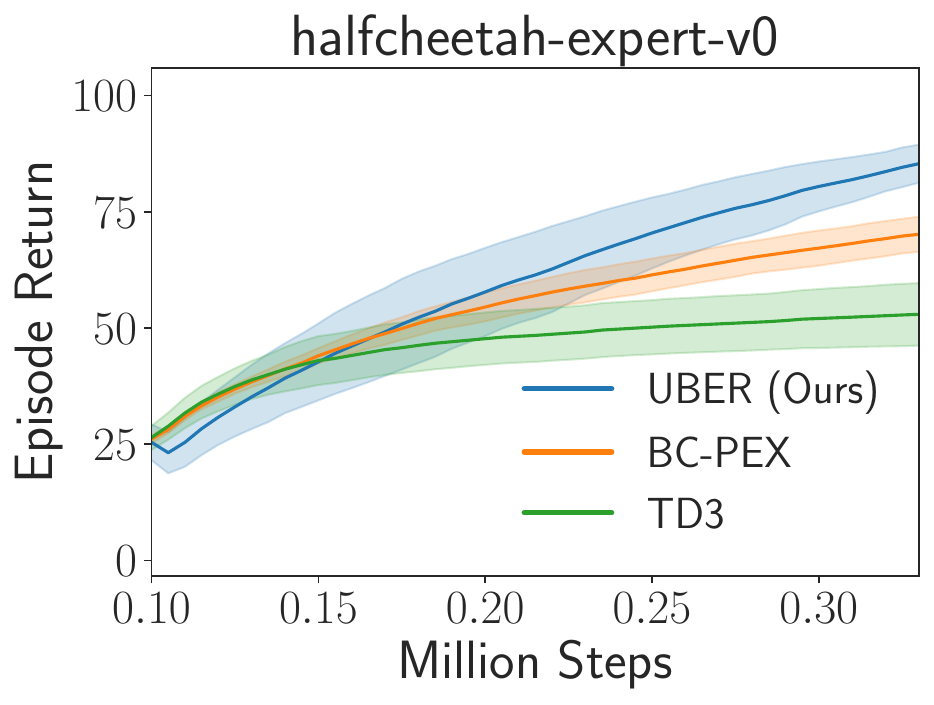}}    
    \subfigure{\includegraphics[scale=0.21]{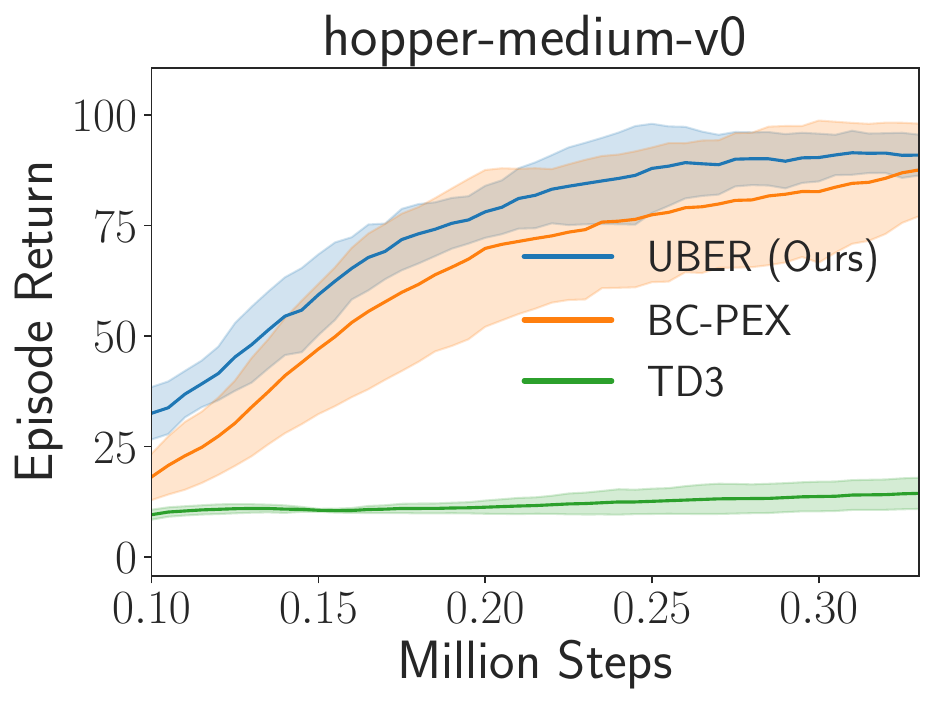}}
    
    \subfigure{\includegraphics[scale=0.21]{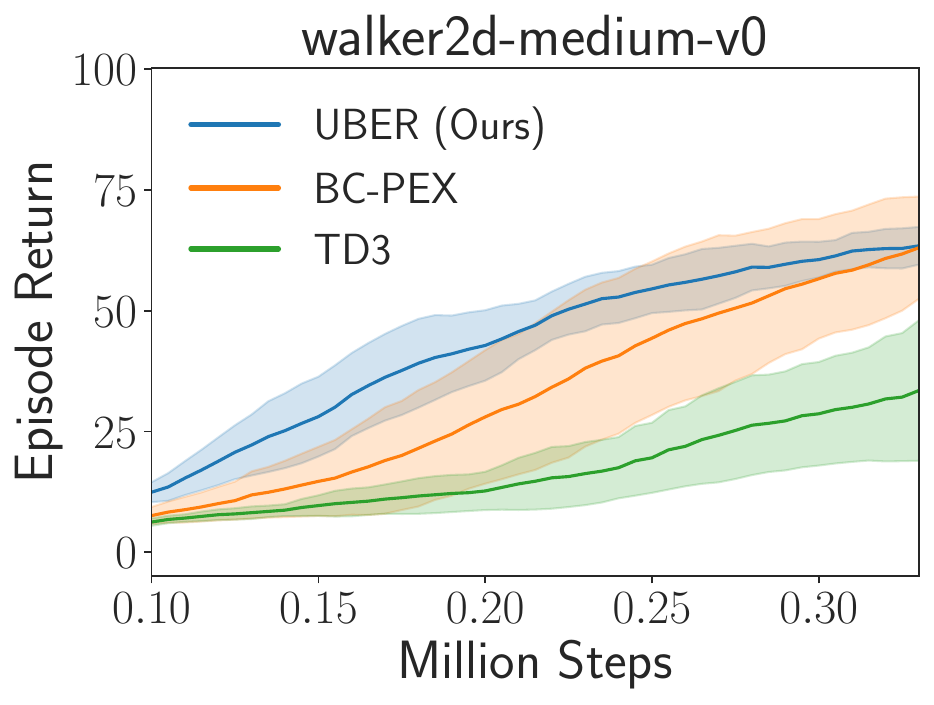}}
    \subfigure{\includegraphics[scale=0.21]{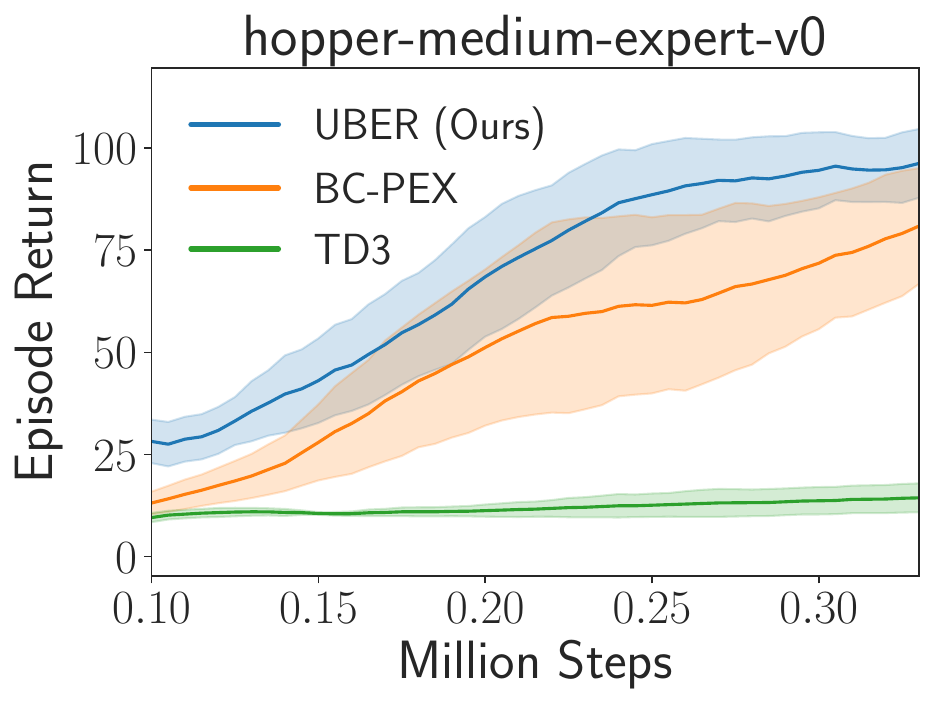}}
    \subfigure{\includegraphics[scale=0.21]{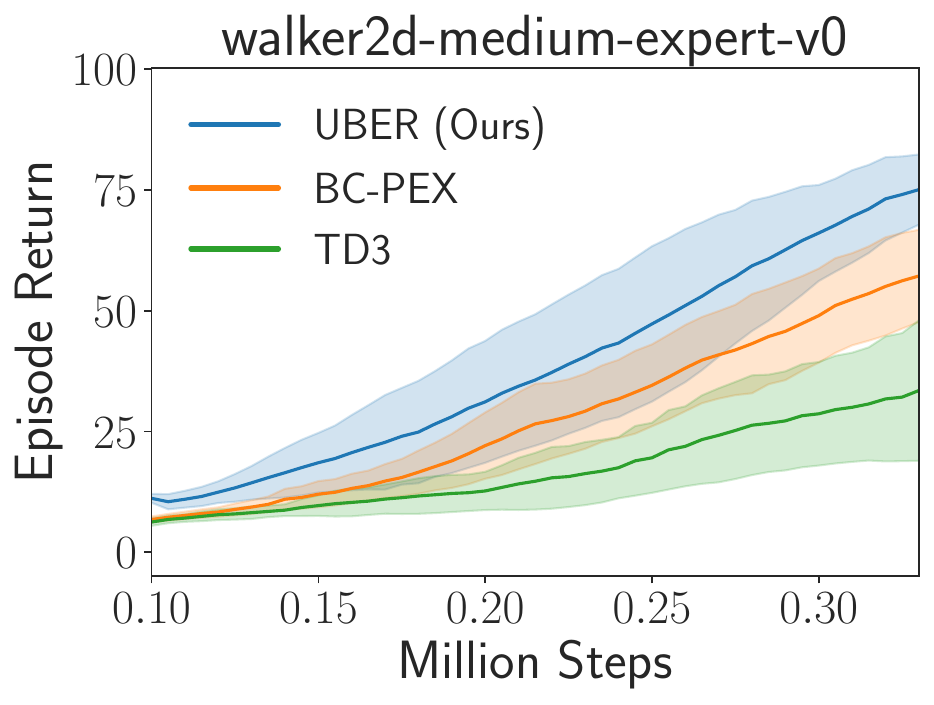}}
    \subfigure{\includegraphics[scale=0.21]{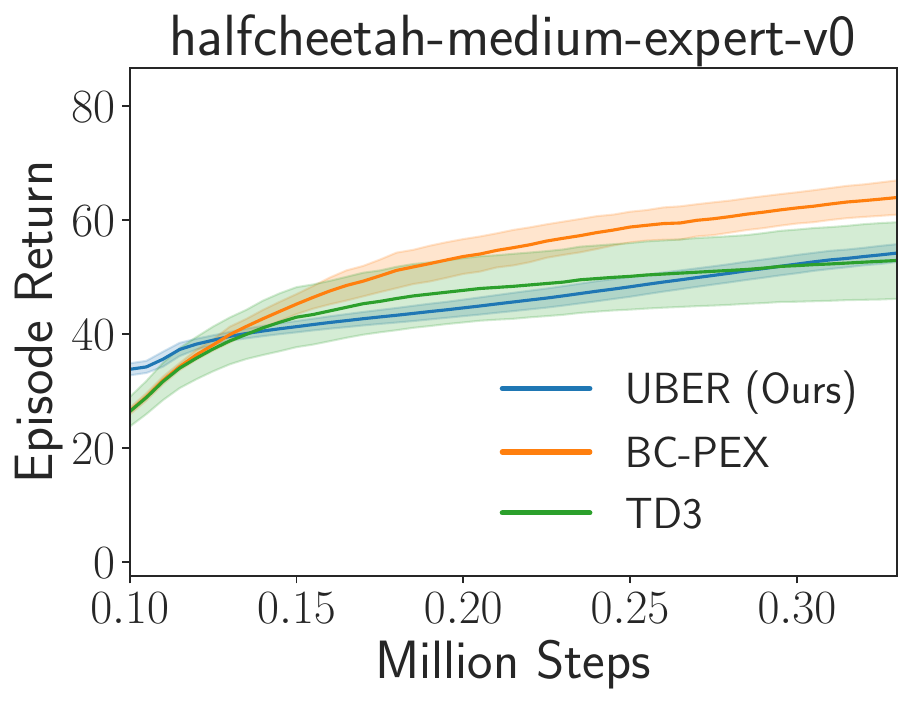}}
    \caption{Comparison between UBER and baselines in the online phase in the Mujoco domain.
    We adopt datasets of various quality for offline training.
    We adopt a normalized score metric averaged with five random seeds.}
    \label{fig: mujoco result}
\end{figure}

\begin{figure}[t]
    \centering
    % \subfigure{\includegraphics[scale=0.21]{new_baseline/hopper-medium-v0_newbaseline.pdf}}
    \subfigure{\includegraphics[scale=0.27]{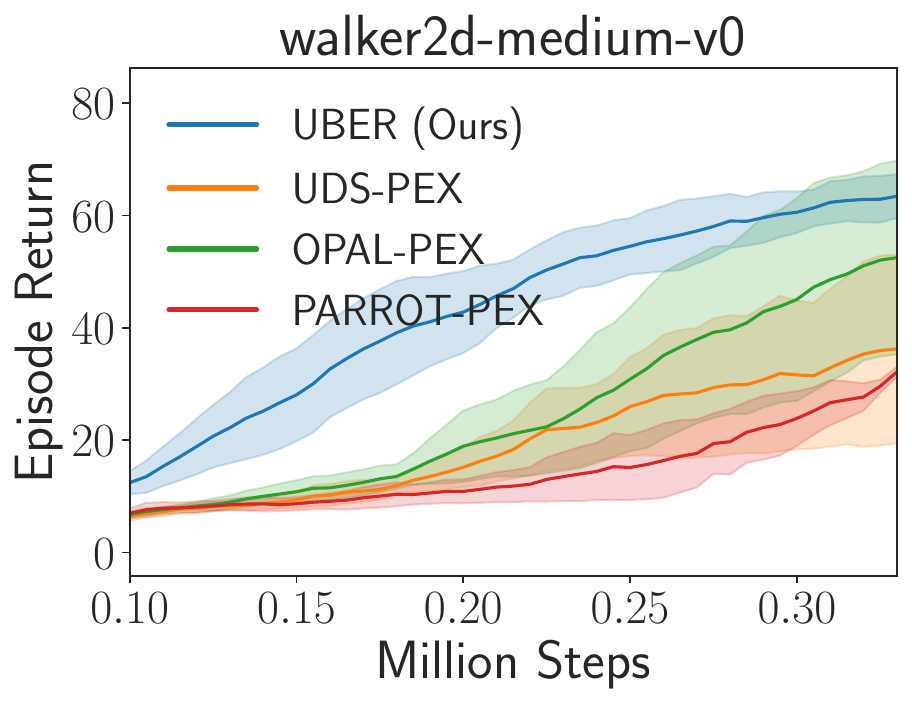}}
    \subfigure{\includegraphics[scale=0.27]{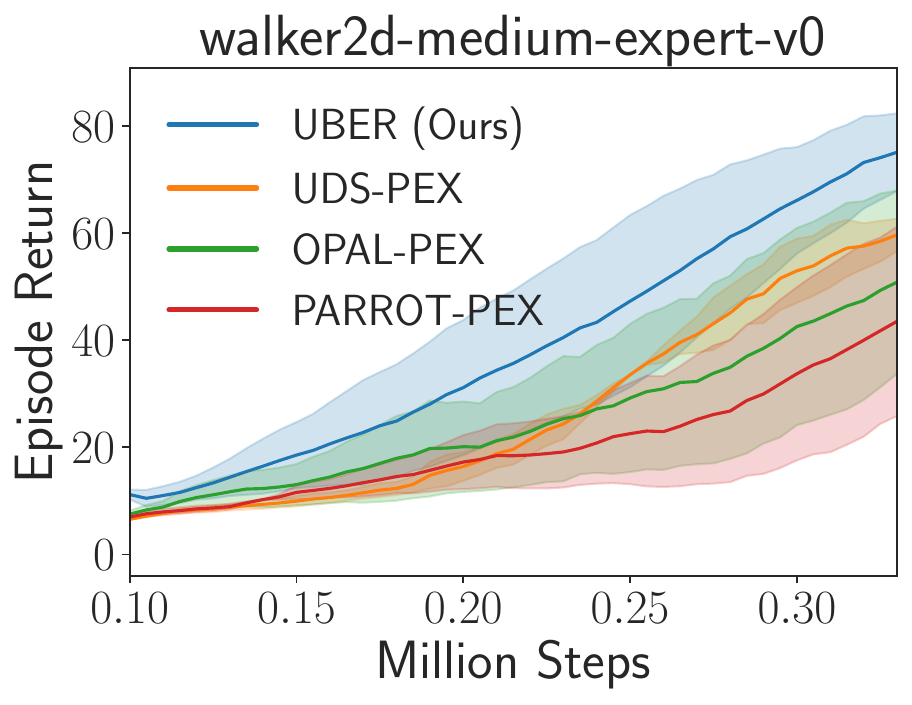}}
    \subfigure{\includegraphics[scale=0.27]{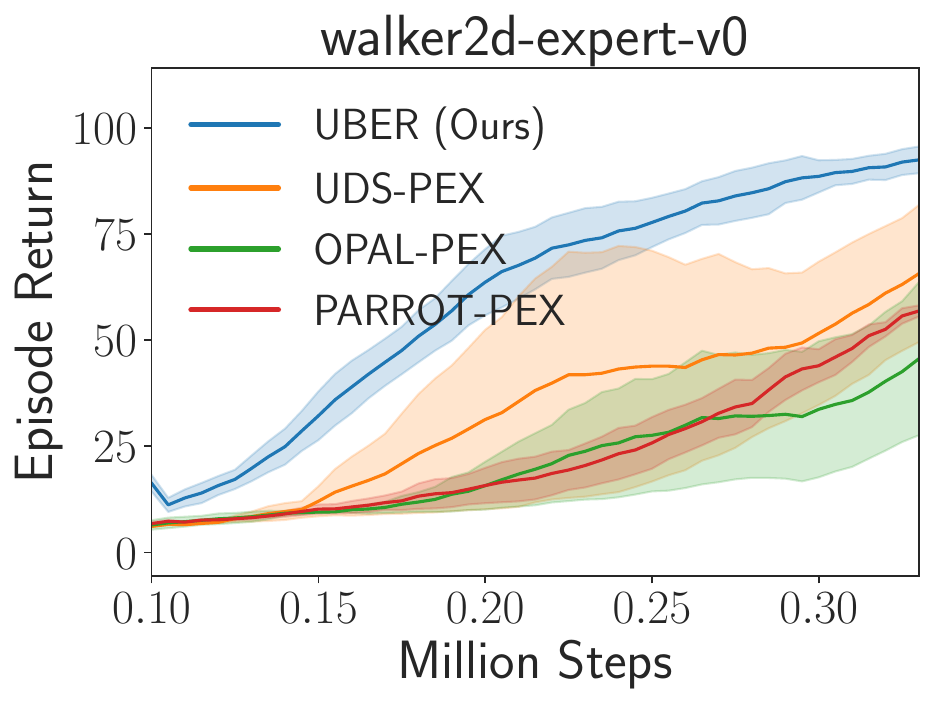}}
    % \subfigure{\includegraphics[scale=0.21]W{new_baseline/hopper-medium-expert-v0_newbaseline.pdf}}
    % \subfigure{\includegraphics[scale=0.21]{new_baseline/hopper-expert-v0_newbaseline.pdf}}    
    % \subfigure{\includegraphics[scale=0.21]{new_baseline/halfcheetah-medium-v0_newbaseline.pdf}}
    % \subfigure{\includegraphics[scale=0.21]{new_baseline/halfcheetah-medium-expert-v0_newbaseline.pdf}}
    \caption{Comparison between UBER and baselines including unsupervised behavior extraction (PARROR OPAL) and data-sharing methods (UDS). The result is averaged with five random seeds. Our method outperforms these baselines by leveraging the diverse behavior set.}
    \label{fig: new baseline result}
\end{figure}

\begin{figure}[t]
    \centering
    \subfigure{\includegraphics[scale=0.22]{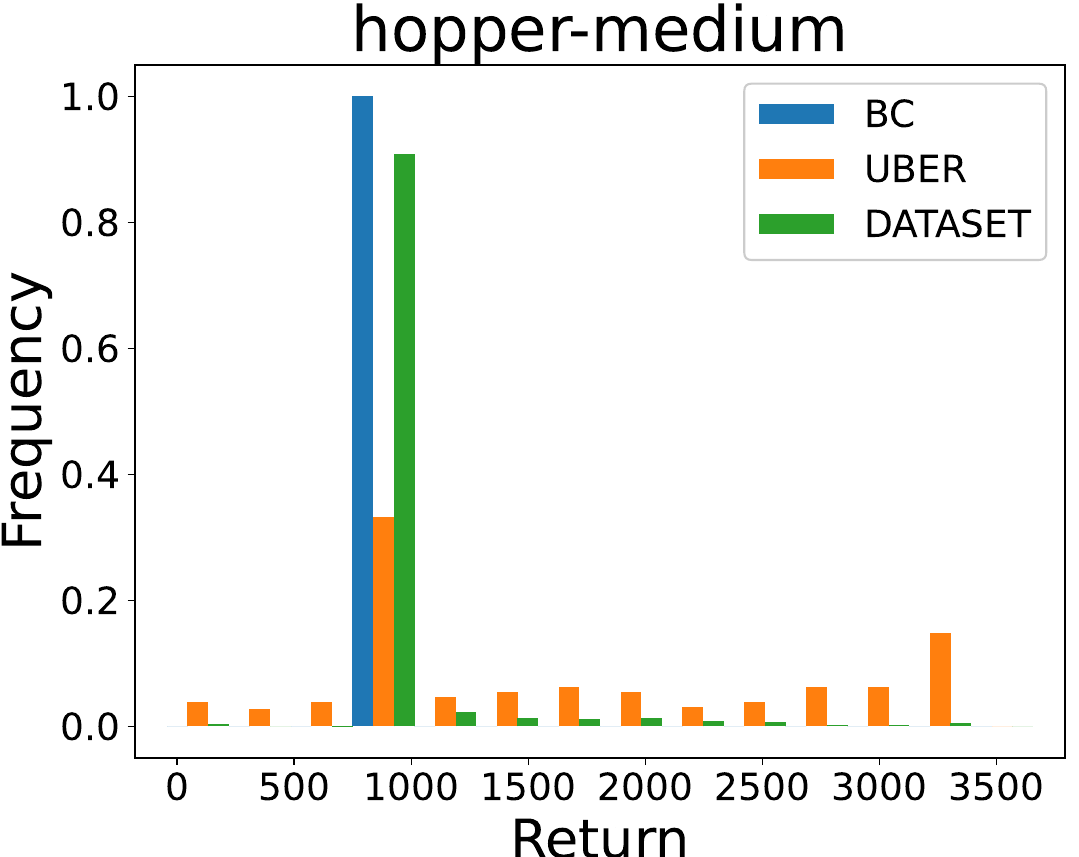}}
    \subfigure{\includegraphics[scale=0.22]{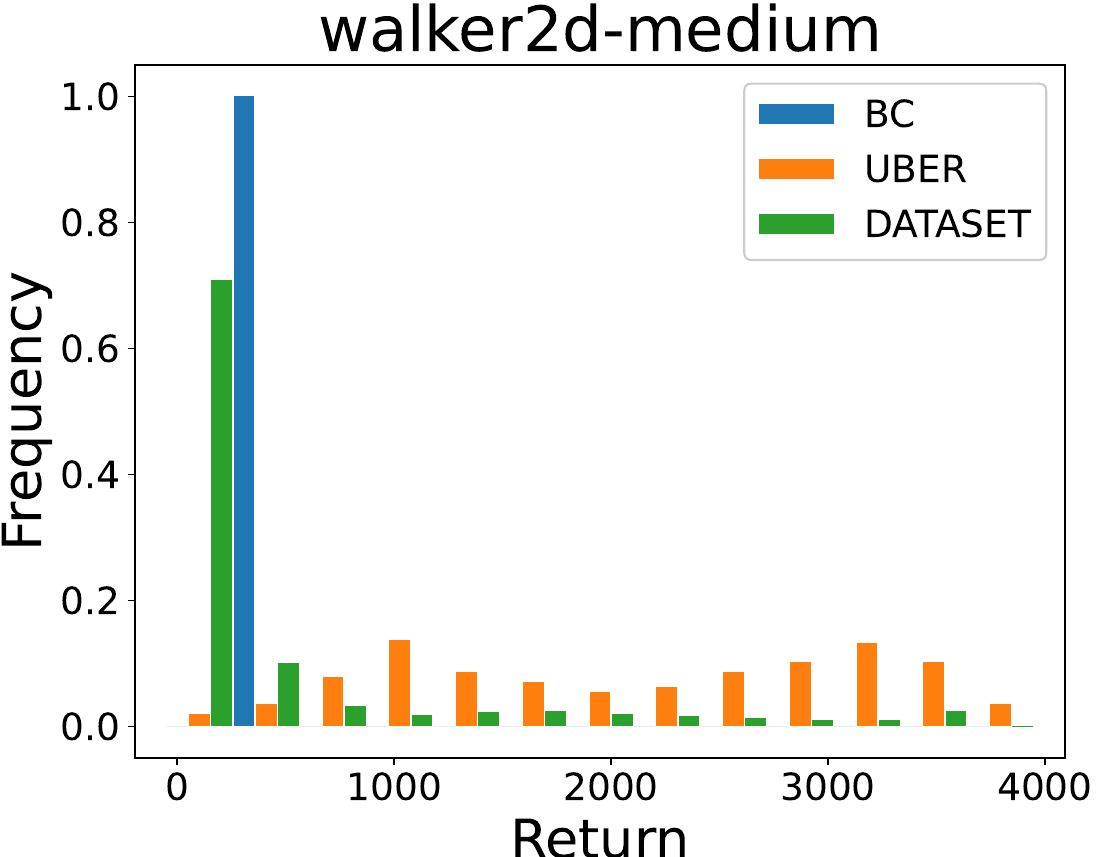}}
    \subfigure{\includegraphics[scale=0.22]{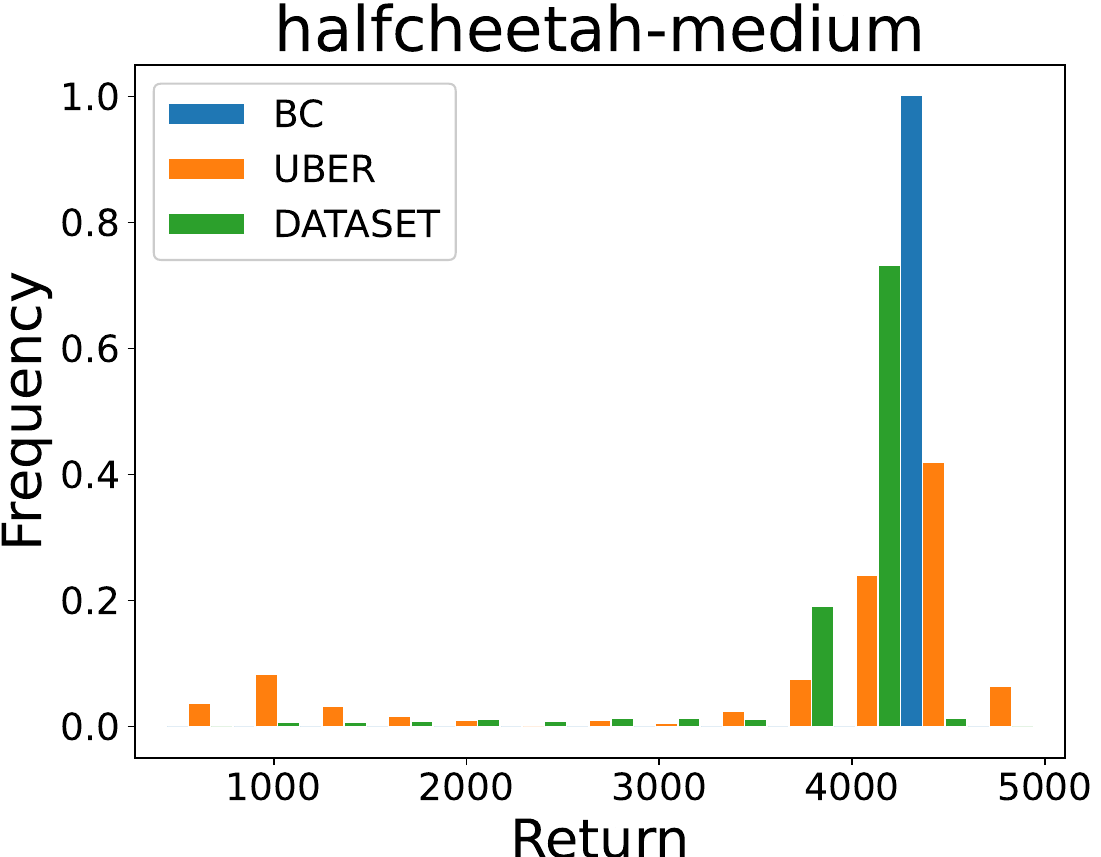}}
    \subfigure{\includegraphics[scale=0.22]{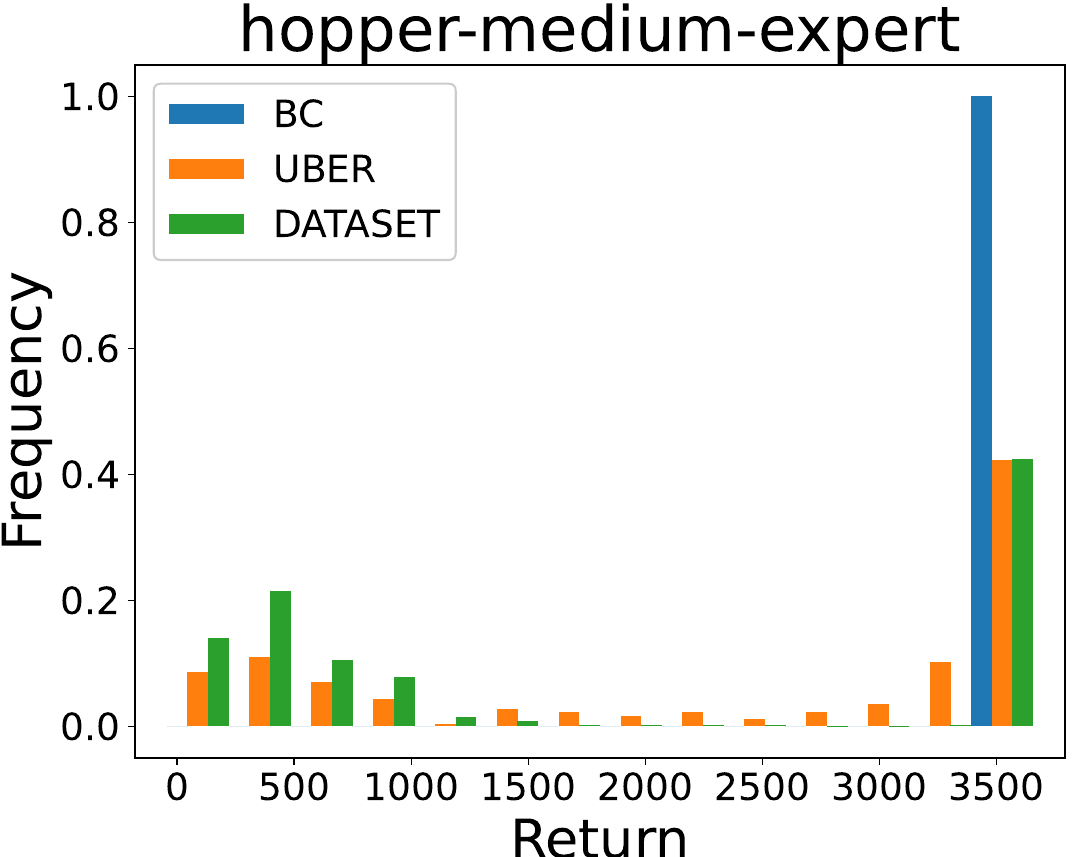}}
    % \subfigure{\includegraphics[scale=0.22]{distribution/hopper-expert.pdf}}
    \subfigure{\includegraphics[scale=0.22]{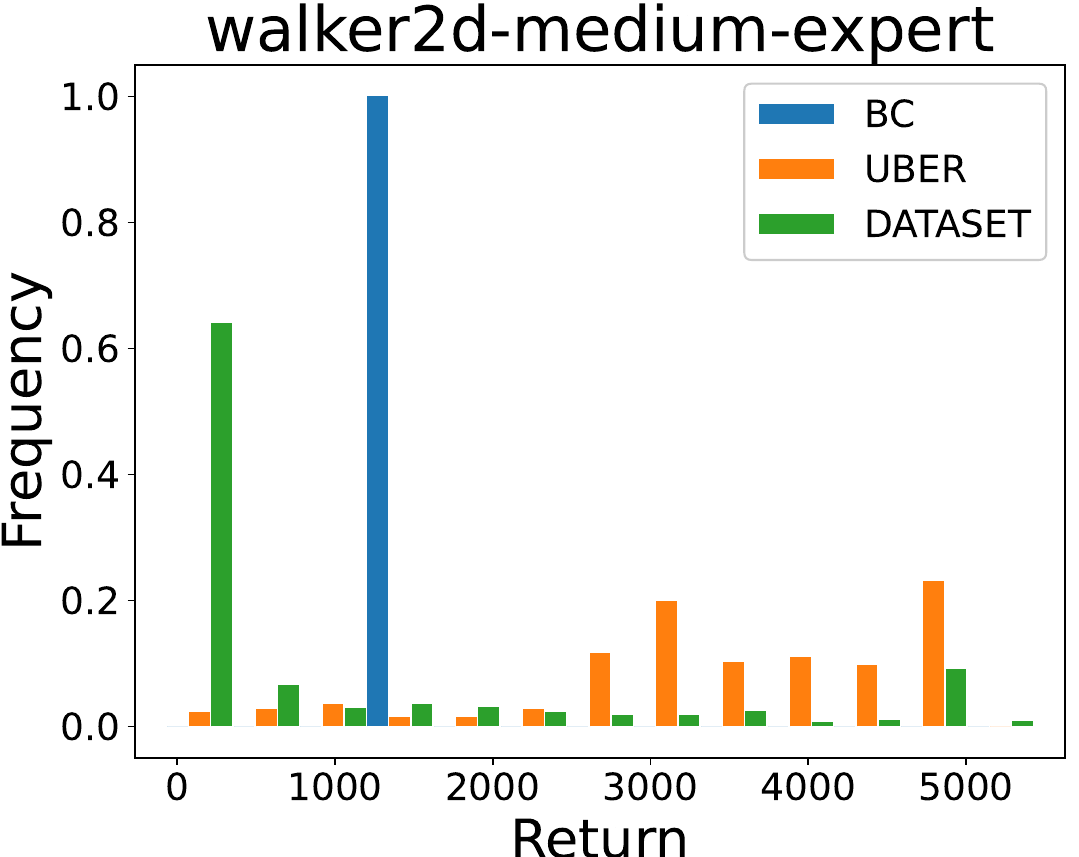}}
    %\subfigure{\includegraphics[scale=0.22]{distribution/walker2d-expert.pdf}}
    \subfigure{\includegraphics[scale=0.22]{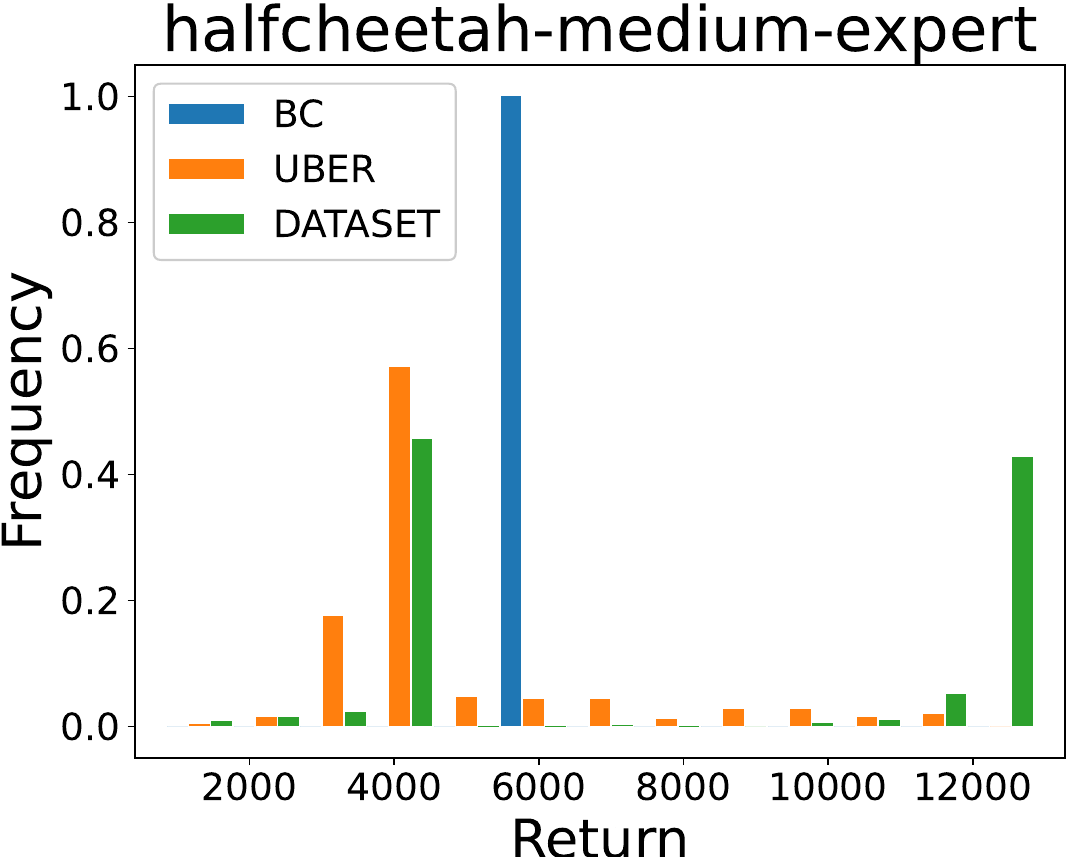}}
    % \subfigure{\includegraphics[scale=0.22]{distribution/antmaze-medium-play.pdf}}
    % \subfigure{\includegraphics[scale=0.22]{distribution/antmaze-medium-diverse.pdf}}
    \caption{The return distribution of the dataset (DATASET) and different behaviors generated from random intent priors (UBER) and behavior cloning (BC). UBER can generate a diverse behavior set that span a return distribution wider than the original dataset.
    }
    \label{fig: random intent sta}
\end{figure}

Our evaluation studies UBER as a pre-training mechanism on reward-free data, focusing on the following
questions: (1) Can we learn useful and diverse behaviors with UBER? (2) Can the learned behaviors from one task using UBER be reused for various downstream tasks? (3) How effective is each component in UBER?
% \begin{enumerate}
    % \item How well can the learned behaviors be used for various downstream tasks?
    % \item Does the random reward generation mechanism remain valid in complex domains like visual motor control?
    % \item How effective is each component in UBER?
    % Whether UBER is general and can be effectively combined with other off-the-shelf offline and online RL algorithms?
% \end{enumerate}

To answer the questions above, we conduct experiments on the standard 
D4RL benchmark~\citep{fu2020d4rl} and the multi-task benchmark Meta-World~\citep{yu2020meta}, which encompasses a variety of dataset settings and tasks. 
% We first extract offline behavior according to Algorithm~\ref{alg: offline rl} on the D4RL benchmark~\citep{fu2020d4rl} of OpenAI gym MuJoCo tasks and multi-task benchmark Meta-World~\citep{yu2020meta}, which encompasses a variety of dataset settings and domains.Then, we conduct online policy reuse based on Algorithm~\ref{alg: online rl}.
We adopt the normalized score metric proposed by the D4RL benchmark, and all experiments are averaged over five random seeds.
Please refer to Appendix~\ref{appendix: exp details} for more experimental details.

\subsection{Experiment Setting}
We extract offline behavior policies from multi-domains in D4RL benchmarks, including locomotion and navigation tasks.
Specifically, The locomotion tasks feature online data with varying levels of expertise.
The navigation task requires composing parts of sub-optimal trajectories to form more optimal policies for reaching goals on a MuJoco Ant robot.
% The visual motor control is particularly challenging as the behavior policies are state-based, which means the data is partially observable in pixel space.

\paragraph{Baselines.}
We compare UBER with baselines with various behavior extraction methods. We compare our method with BC-PEX, which uses a behavior-cloning objective and other unsupervised behavior extraction methods~\citep{yang2022flow}, including OPAL~\citep{ajay2020opal} and PARROT~\citep{singh2020parrot}. We also compare our method with a strong offline-to-online method, RLPD~\citep{ball2023efficient} and an unsupervised data sharing method, UDS~\citep{yu2022leverage}. 
% We first adopt BC-PEX and RLPD~\citep{ball2023efficient} as our baselines for prior-based online algorithms.
% Further, we compare our method with prior offline unsupervised behavior extraction methods~\citep{yang2022flow}~(e.g., OPAL~\citep{ajay2020opal} and PARROT~\citep{singh2020parrot}) and prior unsupervised data sharing methods~(e.g., UDS~\citep{yu2022leverage}). 
Specifically, BC-PEX first extracts behavior from the offline dataset based on behavior cloning and then conducts online policy reuse with PEX~\citep{zhang2023policy} following the schedule of Algorithm~\ref{alg: online rl}.
As for OPAL-PEX and PARROT-PEX, we extract behavior based on the VAE and FLOW models, and then conduct the same online policy reuse pipeline.
We set the reward of the offline dataset to zero to implement UDS-PEX.
As for the pure online baseline, we use RLPD without prior offline data, named RLPD~(no prior) as our baseline.
% Note that the prior offline data used in RLPD must contain the true reward value, so we use it as the Oracle.

% As for the online RL algorithms, we adopt TD3~\citep{fujimoto2018addressing} and RLPD~\citep{ball2023efficient} as our baselines. In addition, we compare UBER with RLPD without prior offline data, named RLPD~(no prior).

% \begin{figure}[t]
%     \centering\subfigure{\includegraphics[scale=0.25]{antmaze/antmaze-umaze-v2.pdf}}\subfigure{\includegraphics[scale=0.25]{antmaze/antmaze-medium-play-v2.pdf}}\subfigure{\includegraphics[scale=0.25]{antmaze/antmaze-large-play-v2.pdf}}
    
%     \subfigure{\includegraphics[scale=0.25]{antmaze/antmaze-umaze-diverse-v2.pdf}}\subfigure{\includegraphics[scale=0.25]{antmaze/antmaze-medium-diverse-v2.pdf}}
%     \subfigure{\includegraphics[scale=0.25]{antmaze/antmaze-large-diverse-v2.pdf}}
%     \caption{Comparison between UBER and baselines in the online phase in the Antmaze domain.
%     Oracle is the performance of RLPD at 500k step.
%     We adopt a normalized score metric averaged with three random seeds.
%     }
%     \label{fig: antmaze result}
% \end{figure}

\begin{table*}[t]
    \centering
    \begin{tabular}{c|ccc|c}
        \toprule
          \emph{500k step scores} & RLPD~(w$\backslash$o true reward) & BC-PEX & UBER~(Ours) & RLPD (w true reward) \\
         \midrule
            umaze & 0.0$\pm$0.0 & \textbf{99.8$\pm$0.3} & 97.9$\pm$0.3 & 99.9$\pm$0.1\\
           umaze-diverse & 0.0$\pm$0.0 & \textbf{99.9$\pm$0.1} & 98.8$\pm$0.5 & 99.9$\pm$0.1\\
           medium-play & 0.0$\pm$0.0 & 0.0$\pm$0.0 & \textbf{94.0$\pm$3.1} & 98.7$\pm$0.9\\
           medium-diverse & 0.0$\pm$0.0 & 0.0$\pm$0.0 & \textbf{96.5$\pm$1.5} & 98.5$\pm$1.3\\
           large-play & 0.0$\pm$0.0 & 0.0$\pm$0.0 & \textbf{83.3$\pm$7.7} & 94.8$\pm$1.5\\
           large-diverse &  0.0$\pm$0.0 & 0.0$\pm$0.0 & \textbf{88.7$\pm$1.5} & 93.5$\pm$1.6\\
        \midrule
          \emph{100k step scores} &  RLPD~(w$\backslash$o true reward) & BC-PEX & UBER~(Ours) & RLPD (w true reward) \\
         \midrule
            umaze  & 0.0$\pm$0.0 & \textbf{95.8$\pm$1.4} & 90.7$\pm$1.8 & 82.5$\pm$0.2\\
           umaze-diverse  & 0.0$\pm$0.0 & \textbf{95.5$\pm$3.7} & \textbf{94.5$\pm$1.1} & 83.3$\pm$0.3\\
           medium-play & 0.0$\pm$0.0 & 0.0$\pm$0.0 & \textbf{91.9$\pm$2.9} & 79.5$\pm$0.6\\
           medium-diverse & 0.0$\pm$0.0 & 0.0$\pm$0.0 & \textbf{91.3$\pm$1.4} & 75.3$\pm$1.3\\
           large-play & 0.0$\pm$0.0 & 0.0$\pm$0.0 & \textbf{62.8$\pm$2.6} & 62.1$\pm$3.7\\
           large-diverse & 0.0$\pm$0.0 & 0.0$\pm$0.0 & \textbf{82.7$\pm$2.8} & 61.4$\pm$4.4\\
        \bottomrule
    \end{tabular}
    \caption{Comparison between UBER and baselines in the online phase in the Antmaze domain at 100k and 500k environment steps. Pure online methods without an offline dataset fail completely due to the difficulty of exploration.}
    \label{tab: antmaze result}
\end{table*}

\begin{figure}[t]
    \centering
    \subfigure{\includegraphics[scale=0.2]{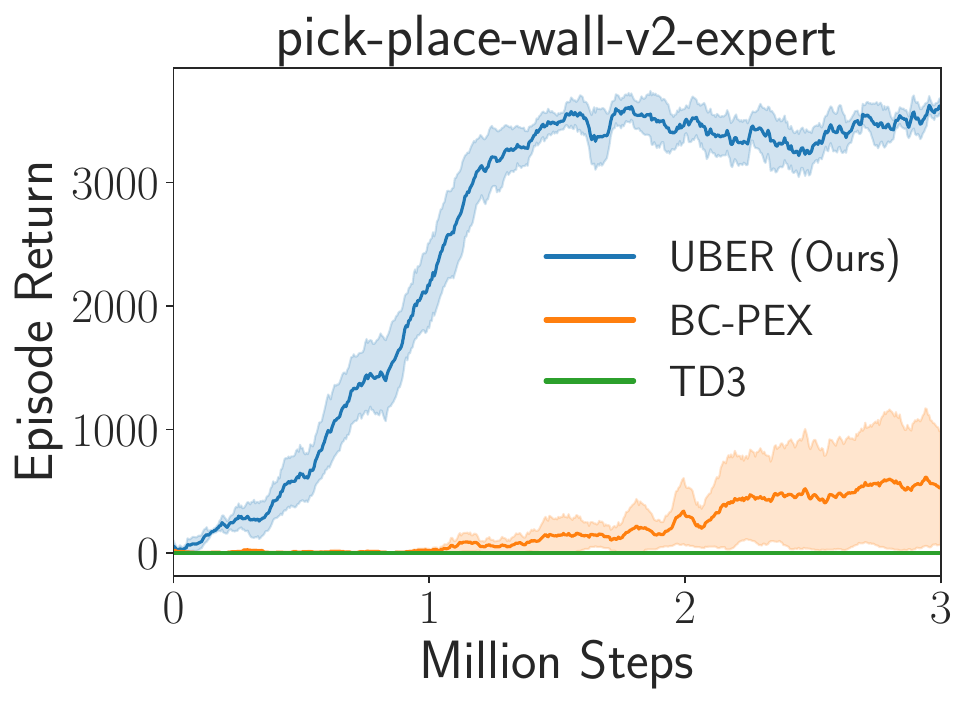}}
    \subfigure{\includegraphics[scale=0.2]{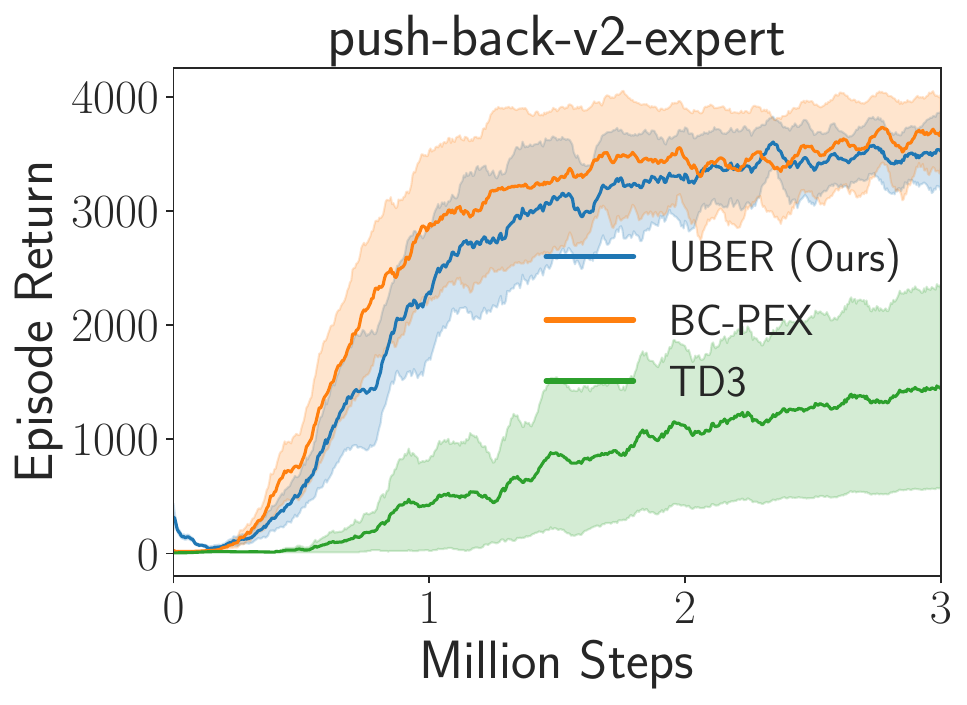}}
    \subfigure{\includegraphics[scale=0.2]{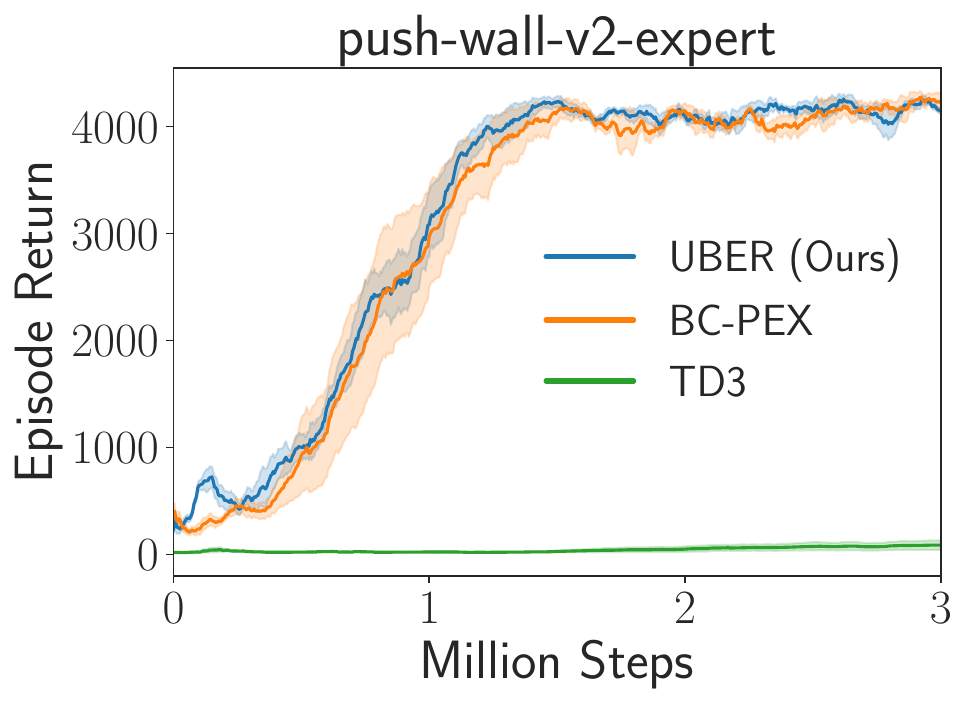}}
    \subfigure{\includegraphics[scale=0.2]{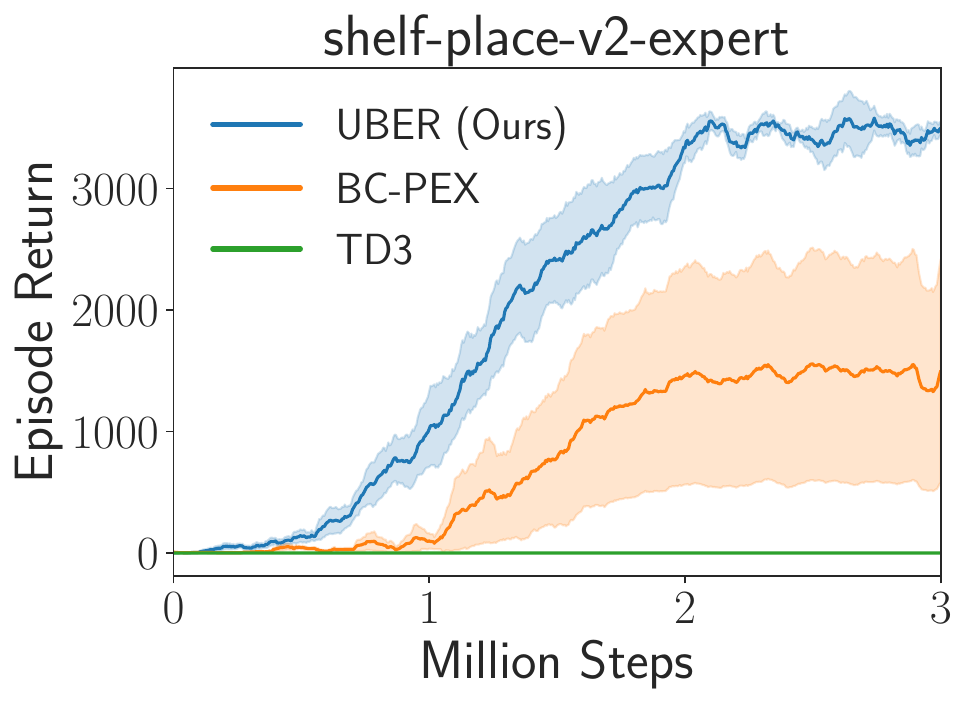}}
    \subfigure{\includegraphics[scale=0.2]{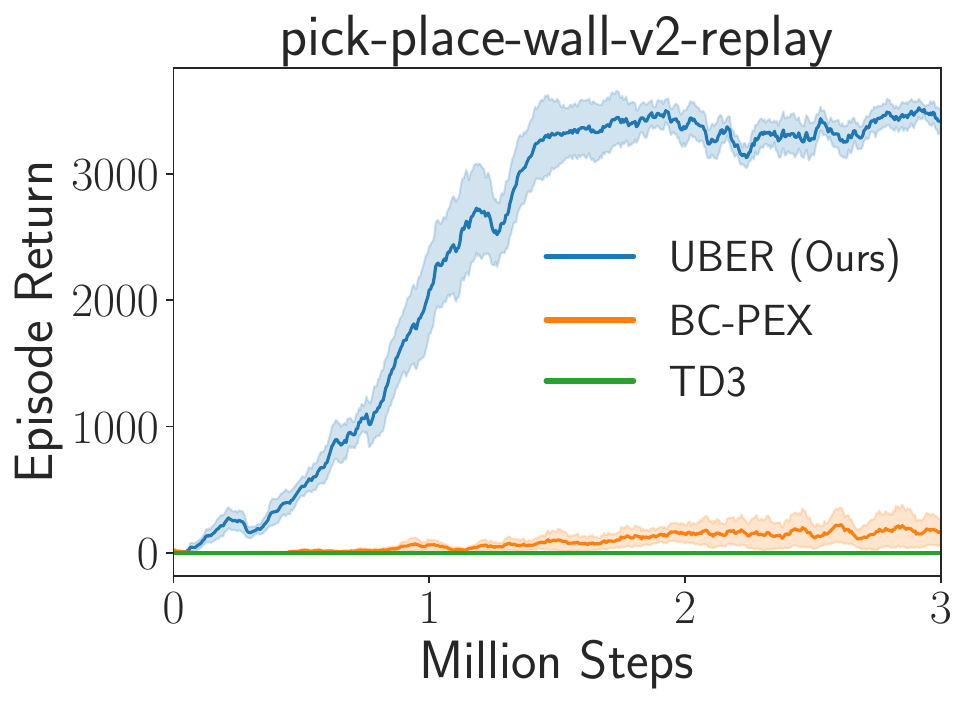}}
    \subfigure{\includegraphics[scale=0.2]{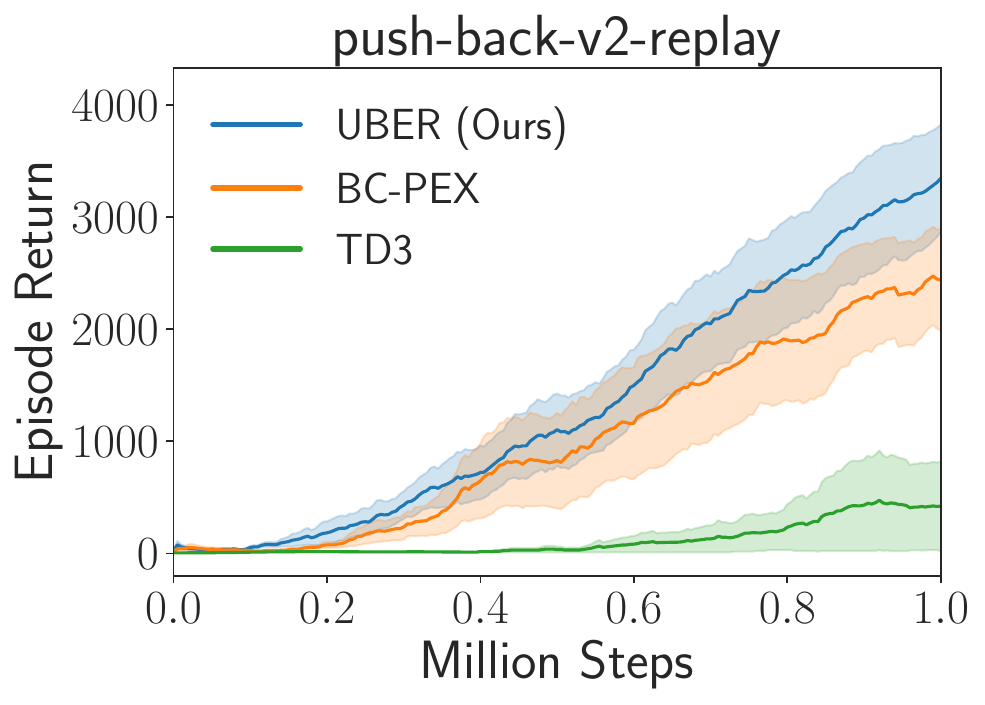}}
    \subfigure{\includegraphics[scale=0.2]{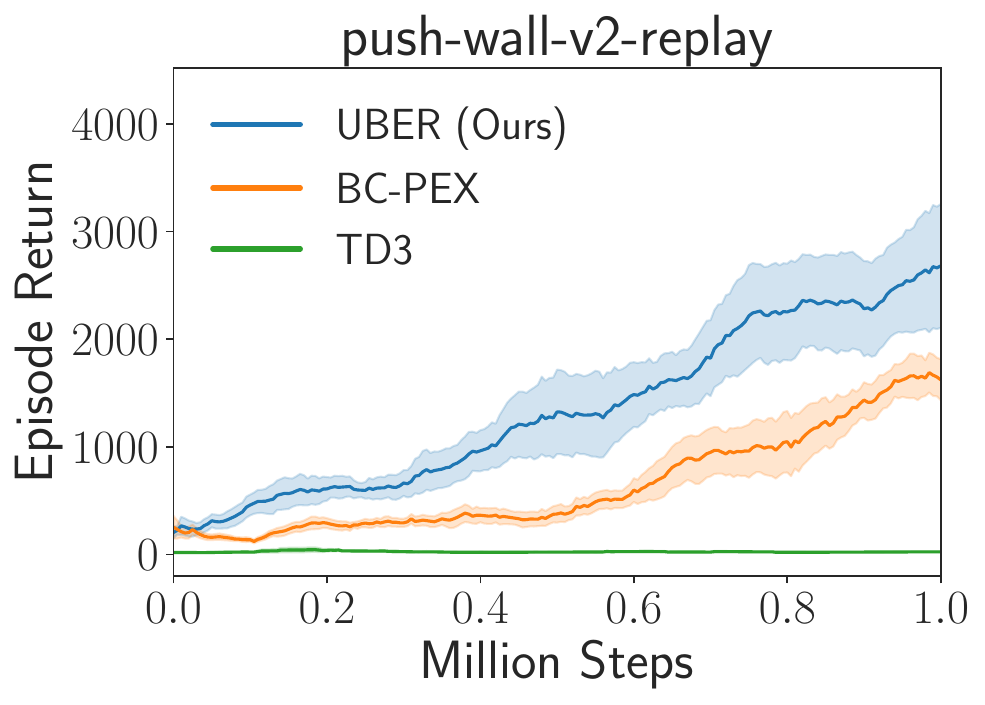}}
    \subfigure{\includegraphics[scale=0.2]{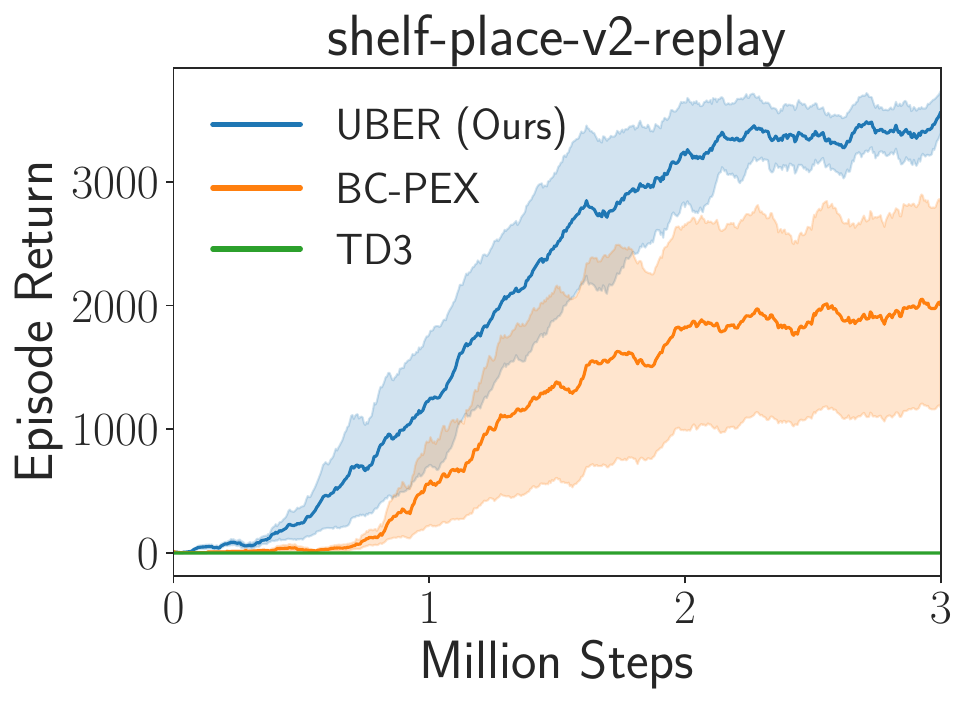}}
    \caption{Experimental results on the meta-world based on two types of offline datasets, expert and replay. 
    All experiment results were averaged over five random seeds. Our method achieves better or comparable results than the baselines consistently.}
    \label{fig: appendix multi-task}
\end{figure}

\begin{figure}[t]
    \centering
    \subfigure{\includegraphics[scale=0.22]{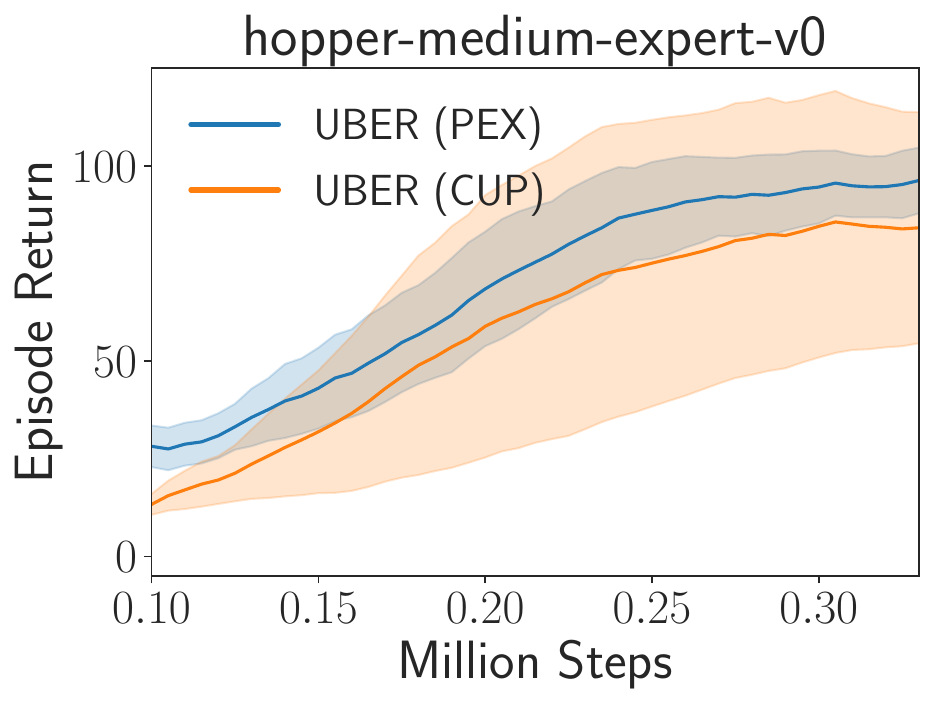}}\subfigure{\includegraphics[scale=0.22]{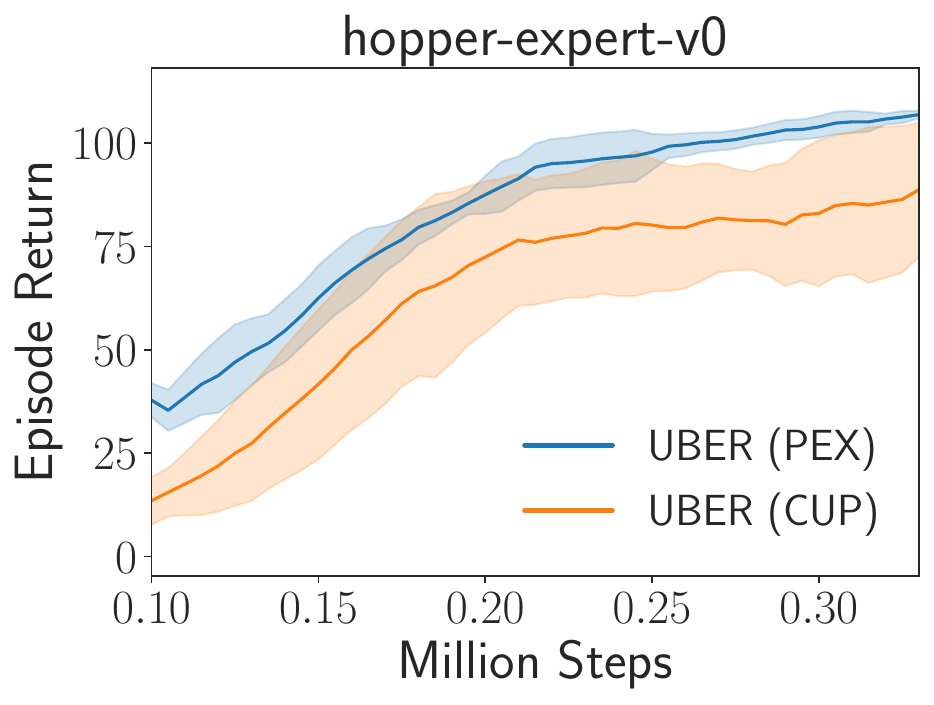}}
    \subfigure{\includegraphics[scale=0.22]{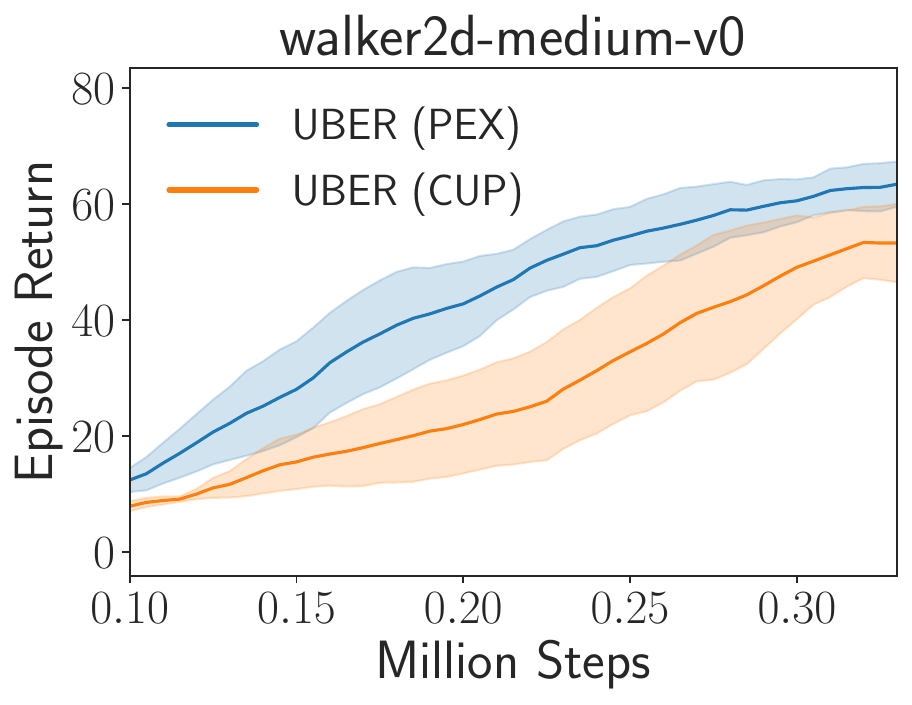}}\subfigure{\includegraphics[scale=0.22]{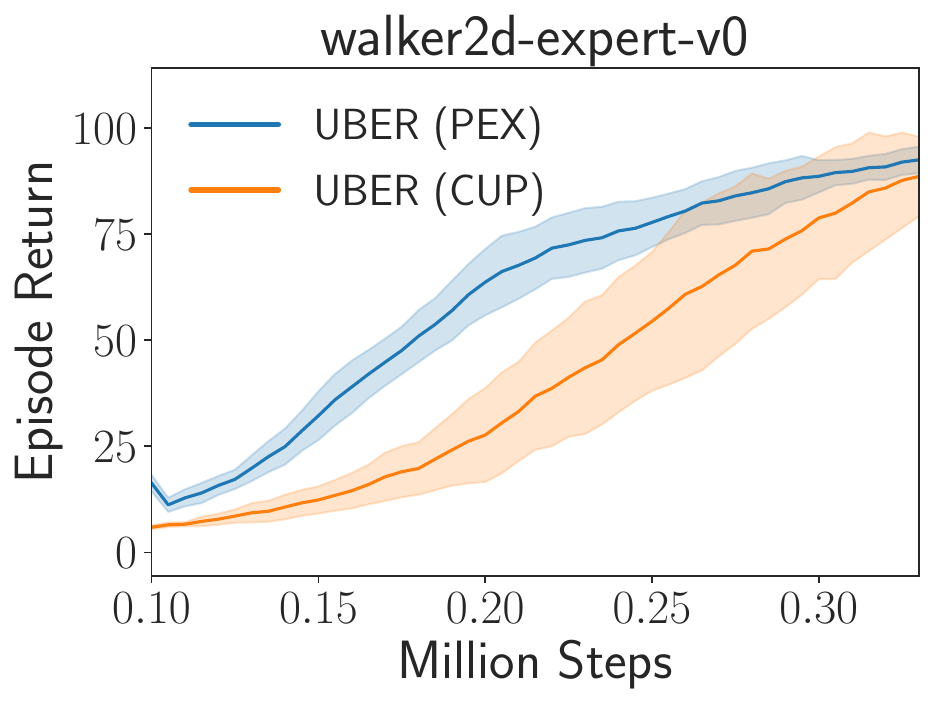}}

    \caption{Ablation study for the online policy reuse module. Using policy expansion (PEX) for downstream tasks is better than critic-guided policy reuse (CUP) in general.}
    \label{fig: policy reuse ablation}
\end{figure}

\begin{figure}[t]
    \centering
    \subfigure{\includegraphics[scale=0.21]{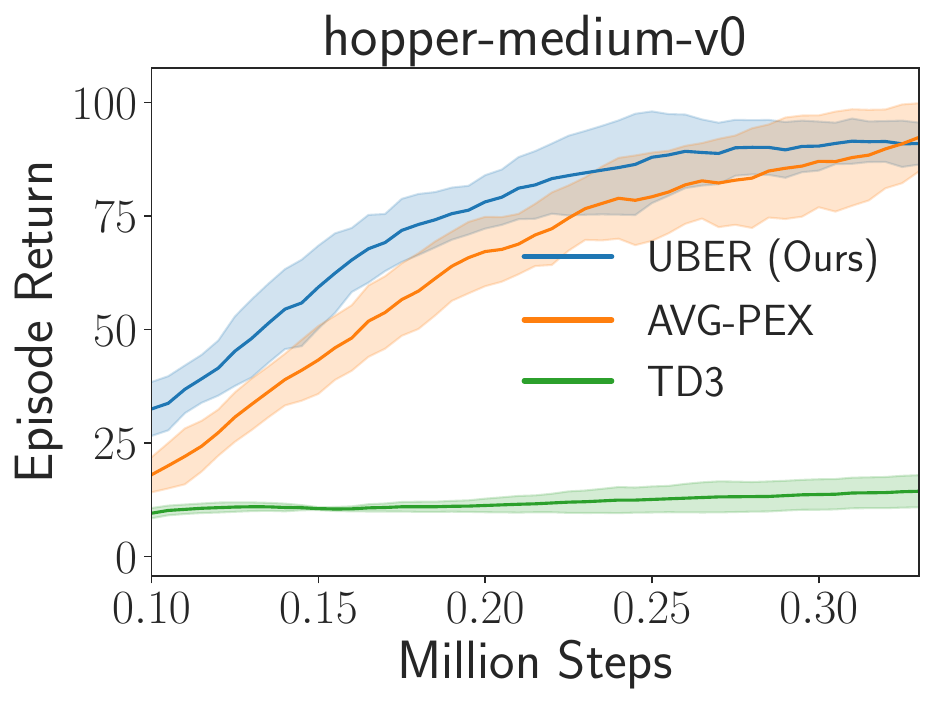}}
    \subfigure{\includegraphics[scale=0.21]{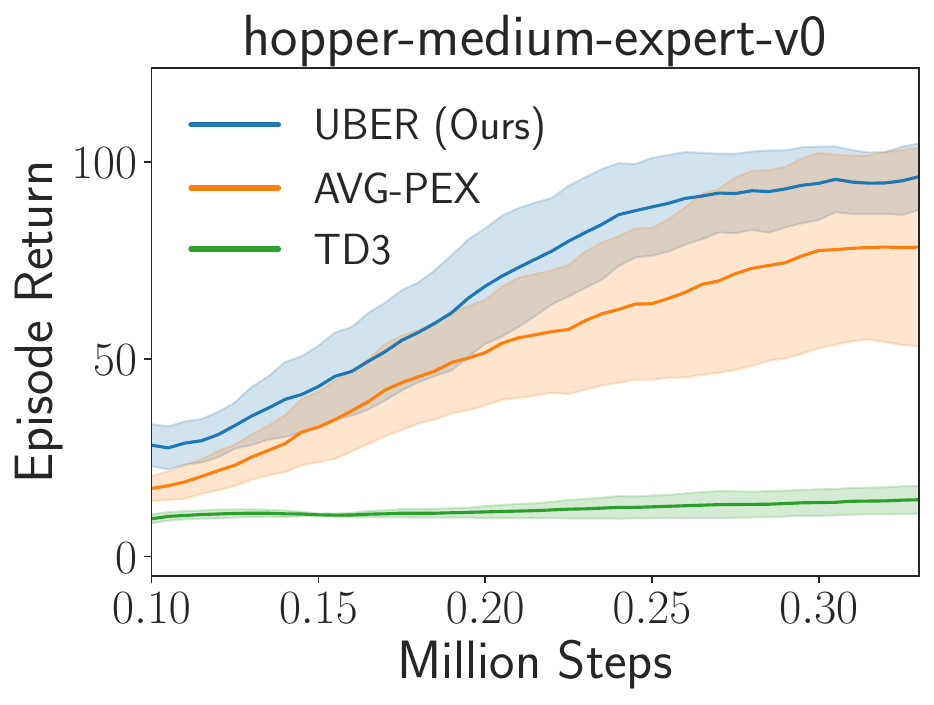}}\subfigure{\includegraphics[scale=0.21]{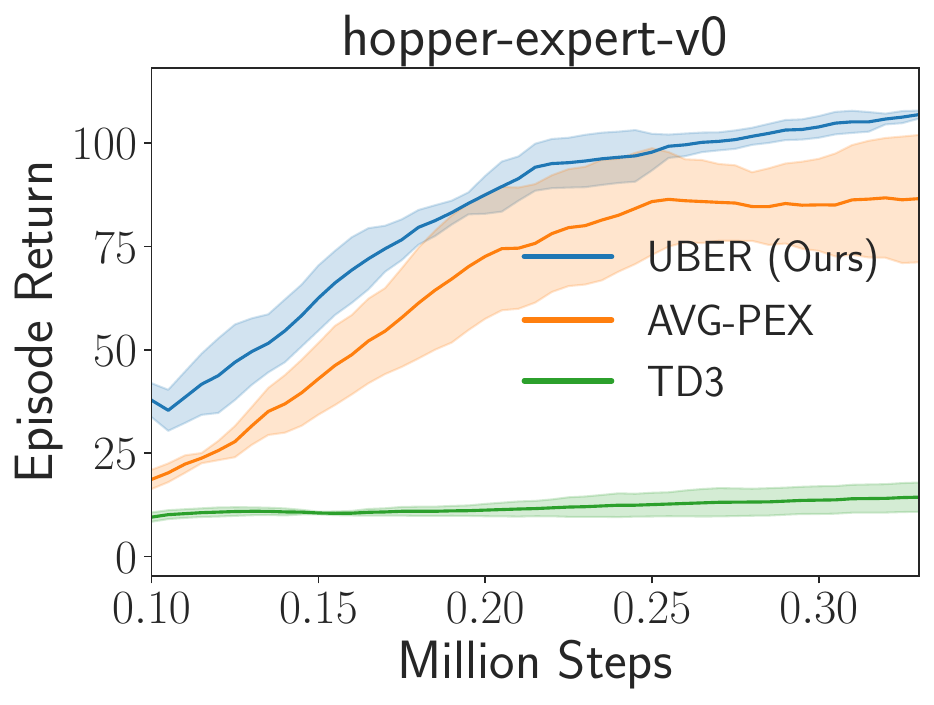}}
    \subfigure{\includegraphics[scale=0.21]{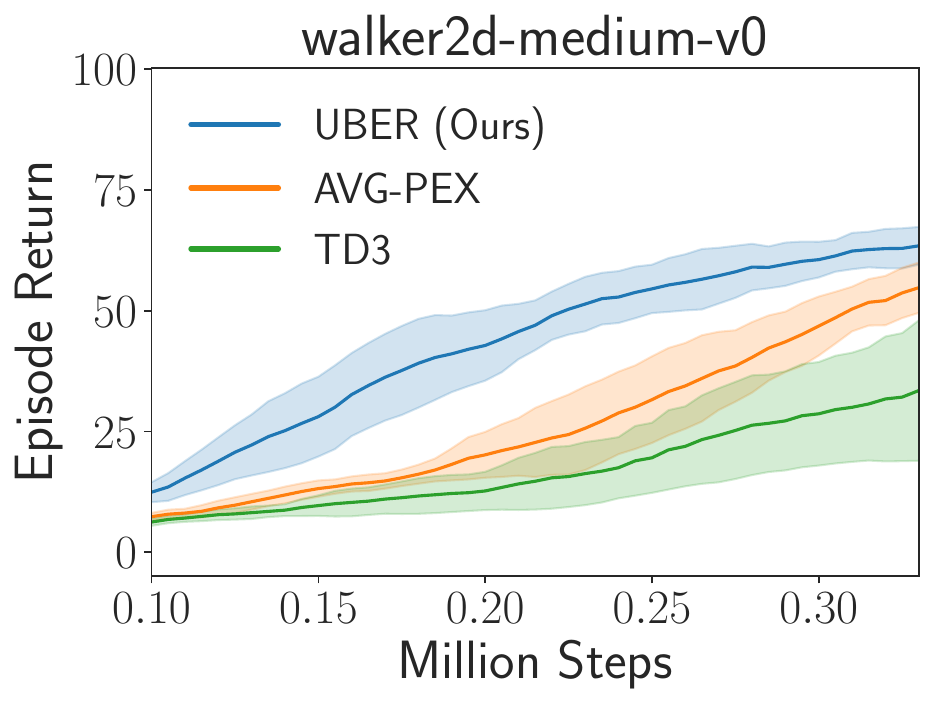}}
    
    %\subfigure{\includegraphics[scale=0.25]{avg/avg_walker2d-medium-expert-v0.pdf}}
    % \subfigure{\includegraphics[scale=0.25]{avg/avg_walker2d-expert-v0.pdf}}
    
    \caption{Ablation study for the random intent priors module. Using average reward for policies learning can lead to nearly optimal performances on the original task. Nevertheless, our method outperforms such a baseline (AVG-PEX) consistently, showing the importance of behavioral diversity.}
    \label{fig: random intent ablation}
\end{figure}

\subsection{Experimental Results}
\paragraph{Answer of Question 1:}
To show that UBER can generate diverse behaviors, we plot the return distribution of behaviors generated by UBER, as well as the original return distribution from the dataset. The experimental results are summarized in Figure~\ref{fig: random intent sta}. We also provide the entropy of the return distribution in Table~\ref{tab:return_entropy} for a numerical comparison. We can see that (1) UBER can extract policies that perform better than the dataset, especially when the dataset does not have high-quality data; (2) UBER can learn a diverse set of behaviors, spanning a wider distribution than the original dataset. 
% We can see that when there are expert behaviors, most of the behaviors generated by UBER successfully lock to the optimal policy. When there are multiple modes of behavior like \texttt{medium-expert}, UBER can learn both behaviors, matching the distribution of the dataset. When there are no expert behaviors, UBER can learn a diverse set of behaviors, some of which can achieve near-optimal performance. 
We hypothesize that diversity is one of the keys that UBER can outperform behavior cloning-based methods.
% we first run TD3+BC~\citep{fujimoto2021minimalist} on the hopper-medium-v0 tasks under 20 random rewards generated by random neural networks for evaluation.
% The results in Figure~\ref{fig: random reward td3+bc result} show that the learned policies are diverse in terms of performance.

% We further count the performance of 256 random intent priors, the reward return of each trajectory in the offline dataset, and the performance of the behavior-cloning~(BC) policy.
% The experimental results in Figure~\ref{fig: random intent sta} show that the random intent priors can generate behaviors from random to expert, even more diverse than the dataset, while BC can only focus on a single policy. As for the useful property, we use TD3+BC and TD3 as the backbone of offline and online algorithms.

Further, we conduct experiments to test if the set of random intents can cover the true intent. We calculate the correlation of $N=256$ random rewards with the true reward and measure the linear projection error. The results in Table~\ref{tab: coe and proj error} indicate that random intents can have a high correlation with true intent, and linear combinations of random rewards can approximate the true reward function quite well. 

To test whether UBER can learn useful behaviors for downstream tasks, we compare UBER with BC-PEX, AVG-PEX, and TD3 in the online phase in the Mujoco domains.
The experimental results in Figure~\ref{fig: mujoco result} show that our method achieves superior performance than baselines. By leveraging the behavior library generated from UBER, the agent can learn much faster than learning from scratch, as well as simple behavior cloning methods.
Further, the results in Figure~\ref{fig: new baseline result} show that UBER also performs better than behavior extraction and data sharing baselines in most tasks. 
This is because prior methods extract behaviors in a behavior-cloning manner, which lacks diversity and leads to degraded performance for downstream tasks.

In addition, to test UBER's generality across various offline and online methods on various domains, we also conduct experiments on Antmaze tasks.
We use IQL~\citep{kostrikov2021offline} and RLPD~\citep{ball2023efficient} as the backbone of offline and online algorithms.
% We compare UBER with RLPD which is a strong offline-to-online methods, and BC-PEX which uses behavior cloning as the behavior extraction method.
The experimental results in Table~\ref{tab: antmaze result} show that our method achieves stronger performance than baselines.
RLPD relies heavily on offline data and RLPD without true reward has zero return on all tasks.
Differently, UBER extracts a useful behavior policy set for the online phase, achieving similar sample efficiency as the oracle method that has reward label for the offline dataset. We also provide the learning curve for the antmaze environment in Figure~\ref{fig: antmaze result}.

\paragraph{Answer of Question 2:}
To test whether the behavior set learned from UBER can benefit multiple downstream tasks, we conducted the multi-task experiment on Meta-World~\citep{yu2020meta}, which requires learning various robot skills to perform various manipulation tasks.
We first extract behaviors with random intent priors from the selected tasks and then use the learned behaviors for various downstream tasks.
The experimental results in Appendix~\ref{appendix: addtional exp} show that the prior behaviors of UBER can be successfully transferred across multi-tasks than baselines.
Please refer to Appendix~\ref{appendix: addtional exp} for the experimental details and results.

\paragraph{Answer of Question 3:}
To understand what contributes to the performance of UBER, we perform ablation studies for each component of UBER.
We first replace the policy expansion module with critic-guided policy reuse~\citep[CUP;][]{zhang2022cup}, to investigate the effect of various online policy reuse modules. 
The experimental results in Figure~\ref{fig: policy reuse ablation} show that UBER+PEX generally outperforms UBER+CUP.
We hypothesize that CUP includes a hard reuse method, which requires the optimized policy to be close to the reused policies during training, 
while PEX is a soft method that only uses the pre-trained behaviors to collect data. This makes PEX more suitable to our setting since UBER may generate undesired behaviors.

Then, we ablate the random intent priors module.
We use the average reward in the offline datasets to learn offline policies and then follow the same online policy reuse process in Algorithm~\ref{alg: online rl}. Using average reward serves as a strong baseline since it may generate near-optimal performance as observed in~\citet{shin2023benchmarks}.
We name the average reward-based offline behavior extraction method as AVG-PEX.
The experimental results in Figure~\ref{fig: random intent ablation} show that while average reward-based offline optimization can extract some useful behaviors and accelerate the downstream tasks, using random intent priors performs better due to the diversity of the behavior set.

\section{Conclusion}
This paper presents Unsupervised Behavior Extraction via Random Intent Priors (UBER), a novel approach to enhance reinforcement learning using unsupervised data. UBER leverages random intent priors to extract diverse and beneficial behaviors from offline, reward-free datasets. Our theorem justifies the seemingly simple approach, and our experiments validate UBER's effectiveness in generating high-quality behavior libraries, outperforming existing baselines, and improving sample efficiency for deep reinforcement learning algorithms.

 UBER unlocks the usage of abundant reward-free datasets, paving the way for more practical applications of RL. UBER focuses on learning behaviors, which is orthogonal to representation learning. It is an interesting future direction to further boost online sample efficiency by combining both approaches. Consuming large unsupervised datasets is one of the keys to developing generalist and powerful agents. We hope the principles and techniques encapsulated in UBER can inspire further research and development in unsupervised reinforcement learning.

\section{Acknowledegement}
This work is supported in part by Science and Technology Innovation 2030 - ``New Generation Artificial Intelligence'' Major Project (No. 2018AAA0100904) and the National Natural Science Foundation of China (62176135).

% Acknowledgements should only appear in the accepted version.

% In the unusual situation where you want a paper to appear in the
% references without citing it in the main text, use \nocite
\nocite{*}

% \bibliography{reference}
% \bibliographystyle{icml2022}

\bibliographystyle{icml2022}
\bibliography{reference}

%%%%%%%%%%%%%%%%%%%%%%%%%%%%%%%%%%%%%%%%%%%%%%%%%%%%%%%%%%%%%%%%%%%%%%%%%%%%%%%
%%%%%%%%%%%%%%%%%%%%%%%%%%%%%%%%%%%%%%%%%%%%%%%%%%%%%%%%%%%%%%%%%%%%%%%%%%%%%%%
% APPENDIX
%%%%%%%%%%%%%%%%%%%%%%%%%%%%%%%%%%%%%%%%%%%%%%%%%%%%%%%%%%%%%%%%%%%%%%%%%%%%%%%
%%%%%%%%%%%%%%%%%%%%%%%%%%%%%%%%%%%%%%%%%%%%%%%%%%%%%%%%%%%%%%%%%%%%%%%%%%%%%%%
\newpage
\appendix
\onecolumn
\section{Missing Proofs}
\label{appenix: proof}
\subsection{Proof of Proposition~\ref{intention_completeness}}
\label{appenix: proof1}
\begin{proposition}[Proposition~\ref{intention_completeness} restated]
    % \label{intention_completeness}
    For any behavior $\pi$, there exists an intent $z$ with reward function $r(\cdot,\cdot,z)$ such that $\pi$ is the optimal policy under $r(\cdot,\cdot,z)$. That is, for any $\pi=\{\pi_h\}_{h=1}^H, \pi_h: \cS\rightarrow \Delta(\cA)$, there exists $z\in \cZ$ such that
    \$
    V^{\pi}_h(s,z) = V^*_h(s,z),~ \forall (s,a,h) \in \cS\times\cA\times [H]. 
    \$
    % For any behavior $\pi$, there exists an intention $z$ with reward function $r_z$ such that $\pi$ is the optimal policy under $r_z$, i.e.
    % \begin{align}
    % V^{\pi}_{h,r_z}(\cdot) = V^*_{h,r_z}(\cdot). 
    % \end{align}
    Moreover, if $\pi$ is deterministic, then there exists $z\in\cZ$ such that $\pi$ is the unique optimal policy under $r(\cdot,\cdot,z)$, in the sense that for all optimal policy $\pi^*_{z}, \pi^*_{h,z}(\cdot|s)=\pi_h(\cdot|s), \forall d^{\pi}_h(s)>0$.
\end{proposition}
\begin{proof}
For any policy $\pi$ and any reward function $r_h(\cdot,\cdot,z)$, we have 
\begin{align}
V^\pi_{h}(s_h,z) =  \sum_{t=h}^H\int{r_t(s,a,z) d^\pi_t(s,a,s_h) \mathrm{d}s \mathrm{d}a},
\end{align}
where $d^\pi_t(s,a,s_h)$ is the visitation density at step $t$ for $\pi$ starting from $s_h$.

Let $r_h(s,a,z)=\mathbbm{1}(\exists s_h \in \cS, d^\pi_h(s,a,s_h)>0)$, then we have 
\begin{align}
    V^{\pi}_{h}(\cdot,z) = H-h = V_{h,\text{max}} = V^*_{h}(\cdot,z). 
\end{align}

If $\pi$ is deterministic, then $d^\pi_h(s,a)$ is a one-hot distribution at any given $s$. So any $\pi'$ such that there exists $s_h$, $\pi_h(\cdot|s_h)\neq \pi'_h(\cdot|s_h)$ and $d_h^\pi(s_h)>0$, we have 

\$
\int{r_h(s_h,a,z) d^{\pi'}_h(s_h,a) \mathrm{d}a} < 1,
\$ 
Then we have $V^{\pi'}_{h}(s_h,z)< V_{h,\text{max}}$, which means $\pi'$ is not optimal. Therefore, the optimal policy $\pi$ is unique.

\end{proof}

Proposition~\ref{intention_completeness} shows that any deterministic policy can be uniquely determined by some intent $z$ so that it is extractable via standard RL.

\subsection{Proof of Theorem~\ref{theorem:1}}
\label{appenix: proof2}
\begin{theorem}[Theorem~\ref{theorem:1} restated.]
    Consider linear MDP as defined in Definition~\ref{assumption:linear}. With an offline dataset $\cD$ with size $N$, and the PEVI algorithm~\citep{jin2021pessimism}, the suboptimality of learning from an intent $z\in \cZ$ with size $|\cZ|$ satisfies
    \#
        \text{\rm SubOpt}(\widehat{\pi};r_z)  \leq 4c\sqrt{\frac{C^\dagger_z d^2H^3\iota}{N}},
    \#
    with probability $1-\delta$ for sufficiently large $N$, where $\iota=\log{\frac{4dN |\cZ|}{\delta}}$ is a logarithmic factor, $c$ is an absolute constant and 
    \$
    C^{\dagger}_{z} = \max_{h\in[H]} \sup_{\|x\|=1}\frac{x^{\top}\Sigma_{\pi_z,h} x}{x^{\top}\Sigma_{\rho_h} x}, 
    \$
    with 
    \$
    \Sigma_{\pi_z,h} &= \EE_{(s,a)\sim d_{\pi_z,h}(s,a)} [\phi(s,a)\phi(s,a)^{\top}],\quad \Sigma_{\rho_h} = \EE_{\rho_h} [\phi(s,a)\phi(s,a)^{\top}].
    \$
\end{theorem}

\begin{proof}
Using the fact that any linear reward function is still a linear MDP, we can reuse the proof for standard offline RL.

Let $\delta' = \frac{\delta}{|\cZ|}$ and following ~\citet{jin2021pessimism}, we have the result immediately with a union bound. Note that the original bound in ~\citet{jin2021pessimism} scales as $\tilde{\cO}(d^3)$, but it can be improved to $\tilde{\cO}(d^2)$ by finer analysis without changing the algorithm~\citep{xiong2022nearly}.
\end{proof}

\subsection{Proof of Theorem~\ref{theorem:2}}
\label{appenix: proof3}
\begin{proof}
Each random reward function can be seen as one dimension of random feature for linear regression. Following Theorem 1 in ~\citet{rudi2017generalization}, and note that the true reward function has zero error, we have that the linear estimator $\widehat{r}_{\rm {ridge}}$ generated by ridge regression satisfy
\$
\E_{(s,a)\sim \rho}(\widehat{r}_{\rm {ridge}}(s,a) - r(s,a))^2 = \E_{(s,a)\sim \rho}\left(\sum_i \widehat{w}^{\rm {ridge}}_i r_i(s,a) - r(s,a)\right)^2\leq c_1 \log^2(18/\delta)/\sqrt{M}.
\$
Noting that 
\$\min_w \E_{(s,a)\sim \rho}\left(\sum_i w_i r_i(s,a) - r(s,a)\right)^2 \leq \E_{(s,a)\sim \rho}\left(\sum_i \widehat{w}^{\rm {ridge}}_i r_i(s,a)- r(s,a)\right)^2,\$
we have result immediately.

% Since the true reward function admits a RKHS representation, we have 
% From simulation lemma, we have $\|V^\pi_{r_1} - V^\pi_{r_2}\|_\infty \leq H\epsilon$ for any $\pi$.

% Using the optimality for $\pi_1$ and $\pi_2$, we have
% \$
% V^{\pi_1}_{r_1} - V^{\pi_2}_{r_2} =  V^{\pi_1}_{r_1} - V^{\pi_1}_{r_2}  + V^{\pi_1}_{r_2} - V^{\pi_2}_{r_2}  \leq V^{\pi_1}_{r_1} - V^{\pi_1}_{r_2}  \leq H\epsilon.
% \$
% Similarly,

% \$
% V^{\pi_2}_{r_2} - V^{\pi_1}_{r_1}  =  V^{\pi_2}_{r_2} - V^{\pi_2}_{r_1}  + V^{\pi_2}_{r_1} - V^{\pi_1}_{r_1}   \leq V^{\pi_2}_{r_2} - V^{\pi_2}_{r_1} \leq H\epsilon.
% \$

% Then we have 
% \$
% |V^{\pi_1}_{r_1} - V^{\pi_2}_{r_1} | & \leq  |V^{\pi_1}_{r_1} - V^{\pi_2}_{r_2} | + |V^{\pi_2}_{r_2} - V^{\pi_2}_{r_1} | \\
% & \leq H\epsilon + H\epsilon \\
% & = 2H\epsilon.
% \$
\end{proof}

% Proposition~\ref{theorem:2} justifies the usage of a random reward function and we provide a geometric interpretation in Figure~\ref{fig:voronoi}.

% \begin{figure}[H]
%     \centering
%     \includegraphics[scale=0.5]{voronoi_austin.png}
%     \caption{Geometric interpretation of random intention priors. Red dots represent random intentions, green dots represent true intentions in the dataset. Blue lines separate near-neighbor regions for different random intentions. According to Proposition~\ref{theorem:2}, as long as the intention prior covers the true intention distribution well, we can effectively recover the desired behavior with low error in terms of the value function.}
%     \label{fig:voronoi}
% \end{figure}
% \clearpage
\section{Experiments on Multi-task Transfer}
\label{appendix: addtional exp}
\paragraph{Answer of Question 2:}
We evaluate our method on Meta-World~\citep{yu2020meta}, a popular reinforcement learning benchmark composed of multiple robot manipulation tasks.
These tasks are correlated~(performed by the same Sawyer robot arm) while being distinct~(interacting with different objectives and having different reward functions).
Following~\cite{zhang2022cup}, we use datasets from three representative tasks: \texttt{Reach}, \texttt{Push}, and \texttt{Pick-Place}, as shown in Figure~\ref{fig: meta world} and we choose \texttt{Push-Wall}, \texttt{Pick-Place-Wall}, \texttt{Push-Back}, and \texttt{Shelf-Place} as the target tasks.

\begin{figure}[h]
    \centering\subfigure{\includegraphics[scale=0.5]{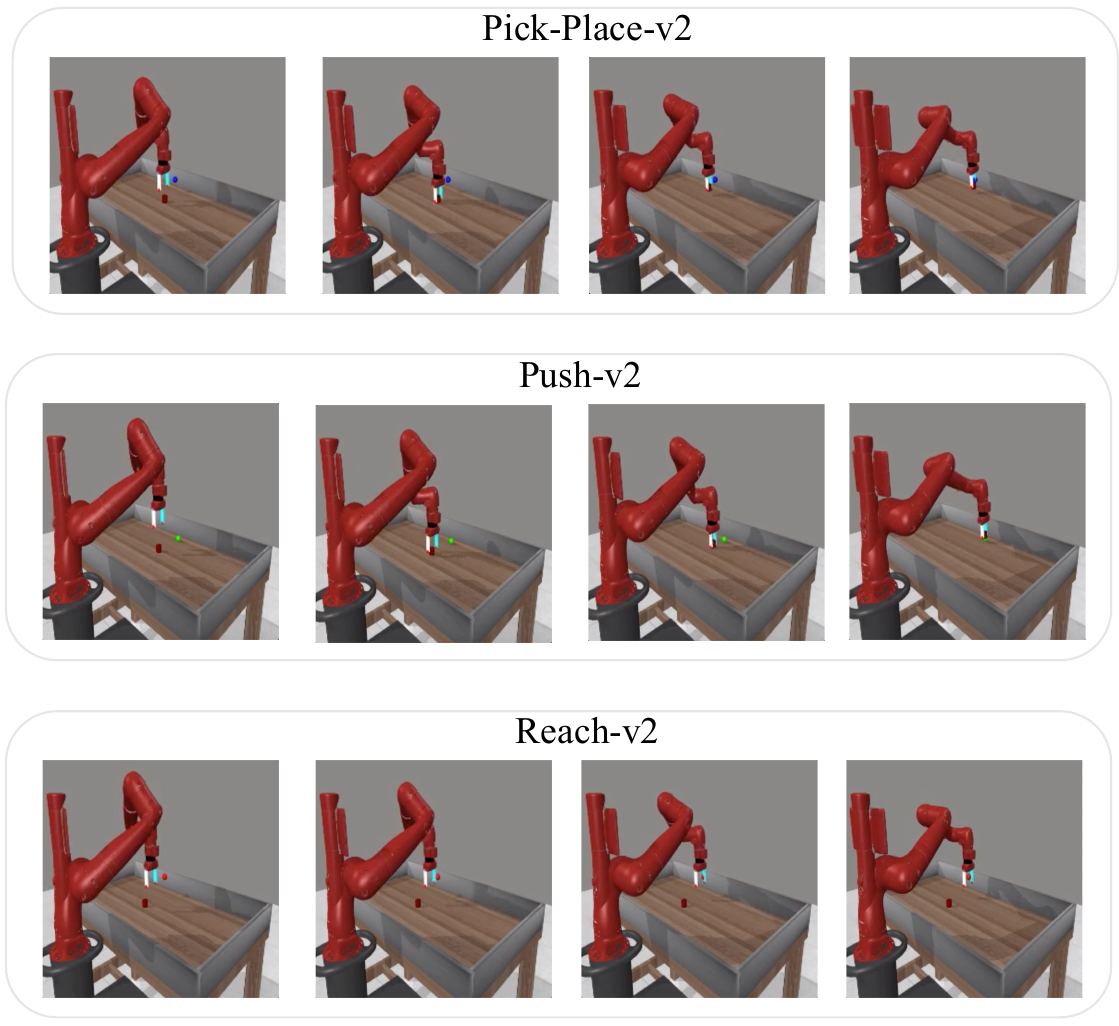}}
    \caption{Visualization of three representative tasks in the metaworld.}
    \label{fig: meta world}
\end{figure}

Two types of offline datasets were considered: an expert policy dataset and a replay dataset generated by TD3. The expert dataset comprised 0.1M data points obtained from expert demonstrations, while the replay dataset consisted of 0.3M data points sampled from 3M TD3 experiences.

Our experimental results, depicted in Figure~\ref{fig: appendix multi-task2}, reveal that UBER outperforms the baselines in most tasks. This underscores UBER's ability to effectively repurpose learned behaviors from various tasks for different downstream applications.

% We consider two types of offline datasets, one is the dataset generated by an expert policy, and the other is a replay dataset generated by TD3. For the expert dataset, we collect 0.1M data from expert demonstrations. For the replay dataset, we use 0.3M data generated from the experience of TD3.
% The experimental results in Figure~\ref{fig: appendix multi-task} show that UBER performs better than baselines in most tasks. This demonstrates that the learned behaviors from some tasks using UBER can be reused for various downstream tasks.

% In implementation, we use TD3+BC and TD3 as the backbone of offline and online algorithms.
% We first extract the random intent priors from three representative tasks and then use these priors for online learning.
% As for the baseline, we first adopt Behavior Cloning to extract the behavior policies from three representative tasks and then follows the same online optimization process.
% In addition, we compare UBER with online TD3.
\begin{figure}[ht]
    \centering\subfigure{\includegraphics[scale=0.4]{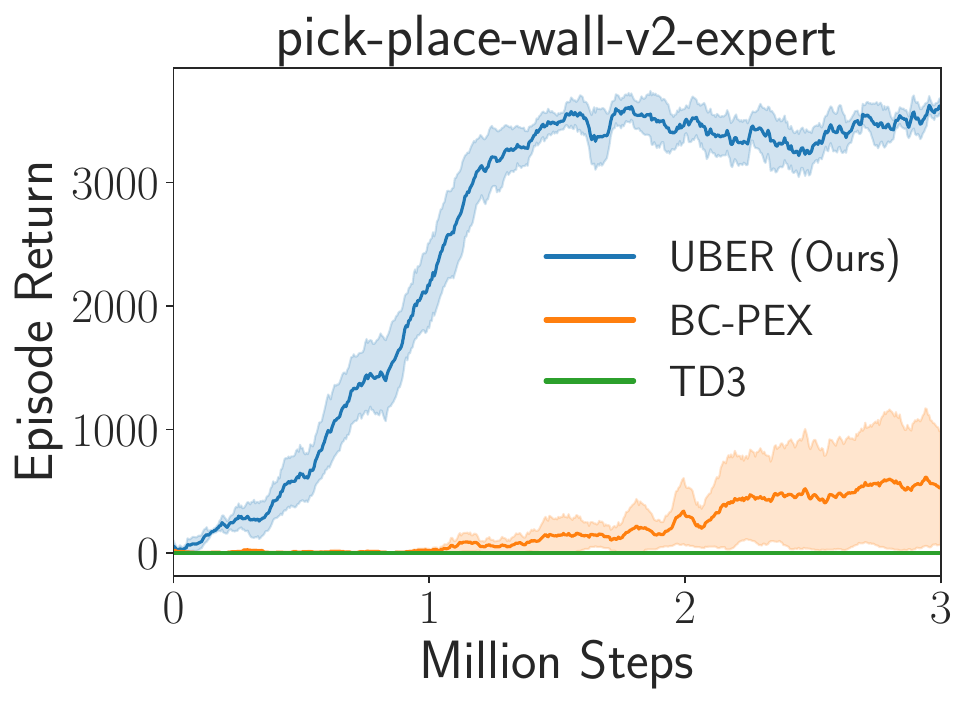}}
    \subfigure{\includegraphics[scale=0.4]{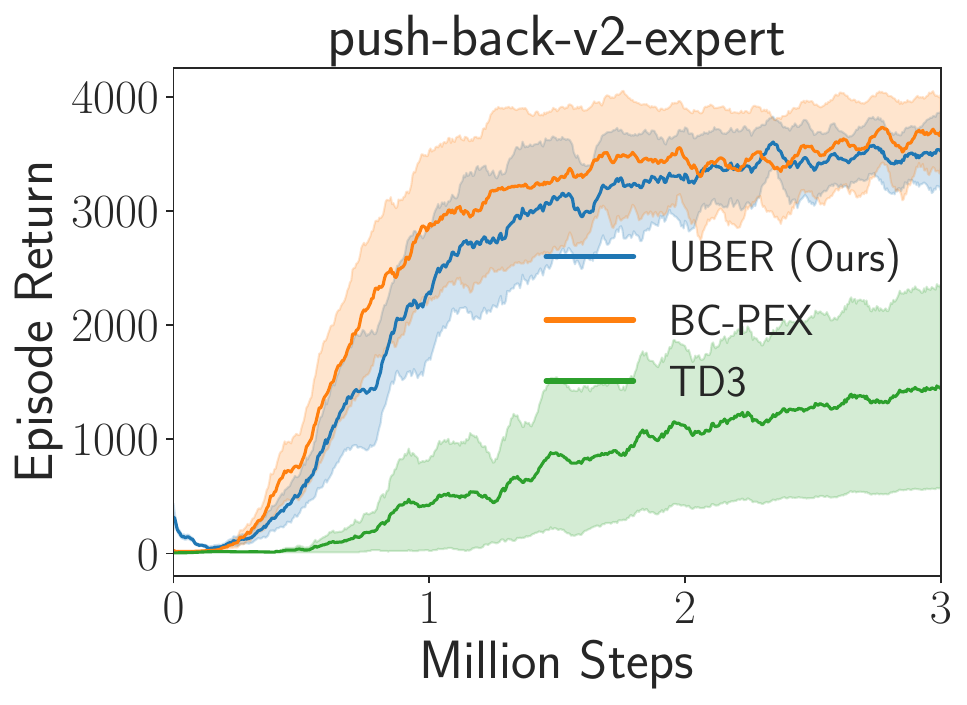}}
    \subfigure{\includegraphics[scale=0.4]{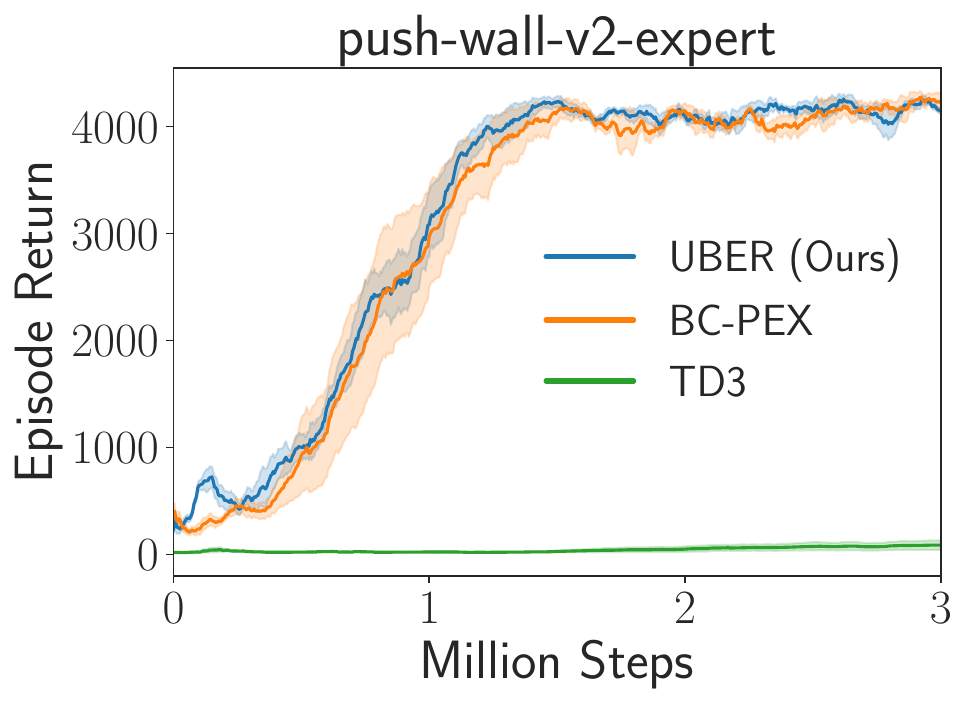}}
    \subfigure{\includegraphics[scale=0.4]{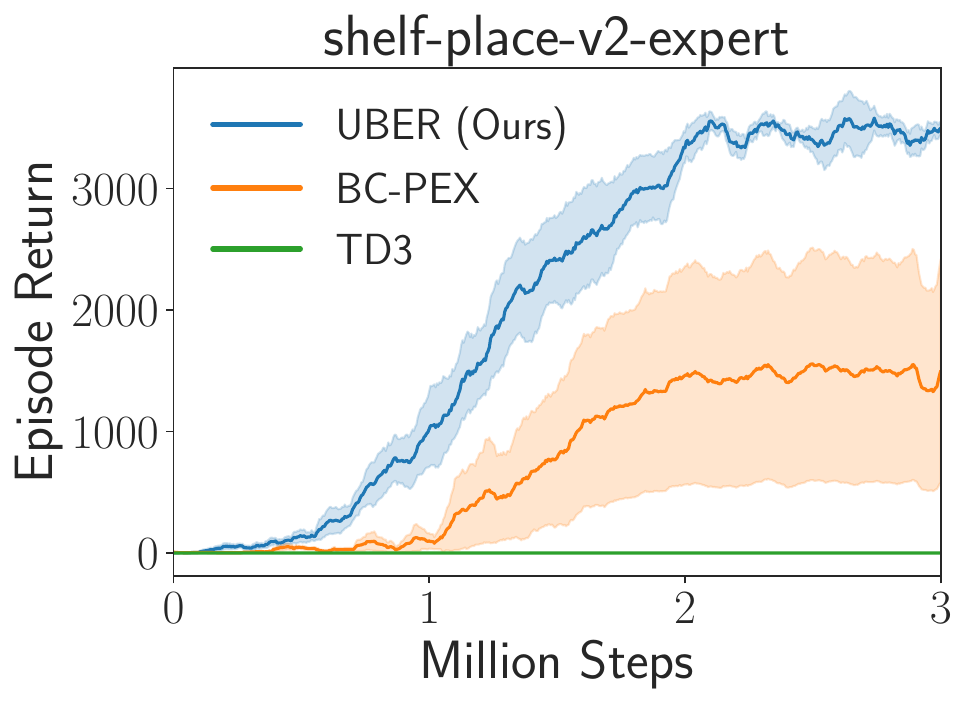}}
    \subfigure{\includegraphics[scale=0.4]{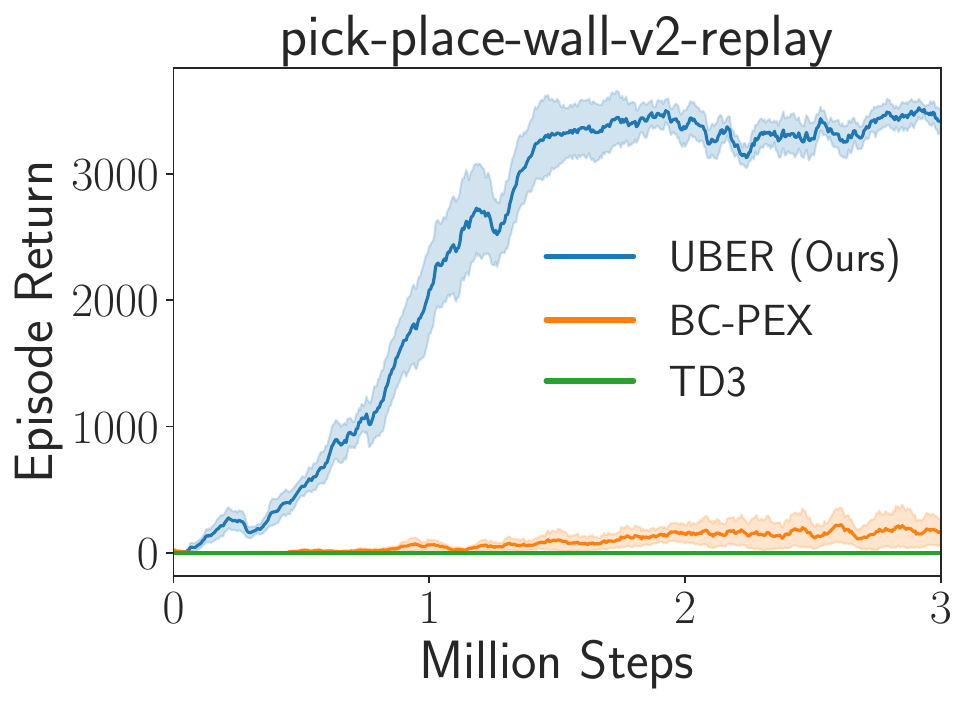}}
    \subfigure{\includegraphics[scale=0.4]{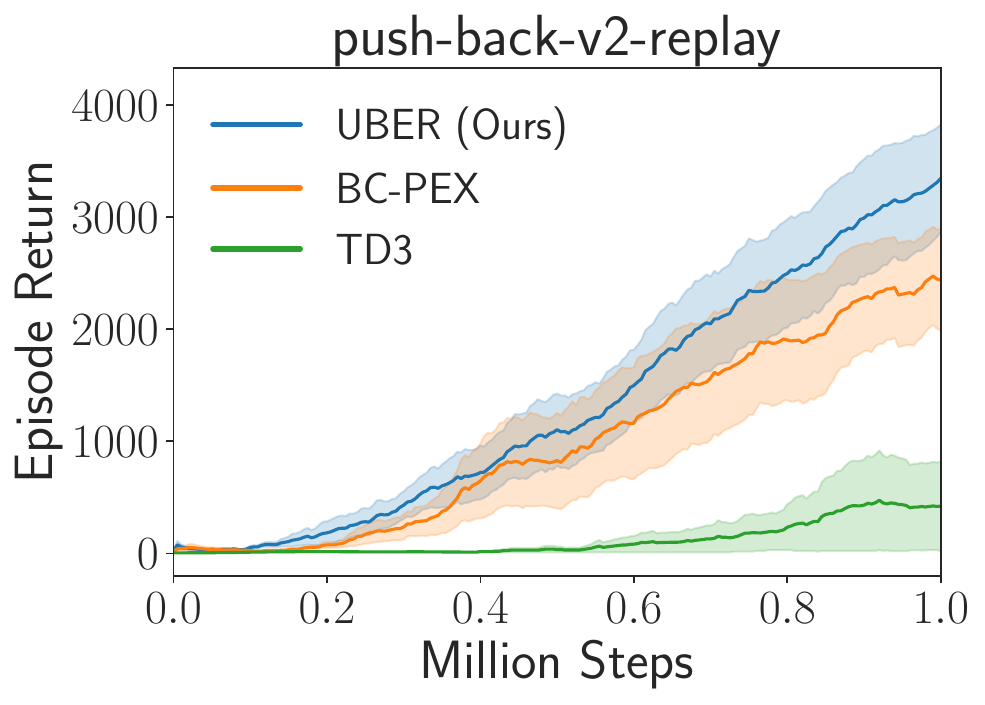}}
    \subfigure{\includegraphics[scale=0.4]{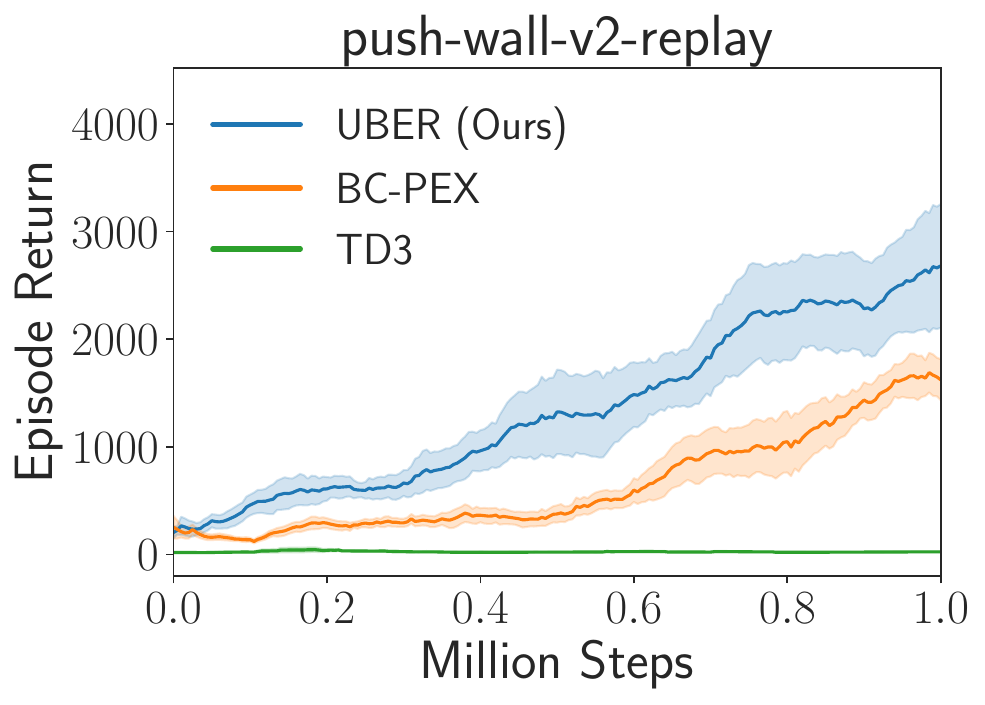}}
    \subfigure{\includegraphics[scale=0.4]{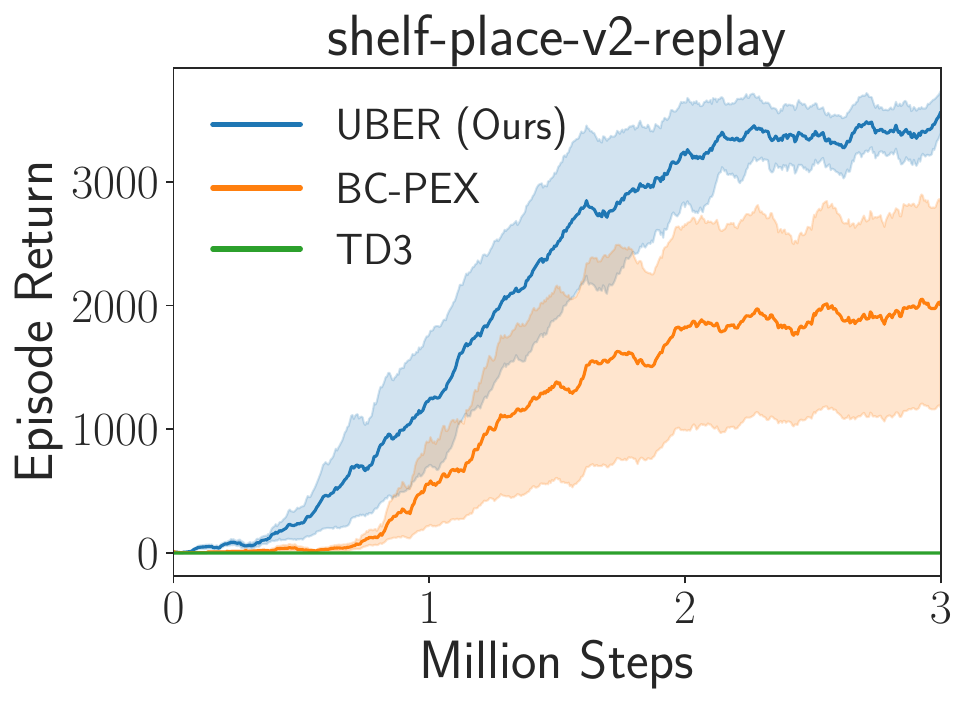}}
    \caption{Complete experimental results on the meta-world based on two types of offline datasets, expert and replay. 
    All experiment results were conducted over five random seeds.}
    \label{fig: appendix multi-task2}
\end{figure}

\section{Additional Experiments}
\label{appendix: addi exp}
\paragraph{Correlation and projection error of random rewards.}
To investigate whether random rewards has a good coverage and a high correlation with the true reward, we calculate the maximum correlation between random rewards and the true reward, as well as the projection error of the true reward on the linear combination of random rewards on Mujoco tasks, as shown in Table~\ref{tab: coe and proj error}.

\begin{table*}[ht]
    \centering
    \begin{tabular}{c|ccc}
        \toprule
          Tasks & Max Correlation & Min Correlation & Projection Error \\
         \midrule
            hopper-medium-v0 & 0.569 & -0.568 & 0.016 \\
           hopper-medium-expert-v0 & 0.498 & -0.540 & 0.015\\
           hopper-expert-v0 & 0.423 & -0.415 & 0.011\\
           halfcheetah-medium-v0 & 0.569 & -0.568 & 0.016 \\
            halfcheetah-medium-expert-v0 & 0.605 & -0.589 & 0.047 \\
            halfcheetah-expert-v0 & 0.370 & -0.461 & 0.021 \\
            walker2d-medium-v0 & 0.475 & -0.582 & 0.046 \\
            walker2d-medium-expert-v0 & 0.495 & -0.472& 0.042 \\
            walker2d-expert-v0 & 0.358 & -0.503 & 0.019 \\
        \bottomrule
    \end{tabular}
    \caption{Minimum and maximum correlation and linear projection error for the true reward function using random reward functions on various tasks. The projection error $\epsilon$ is defined as $\epsilon=\|r-\hat{r}\|/\|r\|$, where $\hat{r}$ is the best approximation using linear combination of random rewards. Note that we use random reward functions (represented by neural networks) based on $(s,a)$ inputs rather than entirely random ones, which have a near-zero correlation with the true reward and $40\%$ projection error.}
    \label{tab: coe and proj error}
\end{table*}
\paragraph{Entropy of return distribution.}
To numerically show that UBER can generate a diverse behavior set beyond the original dataset, we calculate the entropy of the return distribution for each task, as shown in Table~\ref{tab:return_entropy}.

\begin{table}[ht]
    \centering
    \begin{tabular}{c|c@{\hspace{0.2em}}c@{\hspace{0.2em}}c@{\hspace{0.2em}}c@{\hspace{0.2em}}c@{\hspace{0.2em}}c}
    \toprule
         Methods & halfcheetah-m & halfcheetah-me & halfcheetah-e & hopper-m & hopper-me & hopper-e \\ \midrule
        BC & 0 & 0 & 0 & 0 & 0 & 0 \\ 
        DATASET & 1.21 & 1.47 & 0.86 & 0.71 & 1.97 & 0.82 \\ 
        UBER & 2.01 & 1.92 & 2.69 & 2.45 & 2.32 & 1.11 \\ \midrule
         Methods & walker2d-m & walker2d-me & walker2d-e & antmaze-md & antmaze-mr & ~ \\ \midrule
        BC & 0 & 0 & 0 & 0 & 0 & ~ \\
        DATASET & 1.44 & 1.88 & 2.43 & 0.63	& 0.29 & ~ \\
        UBER & 2.67 & 2.61 & 0.41 &	1.77 &	1.69 & ~ \\ \bottomrule
    \end{tabular}
   
    \caption{Entropy of the return distribution for different methods. Our method generate a behavior set that has a higher entropy in the return distribution than the original dataset consistently across various tasks.}
     \label{tab:return_entropy}
\end{table}
\paragraph{Distribution of behaviors.}
To show that UBER can generate diverse behaviors, we plot the distribution of behaviors generated by UBER, as well as the original distribution from the dataset. The experimental results are summarized in Figure~\ref{fig: appendix random intent sta}. We can see that when there are expert behaviors, most of the behaviors generated by UBER successfully lock to the optimal policy. When there are multiple modes of behavior like \texttt{medium-expert}, UBER can learn both behaviors, matching the distribution of the dataset. When there are no expert behaviors, UBER can learn diverse behaviors, some of which can achieve near-optimal performance. We hypothesize that diversity is one of the keys that UBER can outperform behavior cloning-based methods.

\begin{figure}[ht]
    \centering
    \subfigure{\includegraphics[scale=0.22]{distribution/hopper-medium.pdf}}
    \vspace{1mm}    \subfigure{\includegraphics[scale=0.22]{distribution/hopper-medium-expert.pdf}}
    \subfigure{\includegraphics[scale=0.22]{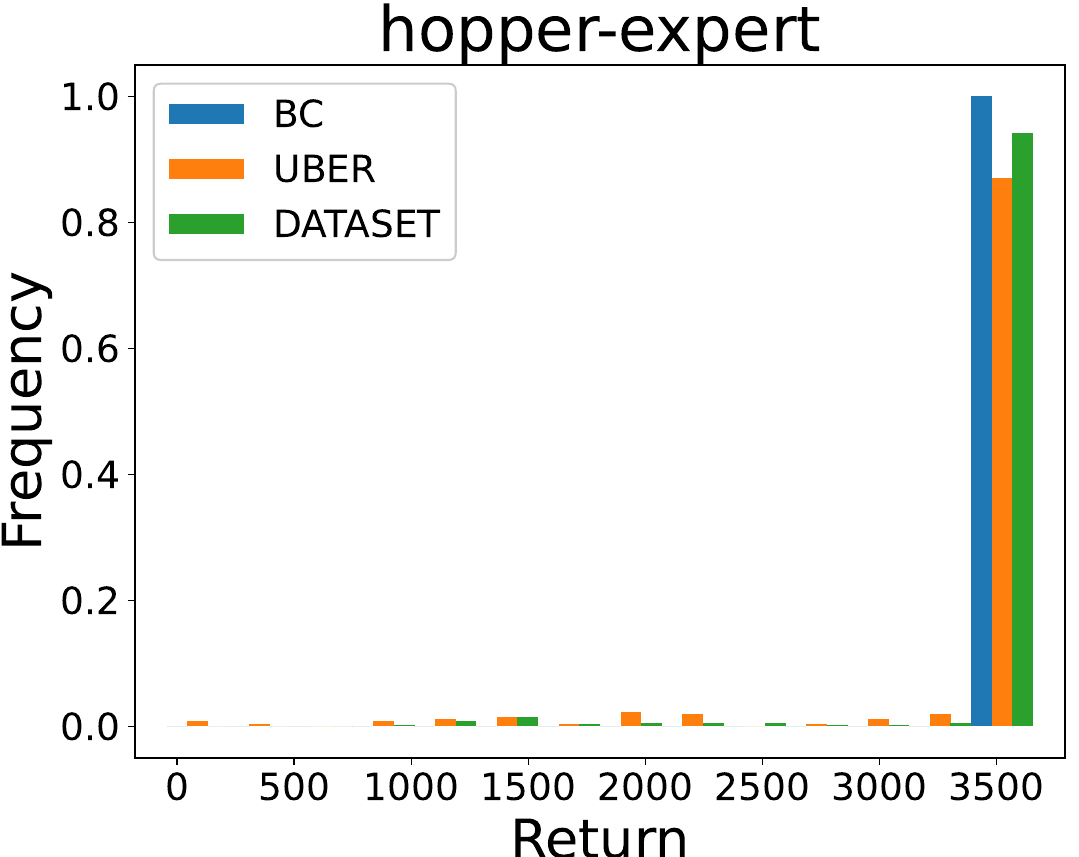}}
    \vspace{1mm}\subfigure{\includegraphics[scale=0.22]{distribution/walker2d-medium.pdf}}
    \subfigure{\includegraphics[scale=0.22]{distribution/walker2d-medium-expert.pdf}}
    \vspace{1mm}\subfigure{\includegraphics[scale=0.22]{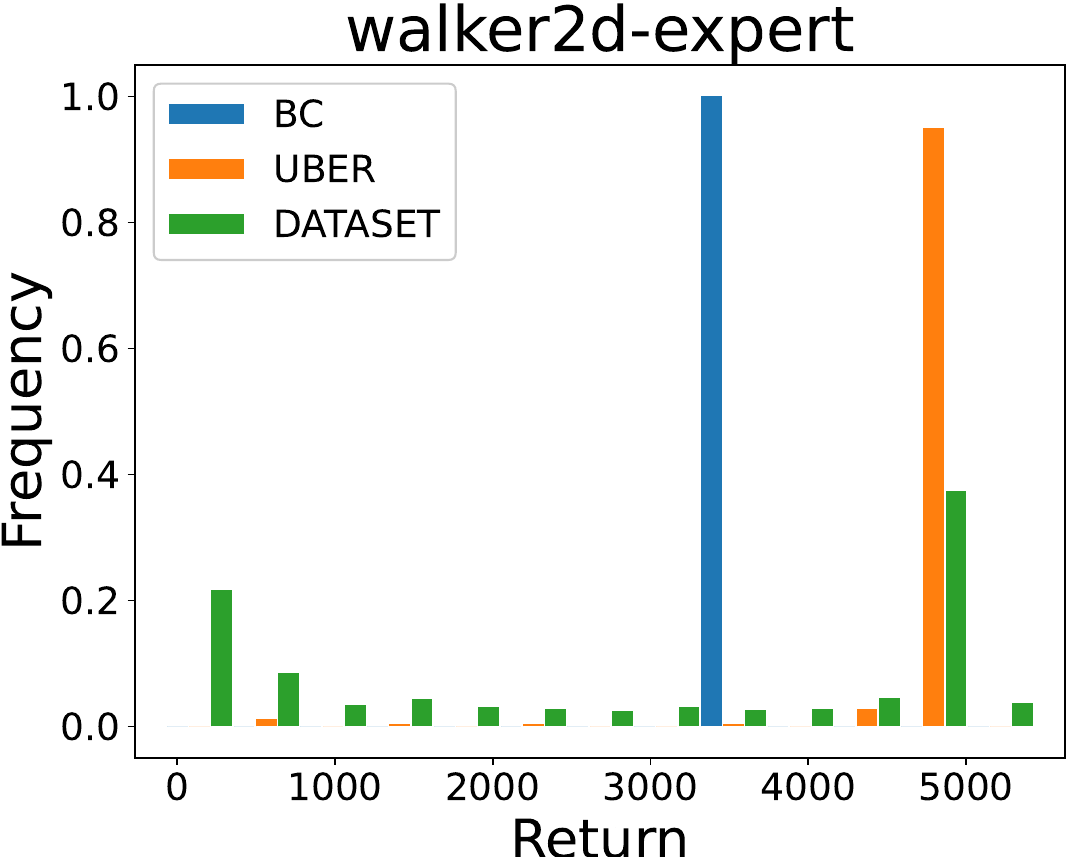}}
    \vspace{1mm}\subfigure{\includegraphics[scale=0.22]{distribution/halfcheetah-medium.pdf}}
    \subfigure{\includegraphics[scale=0.22]{distribution/halfcheetah-medium-expert.pdf}}
    \vspace{1mm}\subfigure{\includegraphics[scale=0.22]{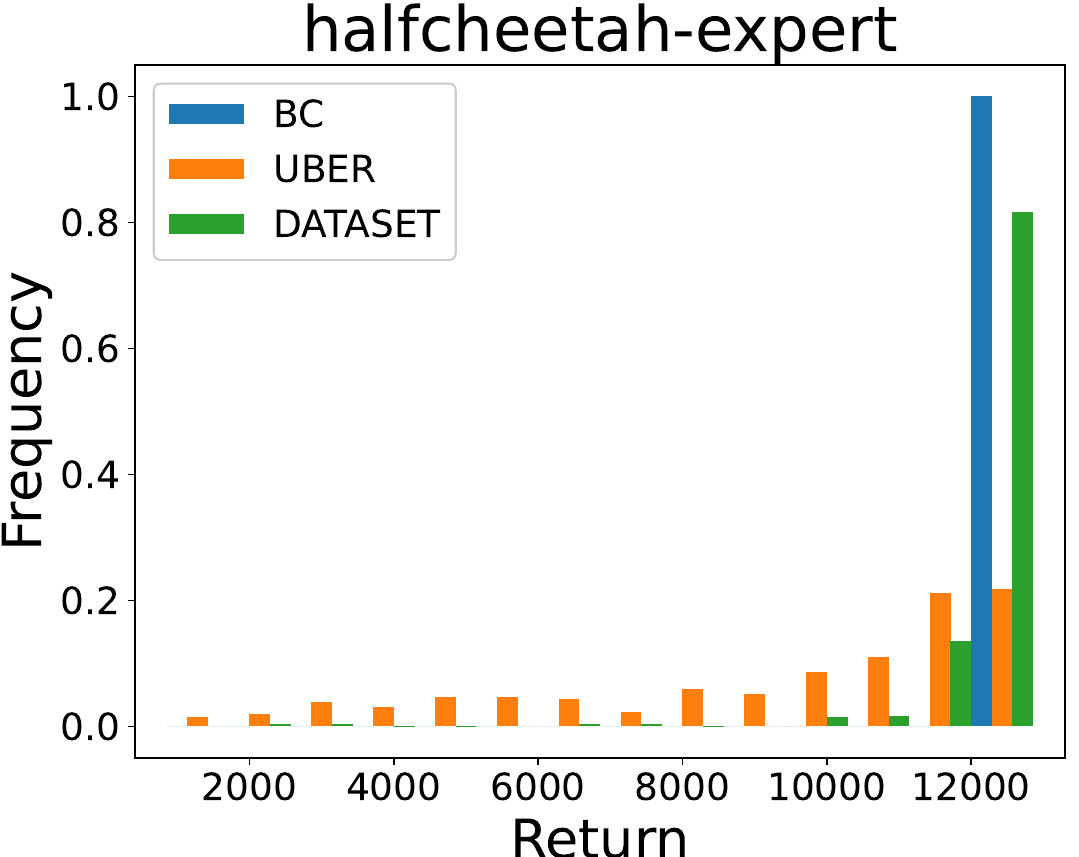}}
    \vspace{1mm}\subfigure{\includegraphics[scale=0.22]{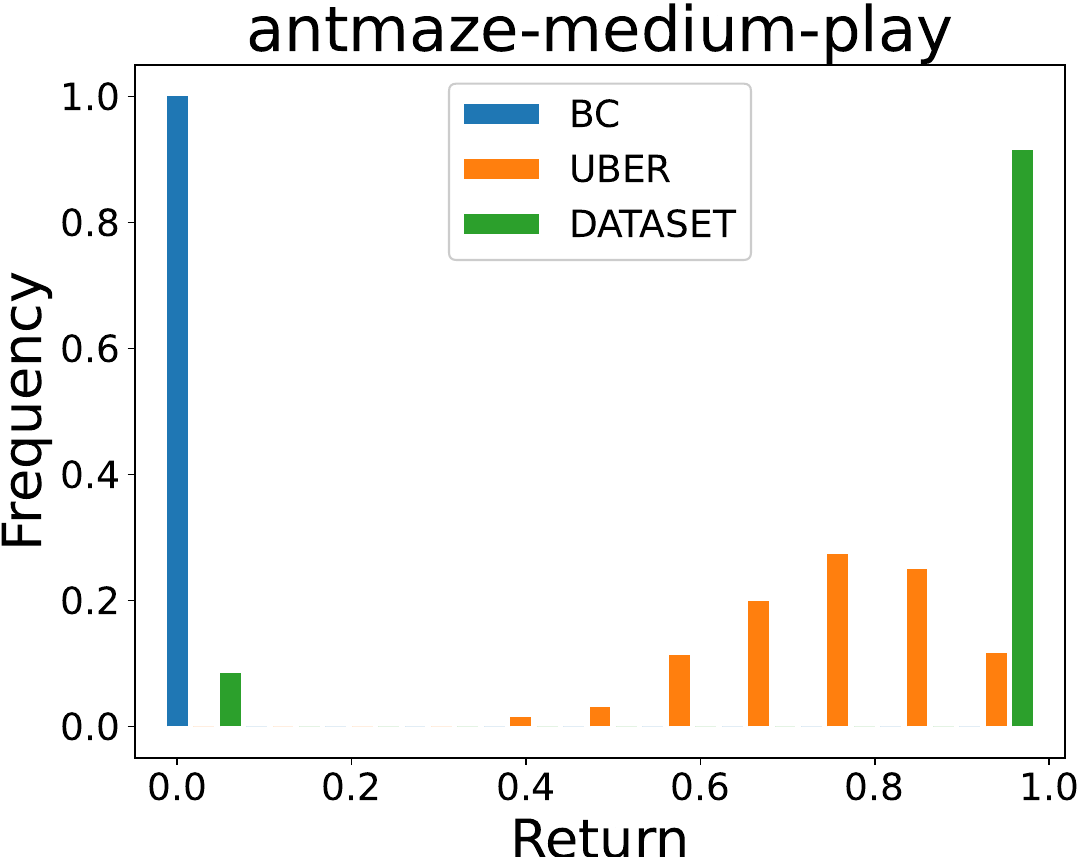}}
    \vspace{1mm}\subfigure{\includegraphics[scale=0.22]{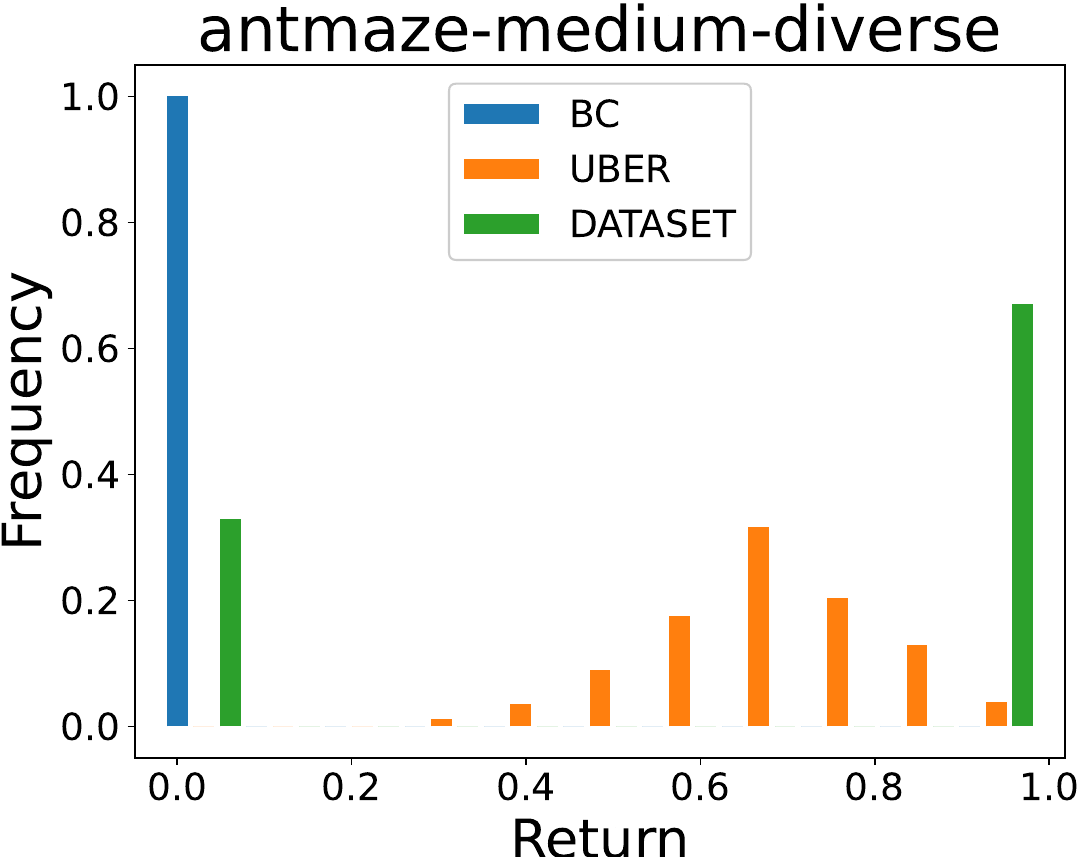}}
    \caption{Distribution of random intent priors, datasets, and behavior cloning policies.
    We restrict the y-axis to a small range to make the distribution clearer.
    }
    \label{fig: appendix random intent sta}
\end{figure}
\paragraph{Learning curve for the Antmaze environment.} Here we provide the detailed learning curve for the Antmaze experiment to demonstrate the efficiency of UBER. We can see from Figure~\ref{fig: antmaze result} that UBER learns a good strategy in hard exploration tasks like \texttt{medium} and \texttt{large}, indicating UBER's strong ability to learn useful behaviors.

\begin{figure}[ht]
    \centering\subfigure{\includegraphics[scale=0.25]{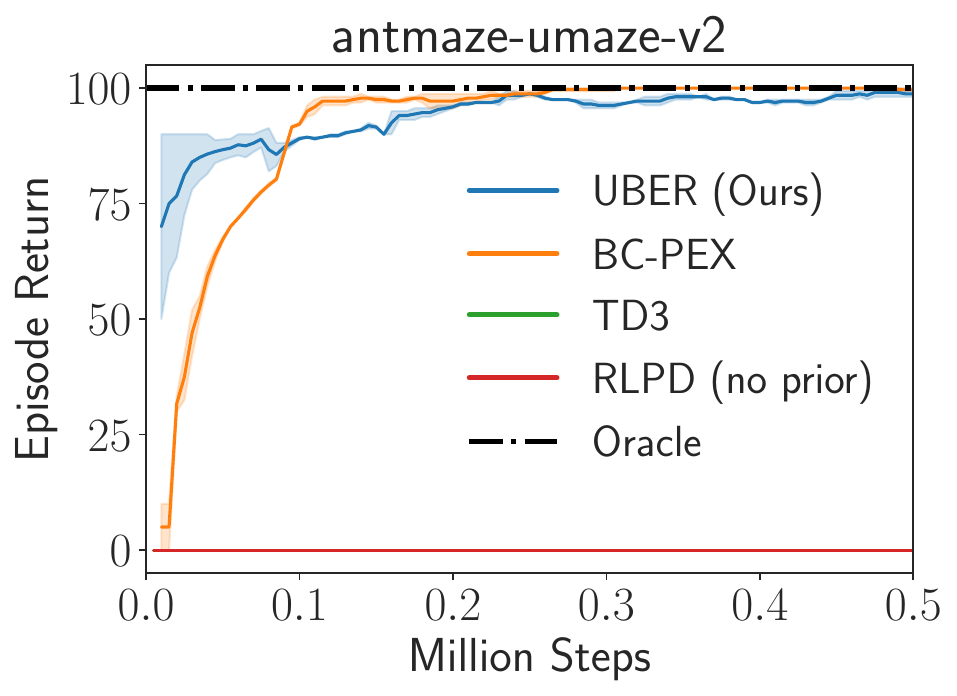}}\subfigure{\includegraphics[scale=0.25]{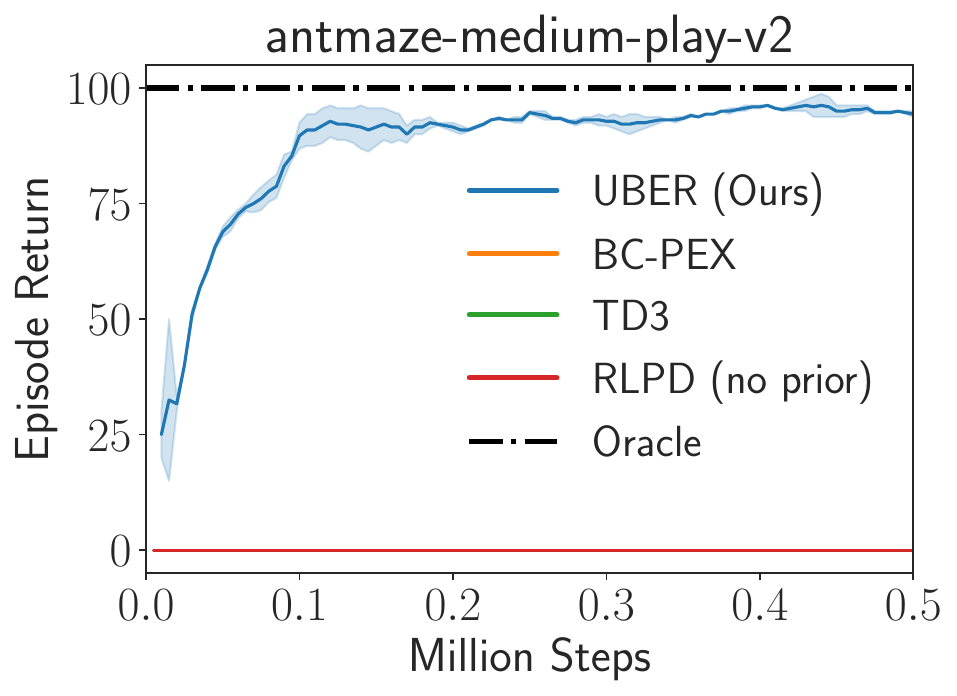}}\subfigure{\includegraphics[scale=0.25]{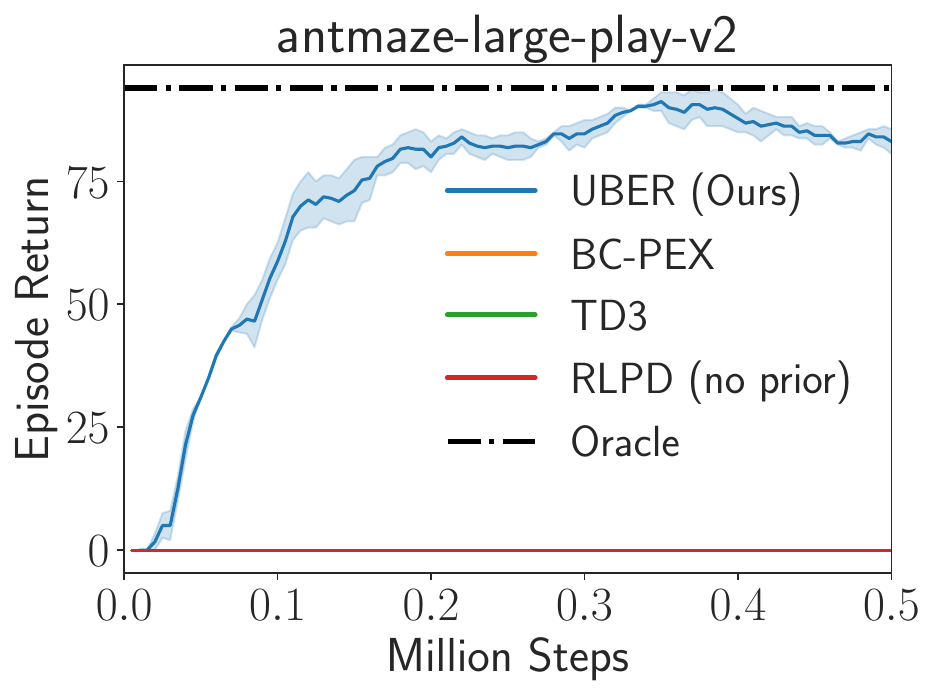}}
    
    \subfigure{\includegraphics[scale=0.25]{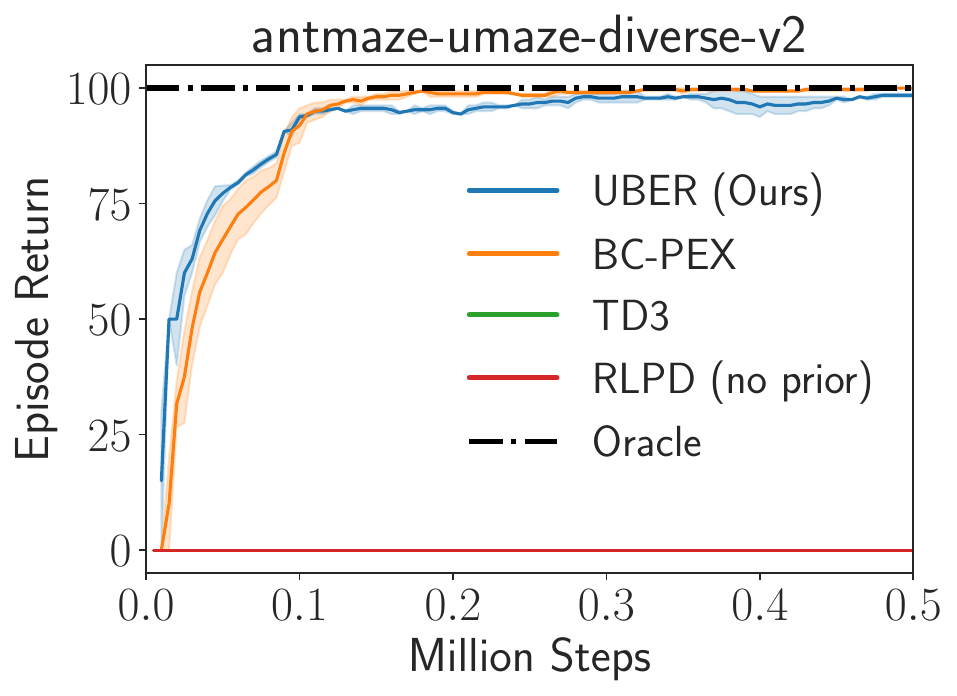}}\subfigure{\includegraphics[scale=0.25]{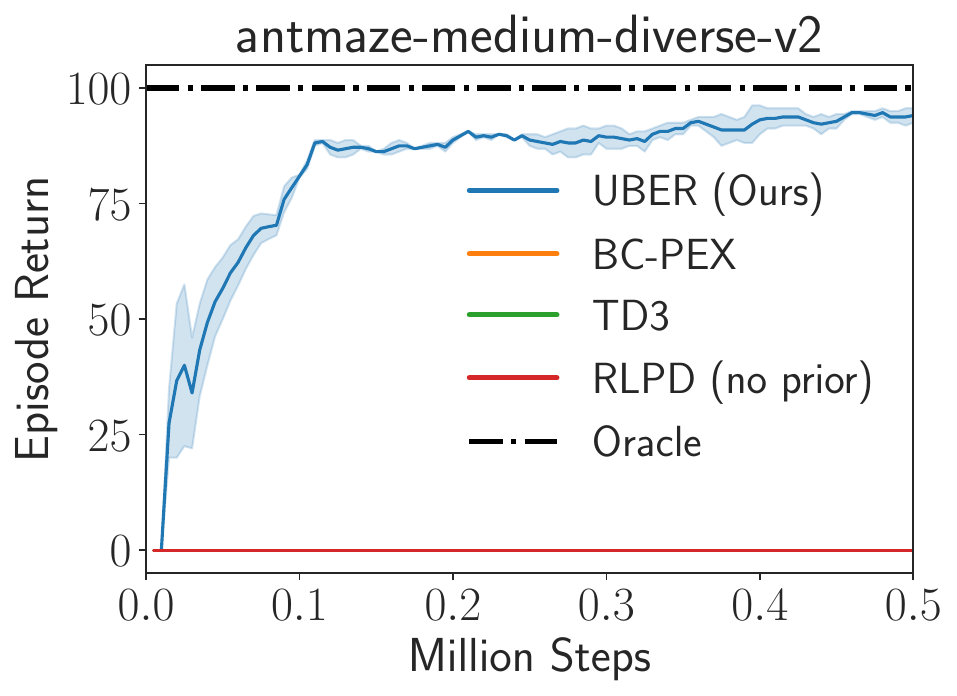}}
    \subfigure{\includegraphics[scale=0.25]{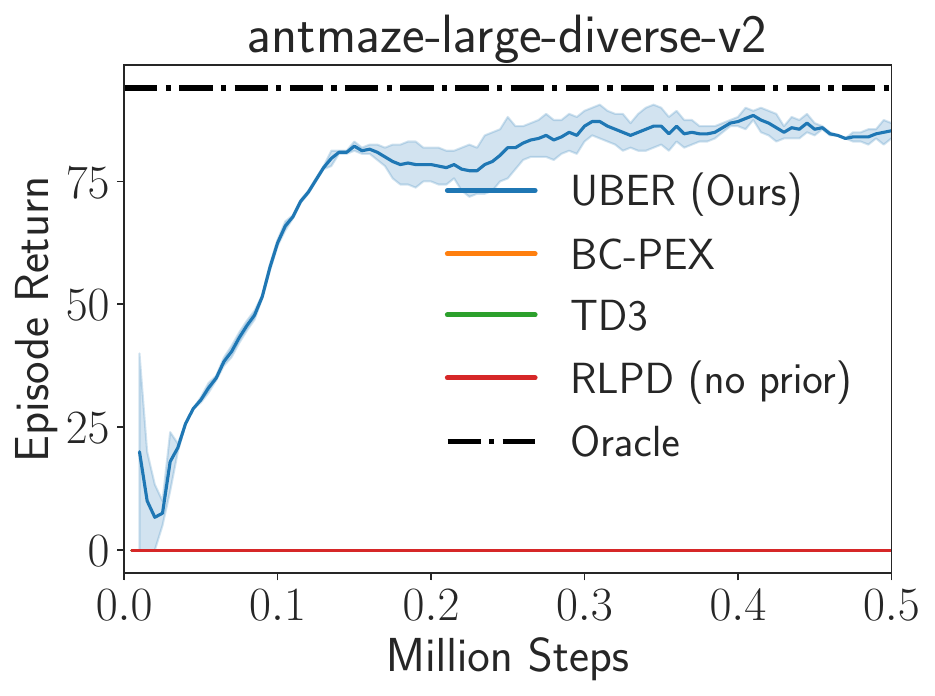}}
    \caption{Comparison between UBER and baselines in the online phase in the Antmaze domain.
    Oracle is the performance of RLPD at 500k step.
    We adopt a normalized score metric averaged with three random seeds.
    }
    \label{fig: antmaze result}
\end{figure}
% \clearpage
\paragraph{Ablation study for the online policy reuse module.}
We first replace the policy expansion module with CUP~\citep{zhang2022cup}, UBER~(CUP) to investigate the effect of various online policy reuse modules. 
The experimental results in Figure~\ref{fig: appendix policy reuse ablation} show that UBER+PEX generally outperforms UBER+CUP.
We hypothesize that CUP includes a hard reuse method, which requires the optimized policy to be close to the reused policies during training, 
while PEX is a soft method that only uses the pre-trained behaviors to collect data. This makes PEX more suitable to our setting since UBER may generate undesired  behaviors.

\begin{figure}[ht]
    \centering
    \subfigure{\includegraphics[scale=0.28]{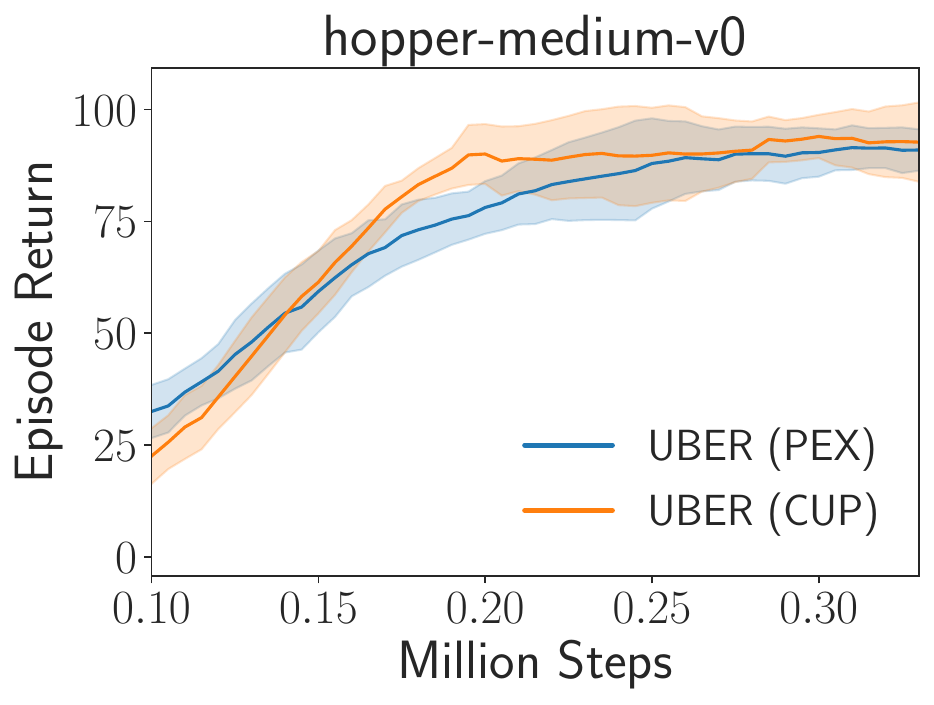}}\subfigure{\includegraphics[scale=0.28]{ablation/ablation_hopper-medium-expert-v0.pdf}}
    \subfigure{\includegraphics[scale=0.28]{ablation/ablation_hopper-expert-v0.pdf}}
    
    \subfigure{\includegraphics[scale=0.28]{ablation/ablation_walker2d-medium-v0.pdf}}
    \subfigure{\includegraphics[scale=0.28]{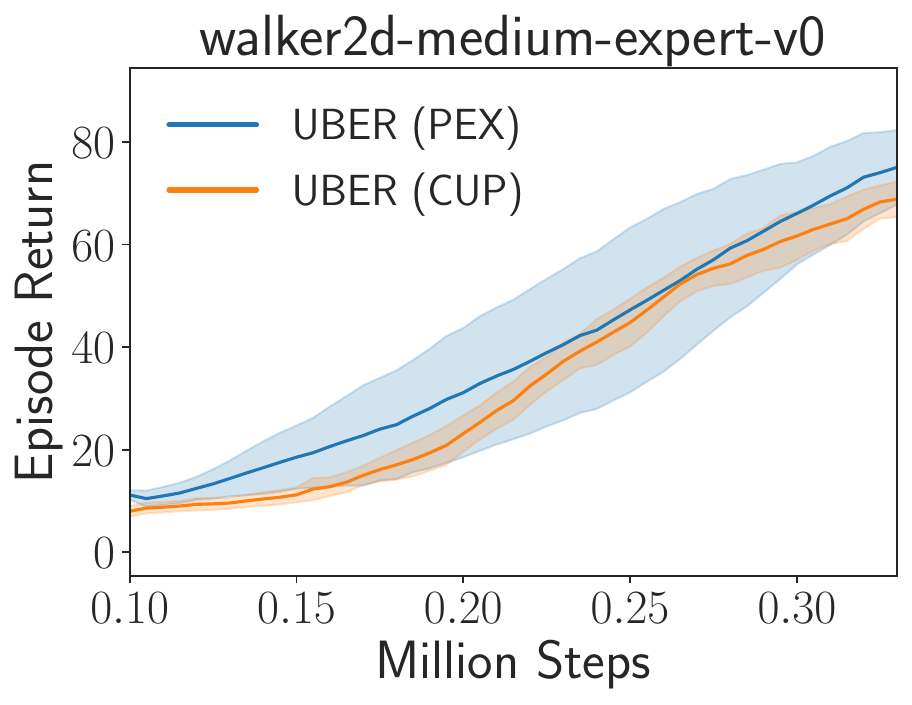}}
    \subfigure{\includegraphics[scale=0.28]{ablation/ablation_walker2d-expert-v0.pdf}}
    
    \subfigure{\includegraphics[scale=0.28]{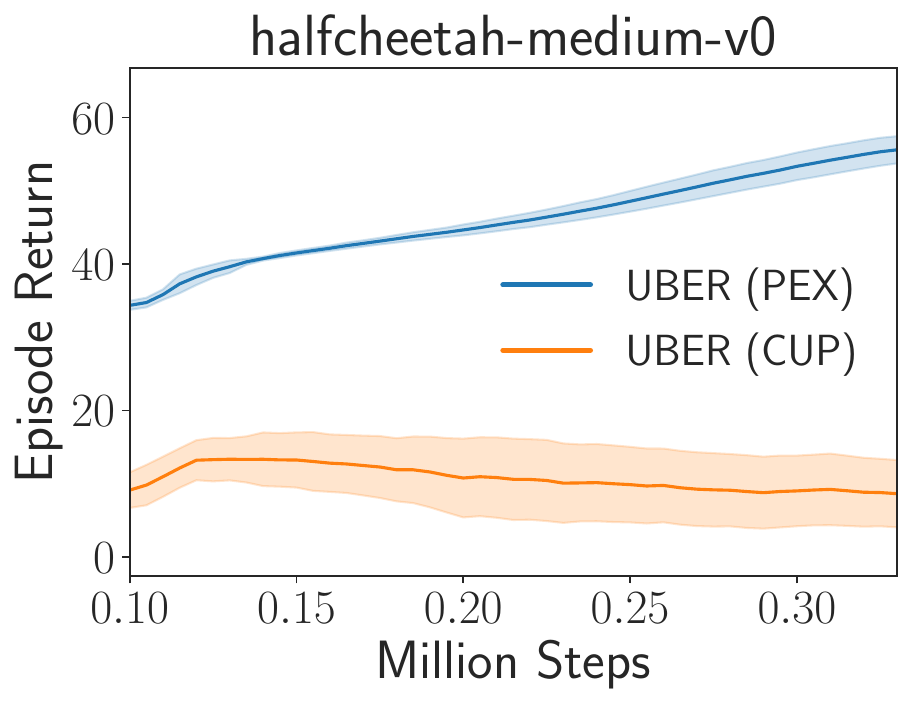}}
    \subfigure{\includegraphics[scale=0.28]{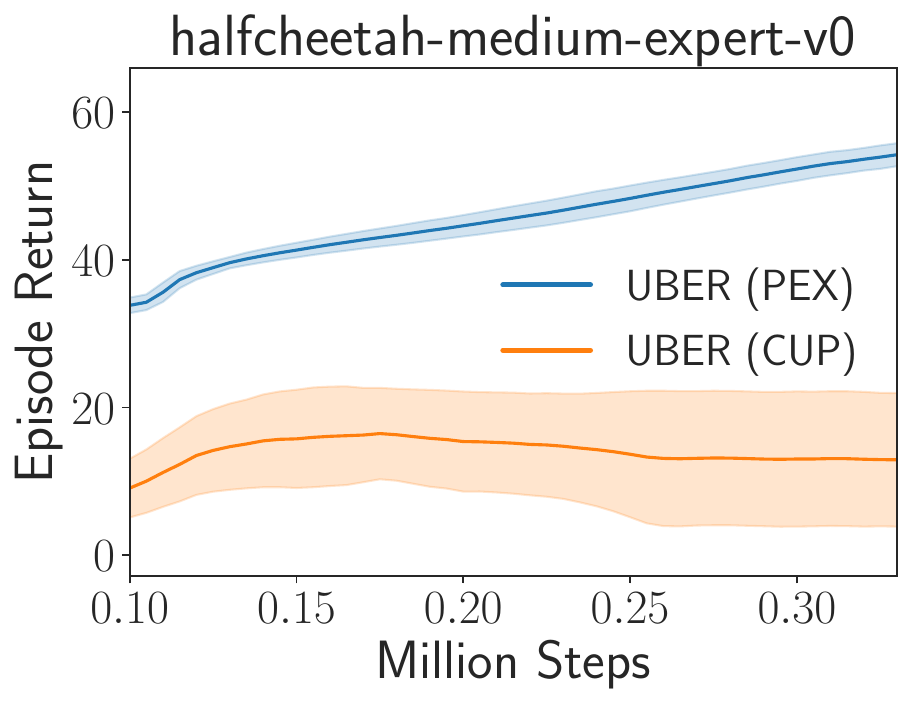}}
    \subfigure{\includegraphics[scale=0.28]{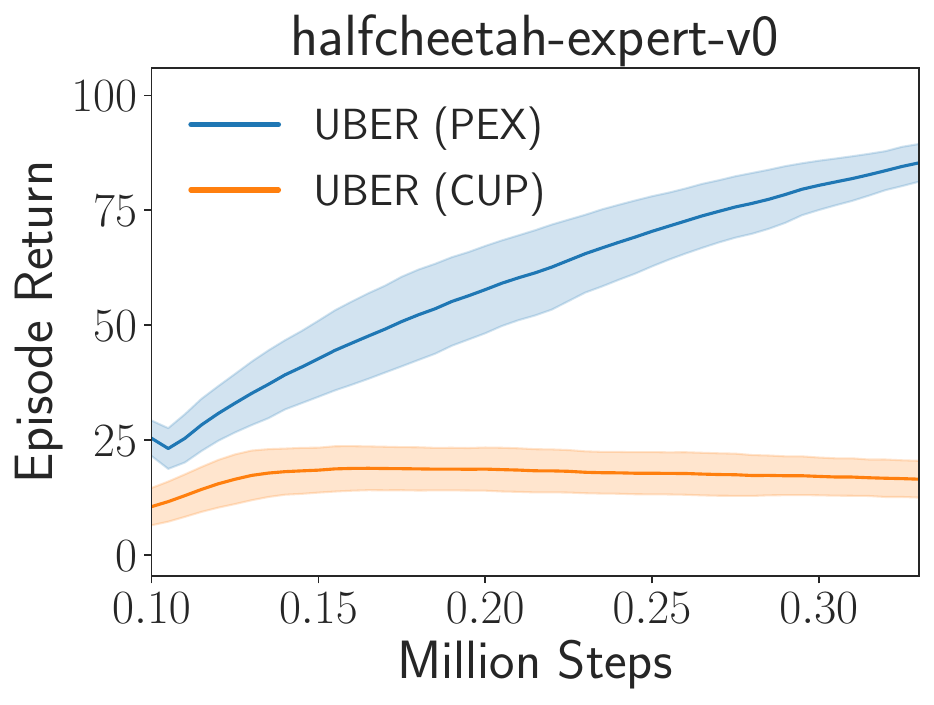}}

    \caption{Complete result for the ablation study for the online policy reuse module. Using policy expansion (PEX) for downstream tasks is better than critic-guided policy reuse (CUP) in general}
    \label{fig: appendix policy reuse ablation}
\end{figure}

\paragraph{Complete comparison with baselines.} Here we provide a complete comparison with several baselines, including OPAL, PARROT and UDS, as shown in Figure~\ref{fig: new baseline result 2}. Our method outperform these baselines due to the ability to extract diverse behaviors.

\begin{figure}[ht]
    \centering
    \subfigure{\includegraphics[scale=0.21]{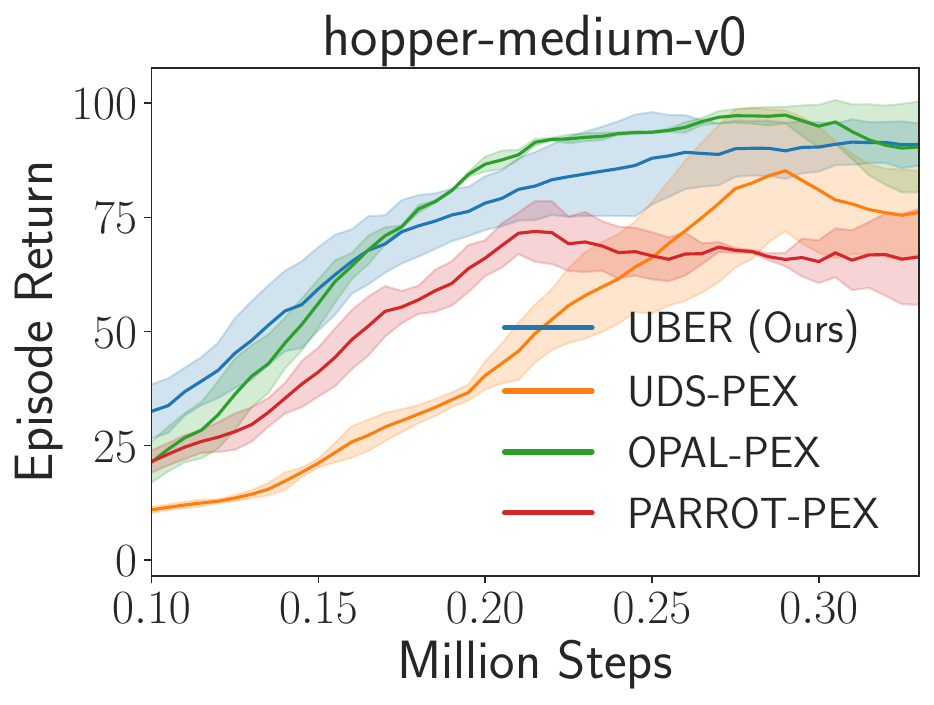}}
    \subfigure{\includegraphics[scale=0.21]{new_baseline/walker2d-medium-v0_newbaseline.pdf}}
    \subfigure{\includegraphics[scale=0.21]{new_baseline/walker2d-medium-expert-v0_newbaseline.pdf}}
    \subfigure{\includegraphics[scale=0.21]{new_baseline/walker2d-expert-v0_newbaseline.pdf}}
    \subfigure{\includegraphics[scale=0.21]{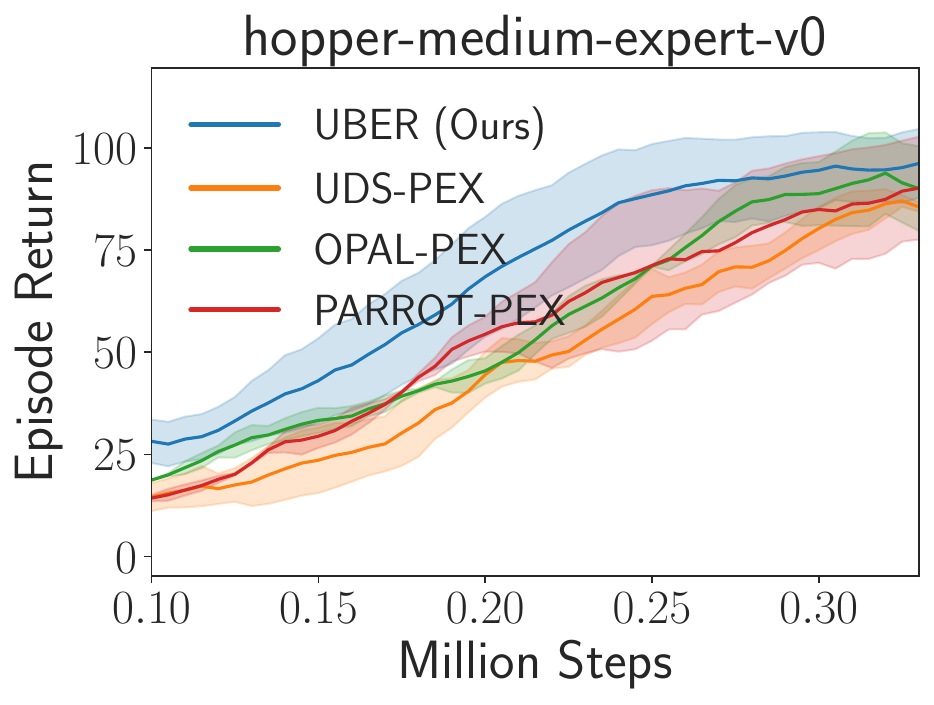}}
    \subfigure{\includegraphics[scale=0.21]{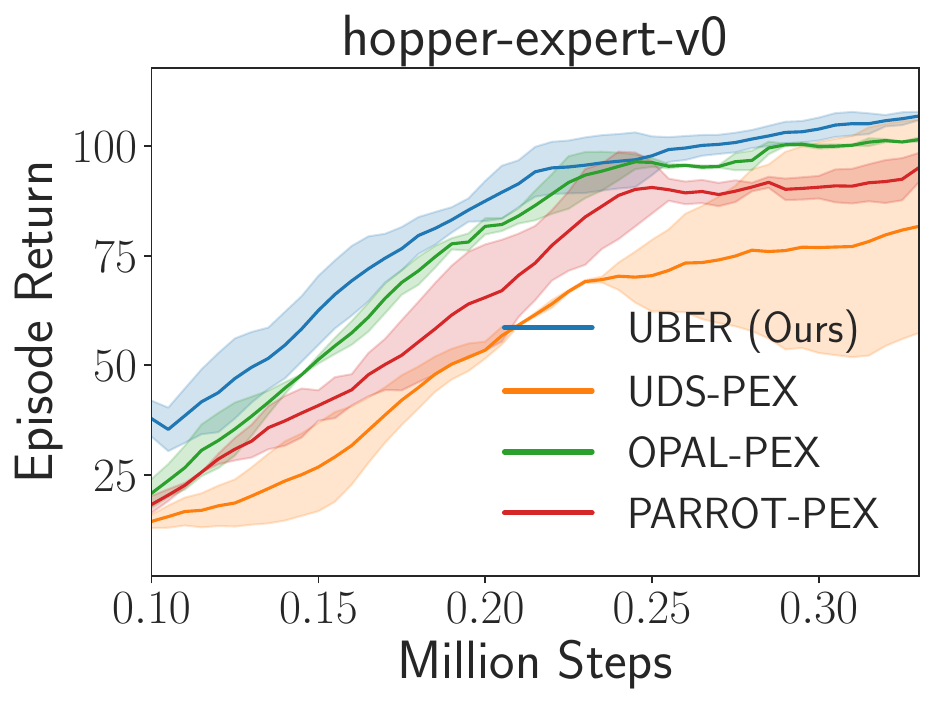}}    
    \subfigure{\includegraphics[scale=0.21]{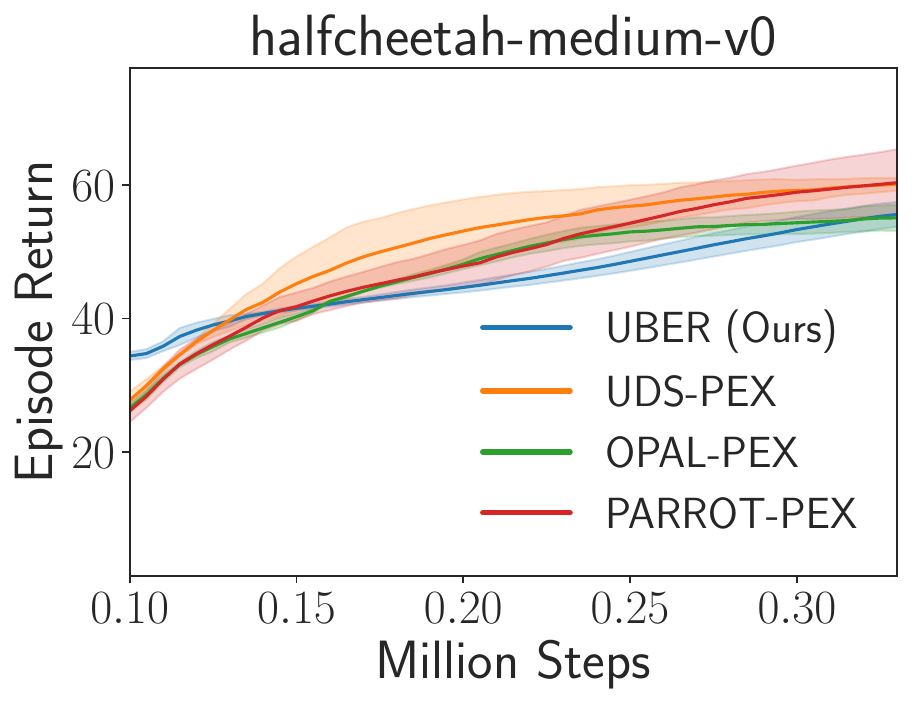}}
    \subfigure{\includegraphics[scale=0.21]{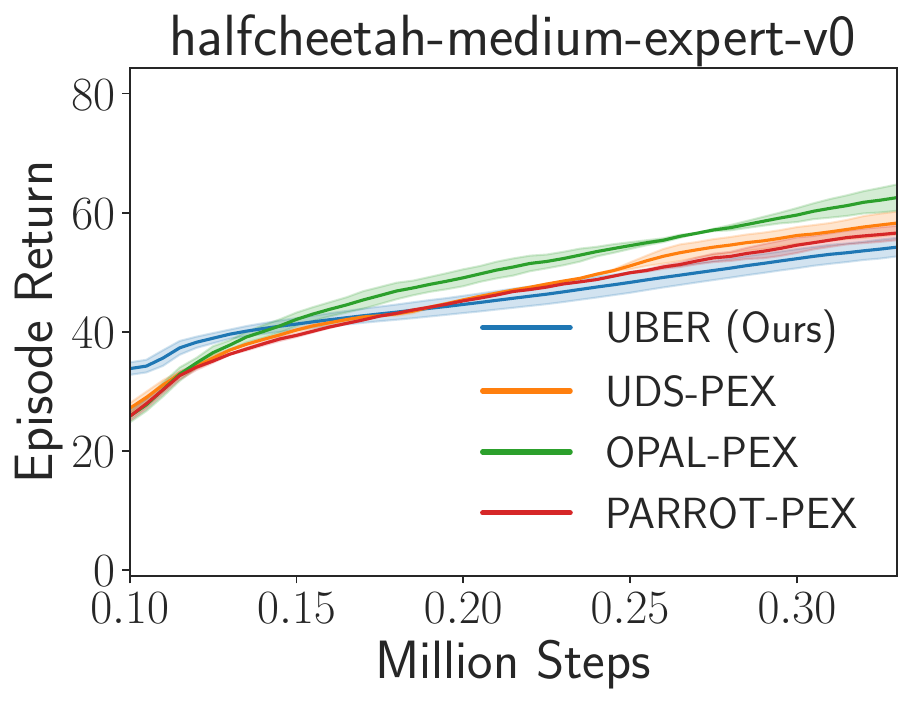}}
    \caption{Complete result for the comparison between UBER and baselines including unsupervised behavior extraction and data-sharing method. We adopt a normalized score metric averaged with five random seeds.}
    \label{fig: new baseline result 2}
\end{figure}

% \clearpage

\paragraph{Ablation study for the random intent priors module.}
we conduct an ablation study for the random intent priors module.
We use the average reward in the offline datasets to learn offline policies and then follow the same online policy reuse process in Algorithm~\ref{alg: online rl}. Using average reward serves as a strong baseline since it may generate near-optimal performance as observed in~\citet{shin2023benchmarks}.
We name the average reward-based offline behavior extraction as AVG-PEX.
The experimental results in Figure~\ref{fig: appendix random intent ablation2} show that while average reward-based offline optimization  can extract some useful behaviors and accelerate the downstream tasks, using random intent priors performs better due to the diversity of the behavior set.

\begin{figure}[ht]
    \centering
    \subfigure{\includegraphics[scale=0.28]{avg/avg_hopper-medium-v0.pdf}}
    \subfigure{\includegraphics[scale=0.28]{avg/avg_hopper-medium-expert-v0.pdf}}\subfigure{\includegraphics[scale=0.28]{avg/avg_hopper-expert-v0.pdf}}
    
    \subfigure{\includegraphics[scale=0.28]{avg/avg_walker2d-medium-v0.pdf}}
    \subfigure{\includegraphics[scale=0.28]{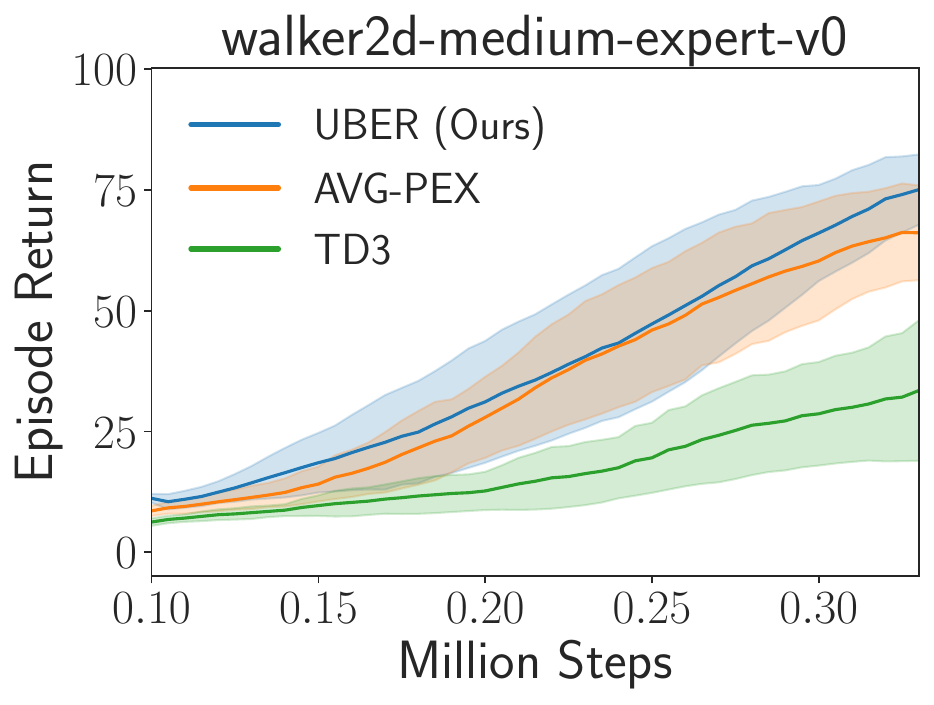}}
    \subfigure{\includegraphics[scale=0.28]{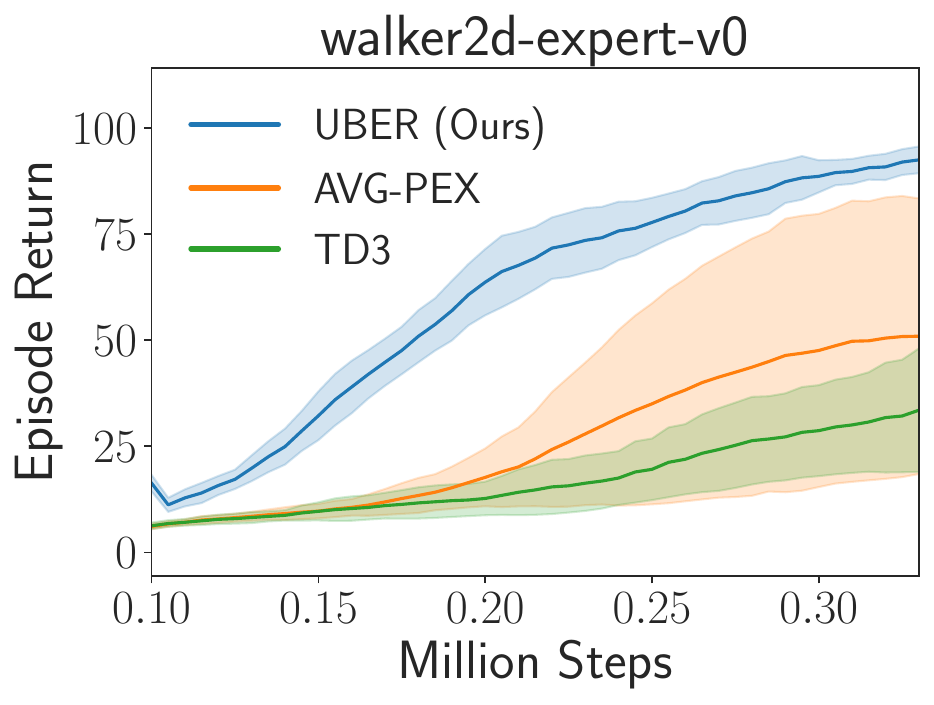}}
    
    \subfigure{\includegraphics[scale=0.28]{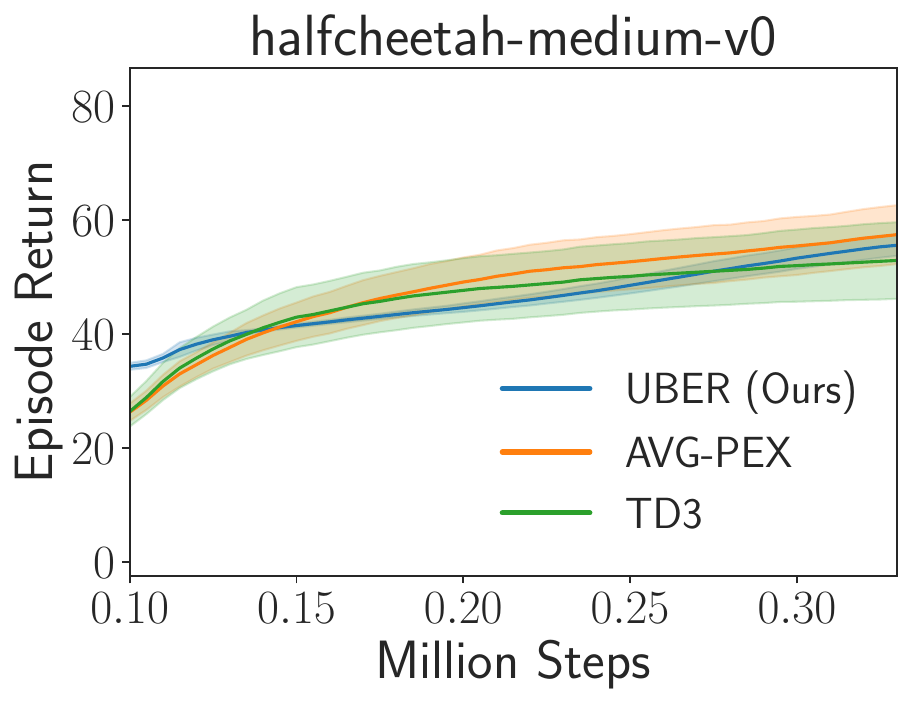}}
    \subfigure{\includegraphics[scale=0.28]{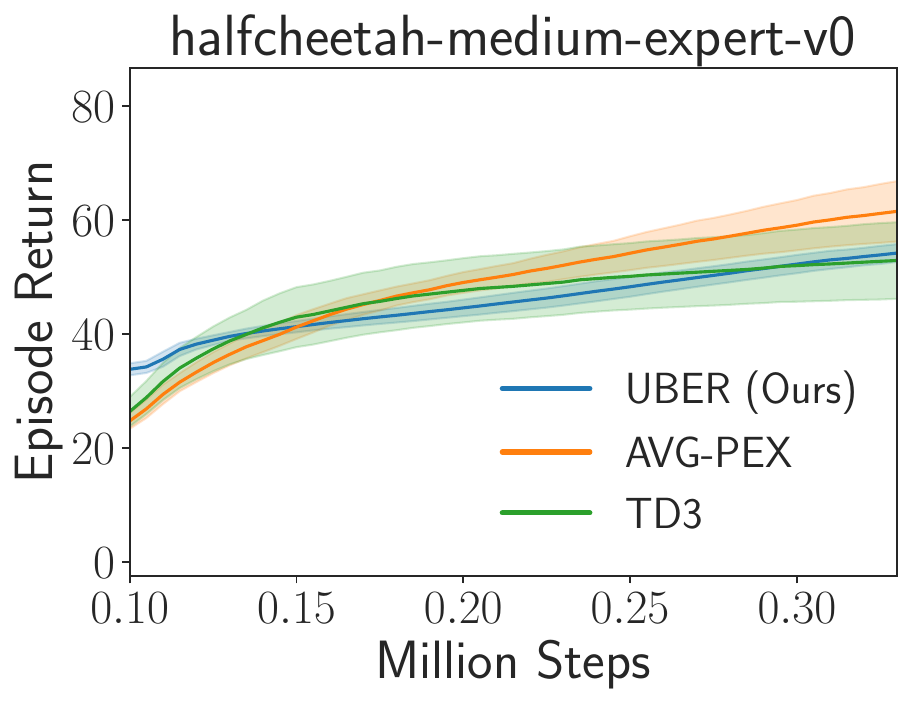}}
    \subfigure{\includegraphics[scale=0.28]{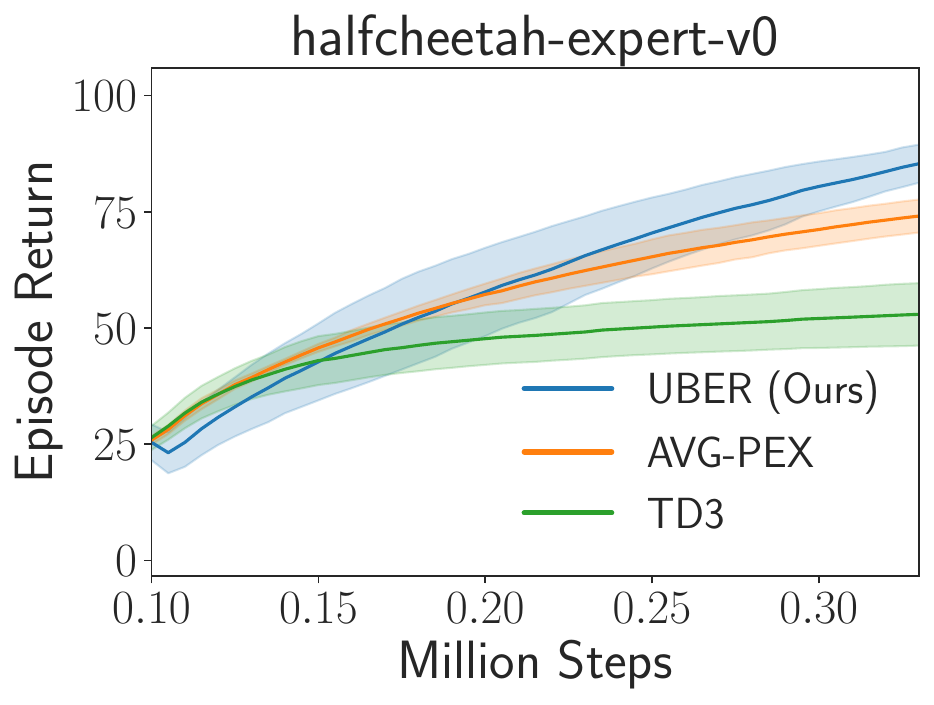}}
    
    \caption{Complete result for the ablation study for the random intent priors module. Using average reward for policies learning can lead to nearly optimal performances on the original task. Nevertheless, our method outperforms such a baseline (AVG-PEX) consistently, showing the importance of behavioral diversity.}
    \label{fig: appendix random intent ablation2}
\end{figure}

\clearpage
\section{Experimental Details}
\label{appendix: exp details}

\paragraph{Experimental Setting.}
For BC-PEX, we first extract behavior from the offline dataset based on the Behavior Cloning, then conduct online policy expansion based on Algorithm~\ref{alg: online rl}.
For Avg-PEX, we first use the average reward for the offline optimization, then conduct online policy expansion based on Algorithm~\ref{alg: online rl}.
For RLPD, we load the offline dataset with the true reward value for the prior data.
For RLPD~(no prior), we do not load offline-dataset into the buffer.

\paragraph{Hyper-parameters.}
For the Mujoco and metaworld tasks, we adopt the TD3+BC and TD3 as the backbone of offline and online algorithms.
For the Antmaze tasks, we adopt the IQL and RLPD as the backbone of offline and online algorithms.
We outline the hyper-parameters used by UBER in Table~\ref{tab: parameters mujoco}, Table~\ref{tab: parameters metaworld} and Table~\ref{tab: parameters antmaze}.

\begin{table}[ht]
    \centering
    \begin{tabular}{ll}
    \toprule
      Hyperparameter   & Value \\
      \midrule
      \hspace{0.3cm} Optimizer & Adam \\
      \hspace{0.3cm} Critic learning rate & 3e-4 \\
      \hspace{0.3cm} Actor learning rate & 3e-4 \\
      \hspace{0.3cm} Mini-batch size & 256 \\
      \hspace{0.3cm} Discount factor & 0.99 \\
      \hspace{0.3cm} Target update rate & 5e-3 \\
      \hspace{0.3cm} Policy noise & 0.2 \\
      \hspace{0.3cm} Policy noise clipping & (-0.5, 0.5) \\
      \hspace{0.3cm} TD3+BC parameter $\alpha$ & 2.5\\
      \midrule
      Architecture & Value \\
      \midrule
      \hspace{0.3cm} Critic hidden dim & 256 \\
      \hspace{0.3cm} Critic hidden layers & 2 \\
      \hspace{0.3cm} Critic activation function & ReLU \\
      \hspace{0.3cm} Actor hidden dim & 256 \\
      \hspace{0.3cm} Actor hidden layers & 2 \\
      \hspace{0.3cm} Actor activation function & ReLU \\
      \hspace{0.3cm} Random Net hidden dim & 256 \\
      \hspace{0.3cm} Random Net hidden layers & 2 \\
      \hspace{0.3cm} Random Net activation function & ReLU \\
      \midrule
      UBER Parameters & Value \\
      \midrule
      \hspace{0.3cm} Random reward dim $n$ & 256 \\
      \hspace{0.3cm} PEX temperature~$\alpha$ & 10 \\
      \bottomrule
    \end{tabular}
    \caption{Hyper-parameters sheet of UBER in Mujoco tasks.}
    \label{tab: parameters mujoco}
\end{table}

\begin{table}[ht]
    \centering
    \begin{tabular}{ll}
    \toprule
      Hyperparameter   & Value \\
      \midrule
      \hspace{0.3cm} Optimizer & Adam \\
      \hspace{0.3cm} Critic learning rate & 3e-4 \\
      \hspace{0.3cm} Actor learning rate & 3e-4 \\
      \hspace{0.3cm} Mini-batch size & 256 \\
      \hspace{0.3cm} Discount factor & 0.99 \\
      \hspace{0.3cm} Target update rate & 5e-3 \\
      \hspace{0.3cm} Policy noise & 0.2 \\
      \hspace{0.3cm} Policy noise clipping & (-0.5, 0.5) \\
      \hspace{0.3cm} TD3+BC parameter $\alpha$ & 2.5\\
      \midrule
      Architecture & Value \\
      \midrule
      \hspace{0.3cm} Critic hidden dim & 256 \\
      \hspace{0.3cm} Critic hidden layers & 2 \\
      \hspace{0.3cm} Critic activation function & ReLU \\
      \hspace{0.3cm} Actor hidden dim & 256 \\
      \hspace{0.3cm} Actor hidden layers & 2 \\
      \hspace{0.3cm} Actor activation function & ReLU \\
      \hspace{0.3cm} Random Net hidden dim & 256 \\
      \hspace{0.3cm} Random Net hidden layers & 2 \\
      \hspace{0.3cm} Random Net activation function & ReLU \\
      \midrule
      UBER Parameters & Value \\
      \midrule
      \hspace{0.3cm} Random reward dim $n$ & 100 \\
      \hspace{0.3cm} PEX temperature~$\alpha$ & 10 \\
      \bottomrule
    \end{tabular}
    \caption{Hyper-parameters sheet of UBER in metaworld tasks.}
    \label{tab: parameters metaworld}
\end{table}

\begin{table}[ht]
    \centering
    \begin{tabular}{ll}
    \toprule
      Hyperparameter   & Value \\
      \midrule
      \hspace{0.3cm} Optimizer & Adam \\
      \hspace{0.3cm} Critic learning rate & 3e-4 \\
      \hspace{0.3cm} Actor learning rate & 3e-4 \\
      \hspace{0.3cm} Mini-batch size & 256 \\
      \hspace{0.3cm} Discount factor & 0.99 \\
      \hspace{0.3cm} Target update rate & 5e-3 \\
      \hspace{0.3cm} IQL parameter~$\tau$ & 0.9 \\
      \hspace{0.3cm} RLPD parameter~$G$ & 20 \\
      \hspace{0.3cm} Ensemble Size & 10 \\
      \midrule
      Architecture & Value \\
      \midrule
      \hspace{0.3cm} Critic hidden dim & 256 \\
      \hspace{0.3cm} Critic hidden layers & 2 \\
      \hspace{0.3cm} Critic activation function & ReLU \\
      \hspace{0.3cm} Actor hidden dim & 256 \\
      \hspace{0.3cm} Actor hidden layers & 2 \\
      \hspace{0.3cm} Actor activation function & ReLU \\
      \hspace{0.3cm} Random Net hidden dim & 256 \\
      \hspace{0.3cm} Random Net hidden layers & 2 \\
      \hspace{0.3cm} Random Net activation function & ReLU \\
      \midrule
      UBER Parameters & Value \\
      \midrule
      \hspace{0.3cm} Random reward dim $n$ & 256 \\
      \hspace{0.3cm} PEX temperature~$\alpha$ & 10 \\
      \bottomrule
    \end{tabular}
    \caption{Hyper-parameters sheet of UBER in Antmaze tasks.}
    \label{tab: parameters antmaze}
\end{table}

\paragraph{Baselines Implementation.}
We adopt the author-provided implementations from GitHub for TD3~\footnote{\url{https://github.com/sfujim/TD3}}, TD3+BC~\footnote{\url{ https://github.com/sfujim/TD3_BC}}, IQL~\footnote{\url{https://github.com/ikostrikov/implicit_q_learning}}, RLPD~\footnote{\url{https://github.com/ikostrikov/rlpd}} and DrQ-v2~\footnote{\url{https://github.com/facebookresearch/drqv2}}.
All experiments are conducted on the same experimental setup, a single GeForce RTX 3090 GPU and an Intel Core i7-6700k CPU at 4.00GHz.

%For the final version, one more page can be added.
%If you want, you can use an appendix like this one, even using the one-column format.
%%%%%%%%%%%%%%%%%%%%%%%%%%%%%%%%%%%%%%%%%%%%%%%%%%%%%%%%%%%%%%%%%%%%%%%%%%%%%%%
%%%%%%%%%%%%%%%%%%%%%%%%%%%%%%%%%%%%%%%%%%%%%%%%%%%%%%%%%%%%%%%%%%%%%%%%%%%%%%%

\end{document}